\documentclass[lettersize,journal]{IEEEtran}
\usepackage{amsmath,amsfonts}
\usepackage{MnSymbol}%
\usepackage{wasysym}%

\usepackage{algorithm}
\usepackage{algorithmic}
\DeclareMathOperator*{\argmax}{arg\,max}
\DeclareMathOperator*{\argmin}{arg\,min}
\usepackage{amsthm}
\usepackage{array}
\usepackage{textcomp}
\usepackage{stfloats}
\usepackage{subfigure}
\usepackage{adjustbox}
\usepackage[hidelinks]{hyperref}
\usepackage{url}
\usepackage{verbatim}
\usepackage{graphicx}
\usepackage{cite}
\usepackage{thmtools}
\usepackage{thm-restate}
\usepackage[dvipsnames]{xcolor}
\usepackage{gensymb}

\begin{document}

\newtheorem{theorem}{Theorem}
\newtheorem{definition}{Definition}
\newtheorem{problem}{Problem}
\newtheorem{myexam}{Example}
\newtheorem{assumption}{Assumption}
\newtheorem{proposition}{Proposition}
\newtheorem{remark}{Remark}
\newtheorem{lemma}{Lemma}
\newtheorem{corollary}{Corollary}

\newcommand{\algcom}[1]{{\color{Peach} #1}}
\newcommand{\algcomdec}[1]{{\color{YellowGreen} #1}}
\newcommand{\algcompro}[1]{{\color{Orchid} #1}}
\newcommand{\MRW}[1]{{\color{blue}[MW: #1]}}
\newcommand{\ap}[1]{{\color{orange}[AP: #1]}}
\newcommand{\LL}[1]{{\color{red}[LL: #1]}}
\newcommand{\MK}[1]{{\color{green}[MK: #1]}}

\newcommand{\pX}{\mathbf{x}}
\newcommand{\dataset}{\mathcal{D}}

\newcommand{\BNN}[0]{{f}^{\mathbf{w}}(x)}
\newcommand{\post}[0]{p(w | \mathcal{D})}
\newcommand{\postprob}[0]{Prob_{w\sim p(w | \mathcal{D})}}
\newcommand{\psafe}[0]{P_{\text{safe}}(T,S)}
\newcommand{\dsafe}[0]{D_{\text{safe}}(T,S)}
\newcommand{\hrulealg}[0]{\vspace{1mm} \hrule \vspace{1mm}}

\title{Adversarial Robustness Certification\\  for Bayesian Neural Networks}

\author{\IEEEauthorblockN{Matthew Wicker\IEEEauthorrefmark{1},
Andrea Patane\IEEEauthorrefmark{2}, Luca Laurenti\IEEEauthorrefmark{3} and
Marta Kwiatkowska\IEEEauthorrefmark{1}}\\ \vspace{1em}
\IEEEauthorblockA{\IEEEauthorrefmark{1} Department of Computer Science,
University of Oxford, Oxford, United Kingdom\\
Email: \IEEEauthorrefmark{1}(matthew.wicker, marta.kwiatkowska)@cs.ox.ac.uk}\\ 
\IEEEauthorblockA{\IEEEauthorrefmark{2} School of Computer Science and Statistics, Trinity College Dublin, Ireland \\
Email: \IEEEauthorrefmark{2}apatane@tcd.ie} \\
\IEEEauthorblockA{\IEEEauthorrefmark{3} Delft Center for Systems and Control (DCSC), TU Delft, Delft, Netherlands \\
Email: \IEEEauthorrefmark{3}luca.laurenti@tudelft.nl}
\thanks{This project received funding from the ERC under the European Union’s Horizon 2020 research and innovation programme (FUN2MODEL, grant agreement No.~834115). MK further acknowledges funding from ELSA: European Lighthouse on Secure and Safe AI project (grant agreement
No. 101070617 under UK guarantee).}
\thanks{Preliminary work on this paper was done while  Andrea Patane and Luca Laurenti were at the University of Oxford funded by FUN2MODEL.}
}

\markboth{Pre-print. Manuscript Under Review}{Wicker \MakeLowercase{\textit{et al.}}: Certification of Bayesian Neural Networks}


\maketitle

%We formulate robustness certificates for BNNs as a guarantee that a compact set of inputs $T \subseteq  \mathbb{R}^m$ is mapped by the BNN into a safe output set $S \subseteq \mathbb{R}^n$. We focus on two main robustness properties of BNNs. 

\begin{abstract}
We study the problem of certifying the robustness of Bayesian neural networks (BNNs) to adversarial input perturbations. 
Given a compact set of input points $T \subseteq  \mathbb{R}^m$ and a set of output points $S \subseteq \mathbb{R}^n$, we define two notions of robustness for BNNs in an adversarial setting: probabilistic robustness and decision robustness. Probabilistic robustness is the probability that for all points in $T$ the output of a BNN sampled from the posterior is in $S$.
On the other hand, decision robustness considers the optimal decision of a BNN and checks if for all points in $T$ the optimal decision of the BNN for a given loss function lies within the output set $S$. Although exact computation of these robustness properties is challenging due to the probabilistic and non-convex nature of BNNs, we present a unified computational framework for efficiently and formally bounding them. Our approach is based on weight interval sampling, integration, and bound propagation techniques, and can be applied to BNNs with a large number of parameters, and independently of the (approximate) inference method employed to train the BNN. We evaluate the effectiveness of our methods on various regression and classification tasks, including an industrial regression benchmark, MNIST, traffic sign recognition, and airborne collision avoidance, and demonstrate that our approach enables certification of robustness and uncertainty of BNN predictions. %allows for the reliable certification of BNNs.

\end{abstract}


\section{Introduction}
%\LL{I have compacted quite a lot the intro as I felt it was a bit too vague and not enough on point. For instance, we were not clear on what we were doing on the paper exactly (definition of probabilistic and decision robustness) Old version is commented out in case you want it back.}

While neural networks (NNs) regularly obtain  state-of-the-art performance in many supervised machine learning problems \cite{aggarwal2021diagnostic, chen2021deep}, they have been found to be  vulnerable to adversarial attacks, i.e., imperceptible modifications of their inputs that trick the model into making an incorrect prediction \cite{IntruigingProperties}.
%Since their discovery,
{Along with several other vulnerabilities \cite{biggio2018wild},} the discovery of adversarial examples has made the deployment of NNs in real-world, safety-critical applications – such as autonomous driving or healthcare – increasingly challenging. The design and analysis of methods that can mitigate such vulnerabilities of NNs, or provide guarantees for their worst-case behaviour in adversarial conditions, has thus become of critical importance \cite{adams2022formal,wei2022safe}.

\begin{figure}
    \centering
    \subfigure{\includegraphics[width=0.4125\textwidth]{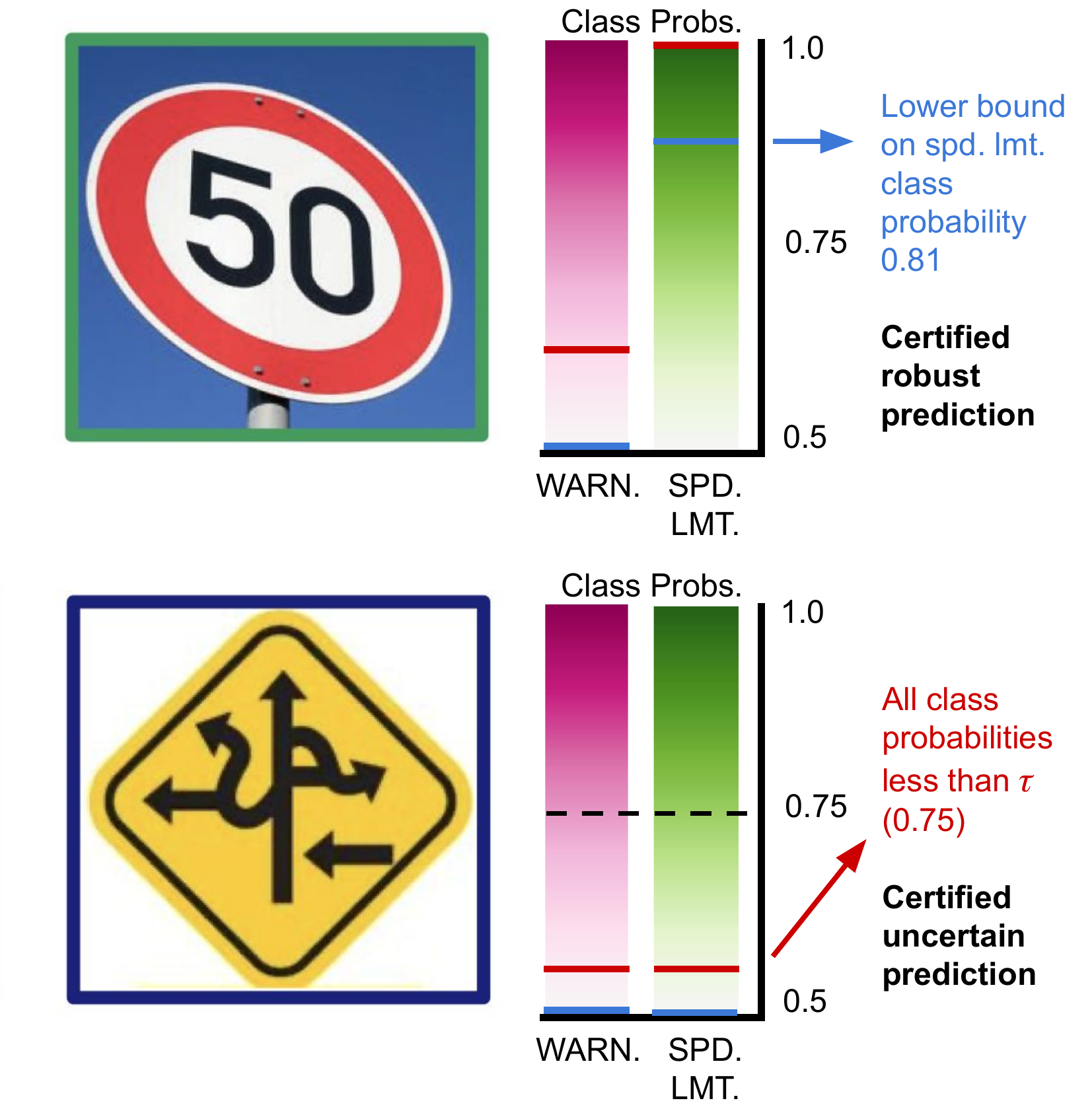}}
    \caption{Certifications for a traffic sign recognition benchmark with two classes: speed limit (spd.\ lmt.) and warning sign (warn.). We plot original images, the upper and lower-bound class probabilities as red and blue horizontal lines, respectively, %blue and red dashes, 
    and a description of the result. \textbf{Top Row:} A 50 km/hr sign from the GTSRB dataset. As the lower bound class probability is 0.81, we certify that all images in the ball are classified correctly as speed limit signs and therefore no adversarial examples exist. \textbf{Bottom Row:} A nonsense traffic sign. As the upper bound probability 
    for all classes is less than a threshold (0.75), we certify that the BNN is uncertain.  }
    \label{fig:GTSRBTitleFig}
\end{figure}

While retaining the advantages intrinsic to deep learning, Bayesian neural networks (BNNs), i.e., NNs with a probability distribution placed over their weights and biases \cite{neal2012bayesian},  enable probabilistically principled evaluation of model uncertainty. Since adversarial examples are intuitively related to uncertainty \cite{kendall2017uncertainties}, the application of BNNs is particularly appealing in safety-critical scenarios. 
In fact, model uncertainty of a BNN can, in theory, be taken into account at prediction time to enable safe decision-making \cite{michelmore2019uncertainty,bekasov2018bayesian,carbone2020robustness,yuan2020gradient}. Various techniques have been proposed for the evaluation of their robustness, including generalisation of gradient-based adversarial attacks (i.e., non-Bayesian) \cite{liu2018adv}, statistical verification techniques \cite{cardelli2019statistical}, as well as approaches based on pointwise (i.e., for a specific test point $x^*$) uncertainty evaluation \cite{smith2018understanding}.
However, to the best of our knowledge, a systematic approach for computing formal (i.e., with certified bounds)  guarantees on the behaviour of BNNs and their decisions against adversarial input perturbations is still missing.

%\LL{Integrate Figure~\ref{fig:GTSRBTitleFig} somewhere}\MRW{Done.}\LL{I would add it more in the experiments paragraph or where we introduce probabilistic robustness}\MRW{Okay, shifted.}
% The need for verification 

%\MK{Need to rework how the properties are defined (safe, robustness, input/output specification, etc), and introduce and refer to them consistently}

%\ap{The following few sentences make it sounds like we are inventing this S-T thing. I think we should point out that is a standard thing}
In this work, we develop a novel algorithmic framework to quantify the adversarial
robustness  of BNNs. %To this end, we adapt the notion of adversarial robustness commonly employed for deterministic NN models to the BNN setting, defined as the invariance of the decision in a small
%neighbourhood of a test point \cite{katz2017reluplex, IntruigingProperties}, and thus study the worst-case effect
%of bounded perturbations. 
In particular, following existing approaches for quantifying the robustness of deterministic neural networks \cite{goodfellow6572explaining, madry2017pgd, gowal2018effectiveness}, we model adversarial robustness as an \emph{input-output specification} defined by a given compact set of input points $T \subseteq  \mathbb{R}^m$ and a given convex polytope output set $S \subseteq \mathbb{R}^n $. A neural network satisfies this specification if all points in $T$ are mapped into $S$, called a safe set. Modelling specifications in this way encompasses many other practical properties such as classifier monotonicity \cite{dvijotham2018dual} and individual fairness \cite{benussi2022individual}. %; however, in this work we focus exclusively on adversarial robustness. 
For a particular specification, we focus on two main properties of a BNN of interest for adversarial prediction settings: \emph{probabilistic robustness} \cite{wicker2020probabilistic, cardelli2019statistical} and \emph{decision robustness} \cite{gowal2018effectiveness, berrada2021make}. %\ap{Cite, Cite, Cite!}. \ap{Imo, you need to make absolutely clear that those are both notions already used in the literature! A key selling point of this paper compared to Deep Mind's one is that we have a unified approach.}
The former, probabilistic robustness, is defined as the probability that a network sampled from the posterior distribution is robust (e.g., satisfies a robustness specification defined by a given $T$ and $S$). Probabilistic robustness attempts to provide a general measure of robustness of a BNN; in contrast, \emph{decision robustness} focuses on the decision step, and evaluates %computes 
the robustness of the optimal decision of a BNN. That is, a BNN satisfies decision robustness for a property if, for all points in $T$, the expectation of the output of the BNN in the case of regression, or the argmax of the expectation of the softmax w.r.t.\ the posterior distribution for classification, are contained in $S$. 

Unfortunately, evaluating probabilistic and decision robustness for a BNN is not trivial, as it involves computing distributions and expectations of high-dimensional random variables passed through a non-convex function (the neural network architecture). 
Nevertheless, we derive a unified algorithmic framework based on computations over the BNN weight space
that yields \emph{certified lower} and \emph{upper bounds} for both properties.
Specifically, we show that probabilistic robustness is equivalent to the measure, w.r.t.\ the BNN posterior, of the set of weights for which the resulting deterministic NN is \emph{robust}, i.e., it maps all points of $T$ to a subset of $S$.  Computing upper and lower bounds for the probability involves sampling compact sets of weights according to the BNN posterior, and {propagating} each of these weight sets, $H$, through the neural network architecture, jointly with the input region $T$, to check whether all the networks instantiated by weights in ${H}$ are safe.
To do so, we generalise bound propagation techniques developed for deterministic neural networks to the Bayesian settings and instantiate explicit schemes for Interval Bound Propagation (IBP) and Linear Bound Propagation (LBP) \cite{gehr2018ai2}.
Similarly, in the case of decision robustness, we show that formal bounds can be obtained by partitioning the weight space into different weight sets, and for each weight set ${J}$  of the partition we again employ bound propagation techniques to compute the maximum and minimum of the decision of the NN for all input points in $T$ and different weight configurations in ${J}$. The resulting extrema are then averaged according to the posterior measure of the respective weight sets to obtain sound lower and upper bounds on decision robustness. 

We perform a %\MK{since you say systematic, I was expecting to see tables of computational performance for different network sizes} 
systematic experimental investigation of our framework on a variety of tasks. We first showcase the behaviour of our methodology on a classification problem from an airborne collision avoidance benchmark \cite{julian2019guaranteeing} and on two  safety-critical industrial regression benchmarks \cite{ucidata}. %In both cases we show that our method can successfully certify BNNs against adversarial perturbations.
We then consider image recognition tasks and illustrate how  our method can scale to verifying BNNs on medium-sized computer vision problems, including MNIST and a two-class subset of the German Traffic Sign Recognition Benchmark (GTSRB) dataset \cite{stallkamp2012gtsrb}. %,  also accounting for uncertainty. 
On small networks, such as those used for airborne collision avoidance ($\sim \hspace{-0.2em} 5000$ parameters), our method is able to verify key properties in under a second, thus enabling comprehensive certification over a fine partition of the entire state space. Moreover, when employed in conjunction with adversarial training \cite{wicker2021bayesian}, we are able to obtain non-trivial certificates for convolutional NNs (471,000 parameters) on full-colour GTSRB images (2,352 dimensions).\footnote{An implementation to reproduce all the experiments can be found at: \url{https://github.com/matthewwicker/AdversarialRobustnessCertificationForBNNs}.}  As an example, we demonstrate the bounds on decision robustness in Figure~\ref{fig:GTSRBTitleFig}, where we plot the upper and lower bound class probabilities (in red and blue respectively) for a BNN trained on a two-class traffic sign recognition benchmark. The bounds are computed for all images within a $\ell_\infty$ ball with radius 2/255 of the two images in the left column of the figure. For the top image of a speed limit sign, our lower bound allows us to verify that the all images within the 2/255 are correctly classified by the BNN as a 50 km/hr sign. For the bottom image of a nonsense traffic sign, our upper bound allows us to verify that the BNN is uncertain for this image and all images in the ball.

In summary, this paper makes the following contributions.

\begin{itemize}
    \item  We present  an algorithmic framework based on convex relaxation techniques for the robustness analysis of BNNs in adversarial settings. 
    \item We derive explicit lower- and upper-bounding  procedures based on IBP and LBP for the propagation of input and weight intervals through the BNN posterior function.
    \item We empirically show that our method can be used to certify BNNs consisting of multiple hidden layers and with hundreds of neurons per layer. 
\end{itemize}
A preliminary version of this paper appeared as \cite{wicker2020probabilistic}. This work extends \cite{wicker2020probabilistic} in several aspects. 
In contrast to \cite{wicker2020probabilistic}, which focused only on probabilistic robustness, here we also tackle decision robustness and embed the calculations for the two properties in a common computational framework. Furthermore, while the method in \cite{wicker2020probabilistic} only computes lower bounds, in this paper we also develop a technique for upper bounds computation. Finally, we substantially extend the empirical evaluation to include additional datasets, 
%Compared to the conference paper, in this work, we introduce several new contributions, including: algorithms for analysis of decision robustness; explicit bound formulation for sample-based posteriors; and an improved method for computing posterior probability of safe weight sets. Further, the experimental evaluation has been consistently extended, to include, 
 evaluation of convolutional architectures and scalability analysis, as well as %\MK{verification not defined, as it was not used before}\MRW{You are right we only describe our method as certification throughout now.}
 certification of  out-of-distribution uncertainty.

%The remainder of the paper is structured as follows. In the next section, we review pertinent works related to certification of neural networks and adversarial robustness of Bayesian machine learning methods. In Section~\ref{sec:background} we give an overview of Bayesian neural networks, and we define the key concepts of probabilistic robustness and decision robustness for Bayesian neural networks in Section~\ref{sec:problemstatement}. In Section~\ref{sec:theorethical_bounds} we demonstrate theoretically how the robustness quantities can be bounded and how these bounds yield sound certification for BNNs. Practical algorithms to compute these bounds, including convex relaxation, are given in Section~\ref{sec:methodology} and Section~\ref{sec:algoirthm}. Finally, we validate our certification framework on a number of datasets in Section~\ref{sec:experiments}.

\section{Related Work}
%\ap{Style note: stuff like "Blah blah. [9] tackles a similar property" - is imho very ugly and difficult to read. The text should be readable withouth reading the brakets - at least put something like ". In [9]"}

%\LL{Maybe I would start with this paragraph describing potential of BNNs and Bayesian models for adv examples. Then, you can have the paragraph on certification of NN where you can focus especially on BNNs after an introductory paragraph on NN verification.}

%\LL{Let's be careful to compare with deepmind. In their empirical evaluation they do compare with our approach showing how their method can be more precise but it is generally more computationally demanding.}\MW{Yes I agree. Perhaps reviewers will ask for this comparison and then we can give it. We need to submit soon otherwise this paper dies in internal revision.}

% BNNs as more robust methods
Bayesian uncertainty estimates have been shown to empirically flag adversarial examples, often with remarkable success \cite{rawat2017adversarial, smith2018understanding}.
These techniques are, however, empirical and can be circumvented by specially-tailored attacks that also target the uncertainty estimation \cite{carlini2017cwattack}. % \ap{Citation missing. I remember there was a work by Carlini that looked into this.}\MRW{Added.} 
%Additionally, BNNs have been shown to be more resistant to adversarial attacks \cite{bekasov2018bayesian, carbone2020robustness}. %Gradient-based attacks for BNNs have been developed in \cite{liu2018adv, zimmermann2019comment}. 
Despite these attacks, it has been shown that BNN posteriors inferred by Hamiltonian Monte Carlo tend to be more robust to attacks than their deterministic counterparts \cite{bekasov2018bayesian}. Further, under idealised conditions of infinite data, infinitely-wide neural networks and perfect training, BNNs are provably robust to gradient-based adversarial attacks \cite{carbone2020robustness}.  % show that in a noiseless setting BNN posteriors \MK{HMC not defined}\MRW{Updated} inferred by Hamiltonian Monte Carlo are more robust to attacks than their counterpart. \MK{reword, unclear} \MRW{Updated} 
%In \c, it is shown that 
%gradient-based methods provably fail to attack BNNs but only this only holds for under idealised conditions including infinite data, exact inference and infinitely wide, fully-trained BNNs \cite{carbone2020robustness}. 
%\MK{not clear what 'purpose' means} \MRW{Updated}
However, while showing that BNNs are promising models for defending against adversarial attacks, the arguments in \cite{bekasov2018bayesian} and \cite{carbone2020robustness} do not provide concrete bounds or provable guarantees for when an adversarial example does not exist for a given BNN posterior.

%However, methods for flagging adversarial examples via uncertainty and for computing adversarial examples are empirical and cannot guarantee the existence or absence of adversarial points. Moreover, the theoretical argument in \cite{carbone2020robustness} does not provide us with practical bounds for when adversarial examples do not exist for a given BNN posterior.  
%However, these approaches only consider pointwise uncertainty estimates, that is, specific to a particular test point.
In \cite{cardelli2019statistical, michelmore2019uncertainty}, the authors tackle similar properties of BNNs to those discussed in this paper. Yet these methods only provide bounds on probabilistic robustness and the bounds are \textit{statistical}, i.e., only valid up to a user-defined, finite probability $1-\delta$, with $\delta > 0$. In contrast, the method in this paper covers both probabilistic and decision robustness and computes bounds that are sound for the whole BNN posterior (i.e., hold with probability $1$).   % However, their methods are \MK{this needs to be finessed - you also use sampling, so why are your bounds formal? You also use Bonferoni bounds, which are statistical. So what is the difference? Is it in confidence bound guarantees?} statistical in nature, and hence in practice cannot provide formal guarantees on the obtained bounds, which are important for safety-critical applications.
%In contrast, the approach we develop in this paper builds on convex relaxation techniques and iterative sampling from the posterior to compute certified guarantees over the BNN behaviour in adversarial settings. 
%\LL{From here I would move to the previous paragraph, as in that paragraph we were already discussing about verification of BNNs.}
In \cite{wicker2021bayesian}, the authors incorporate worst-case information via bound propagation into the likelihood in order to train BNNs that are more adversarially robust; while that work develops a principled defense for BNNs against attack, it does not develop or study methods for analyzing or guaranteeing their robustness.   %\ap{But...? Otherwise it seems you are just adding the self-citation for the sake of it}. 
%\MK{unclear what this is saying}\MK{you don't state what you mean by 'specification'}\MRW{Removed as it is not fully relevant.}

Since the publication of our preliminary work    \cite{wicker2020probabilistic}, the study of  \cite{berrada2021verifying} has further investigated certifying the posterior predictive of BNNs. The definition in \cite{berrada2021verifying} corresponds to a subset of what we refer to as \textit{decision} robustness, but their method only applies to BNNs whose posterior support has been clipped to be in a finite range. Here, we pose a more general problem of certifying decision and probabilistic robustness of BNNs, 
and can handle posteriors on continuous, unbounded support, which entails the overwhelming majority of those commonly employed for BNNs.  Furthermore, following the preliminary version of this paper \cite{wicker2020probabilistic}, %the preliminary version of this work,
\cite{wicker2021certification} introduced %used 
a technique for probabilistic robustness certification 
implemented via a recursive algorithm  the %inside of a backwards recursion algorithm 
that operates over the state-space of a model-based control scenario. %Similarly, 
\cite{lechner2021infinite} uses similar methods to those presented in \cite{wicker2020probabilistic} %the preliminary version of this paper, 
to study infinite-time horizon robustness properties of BNN control policies by checking for safe weight sets and modifying the posterior so that only safe weights have non-zero posterior support. %\ap{I have not read these two papers but we should really careful how we phrase this! At the moment, it sounds like they do something more general than us, and that that is that! Do they have any assumption we don't have? Do they tackle decision and probabilistic? Is their method fully formal? Be specific otehrwise this sentence cuts the ground off of our paper.}%\LL{Yes, this must be improved. Also because if I remember correctly, those methods only consider a fixed and unique weight interval and propagate that through the network. Of course, that will not work in many setting where a single interval to contain most of the probability mass would be too large to be propagated through the neural network due to the conservativism of IBP. }%\MK{Cite Henzinger NeurIPS paper? UAI 2021?} \MRW{Added,}
%In contrast, probabilistic robustness aims at  computing the uncertainty for compact subspaces of
%input points, thus taking into account worst-case adversarial perturbations of the input point when considering its neighbourhood. A probabilistic property analogous to that considered in this paper has been studied for BNNs in \cite{cardelli2019statistical, michelmore2019uncertainty}.
%However,  the solution methods presented in \cite{cardelli2019statistical, michelmore2019uncertainty} are statistical, with confidence bounds,  and in practice cannot give certified guarantees for the computed values, which are important for safety-critical applications.
%Our approach, instead building on non-linear optimisation relaxation techniques, computes a certified lower bound.

%\LL{Would not this paragraph go before of the previous one logically?}
Most existing certification methods in the literature
are designed for deterministic NNs.
Approaches studied %in the literature 
include abstract interpretation \cite{gehr2018ai2}, mixed integer linear programming \cite{tjeng2017evaluating, raghunathan2018semidefinite, dvijotham2018verification, wong2018provable}, game-based approaches \cite{wicker2018feature, wu2018game}, and SAT/SMT \cite{katz2017reluplex, julian2019guaranteeing}.
In particular, \cite{weng2018towards,zhang2018efficient,gowal2018effectiveness}  employ relaxation techniques from non-convex optimisation to compute guarantees over deterministic neural network behaviours, specifically using Interval Bound Propagation (IBP) and Linear Bound Propagation (LBP) approaches. %that have been employed for certification of deterministic neural networks.
However, these methods cannot be used for BNNs because they all assume that the weights of the networks are deterministic, i.e., fixed to a given value, while in the Bayesian setting we need to certify the BNN for a continuous range of values for weights that are not fixed, but distributed according to the BNN posterior.

%GP methods
In the context of Bayesian learning, methods to compute adversarial robustness measures have been explored for Gaussian processes (GPs), both for regression \cite{cardelli2018robustness} and classification tasks  \cite{smith2019adversarial,patane2021adversarial}.
However, because of the non-linearity in NN architectures, GP-based approaches cannot be directly employed for BNNs. Furthermore, the vast majority of approximate Bayesian inference methods for BNNs do not employ Gaussian approximations over latent space \cite{blundell2015weight}. 
%the vast majority of inference methods employed for BNNs do not have a Gaussian approximate distribution in latent/function space, because of the non-linearities in NN architectures that lead to a non-Gaussian predictive distribution for BNNs \cite{blundell2015weight}. 
%Hence, certification techniques for GPs cannot be directly employed for BNNs. 
%This also includes BNNs trained with variational inference, where a Gaussian distribution is assumed in the weight space (e.g., Bayes by Backprop \cite{blundell2015weight}).
%In fact, because of the non-linearity of the BNN structure, even in this case the distribution of the BNN is not Gaussian in function space. \LL{Delete above sentence}
%Though methods to approximate BNN inference with GP-based inference have been proposed \cite{emtiyaz2019approximate}, the guarantees obtained on a GP aimed to approximate a BNN would not provide us with sound guarantees on the true BNN. % would apply to the approximation and not the actual BNN, and would not provide error bounds. \LL{Above sentence is not very clear}
%without guarantees on the approximation error.
%Thus, methods designed for GPs can only be applied approximately, and without guarantees on the approximation error.
In contrast, our method is specifically tailored to take into account the non-linear nature of BNNs and can be directly applied to a range of approximate Bayesian inference techniques used in the literature.

\section{Background}\label{sec:background}

In this section, we overview the necessary background and introduce the notation we use throughout the paper. We focus on neural networks (NNs) employed in a supervised learning scenario, where we are given a dataset of $n_{\mathcal{D}}$ pairs of inputs and labels, $\mathcal{D} = \{(x_i, y_i)\}_{i=1}^{n_{\mathcal{D}}}$, with $x_i \in \mathbb{R}^m$, and where each target output $y \in \mathbb{R}^{n}$ is %\ap{this was $n_c$ in the intro, check for consistency of notation throughout}
either a one-hot class vector for classification or a real-valued vector for regression. %\LL{Define $y$. Is it regression or classification?} 
%We then proceed to define the local robustness problem. %and briefly discuss problems with direct application of prior definitions to the Bayesian setting. 

\subsection{Bayesian Deep Learning}

Consider a feed forward neural network $f^w:\mathbb{R}^{m}\to\mathbb{R}^n$, parametrised by a vector  $w \in \mathbb{R}^{n_w}$ containing all its weights and biases. 
We denote with $f^{w,1},...,f^{w,K}$ the $K$ layers of $f^w$ and % and with W_i$ or $b_i$. %,
%Given a NN $f^w$ composed of $K$ layers, we denote by $f^{w,1},...,f^{w,K}$ the layers of $f^w$ 
%and by $w=\big(\{W_{i}\}_{i=0}^{K-1}\big) \cup \big(\{b_{i}\}_{i=0}^{K-1}\big)$ the aggregate vector of the neural network's parameters, where $W_i$ and $b_i$ represent weights and biases of the $i$th layer of $f^w$. %We denote fully-connected and convolutional neural network architectures respectively with FCNN and CNN. 
take the activation function of the $i$th layer to be $\sigma^{(i)}$, abbreviated to just $\sigma$ in the case of the output activation.\footnote{We assume that the activation functions have a finite number of inflection points, which holds for activation functions commonly used in practice \cite{GoodfellowBook}.}
%
%We assume that all activation functions of the NN are continuous monotonic, \LL{Do we need monotonicuty?}\MRW{Yes. Otherwise convex relaxations do not work.}\LL{Well, not really. As long as it is continuous I can always find a linear upper and lower bound.}
%which holds for the vast majority of activation functions used in practice such as sigmoid, ReLu, and tanh \cite{deepbook}. This guarantees that $f^w$ is a continuous function. 
Throughout this paper, we will use $f^w(x)$ to represent pre-activation of the last layer. % and the output of the last layers activation function to be $\sigma(f^w(x))$.

%\paragraph{Bayesian Learning}

Bayesian learning of deep neural network starts with a prior distribution, $p_{\mathbf{w}}(w)$, over the vector of random variables associated to the weights, $\mathbf{w}$.
%with $\mathbf{w}$ being the vector of random variables associated to the parameter vector $w$.
Placing a distribution over the weights defines a stochastic process indexed by the input space, which we denote as $f^{\mathbf{w}}$.
After the data set  $\mathcal{D}$ has been observed, the BNN prior distribution is updated according to the likelihood, $p(\dataset | w) = \prod_{i=1}^{n_\dataset} p(y_i | x_i, w)$, which models how likely (probabilistically speaking) we are to observe an output under the stochasticity of our model parameters and observational noise \cite{bishop1995neural}. The likelihood function,
$p(y_i | x_i, w)$, generally takes the shape of a softmax for multiclass classification and a multivariate Gaussian for regression.
The posterior distribution over the weights given the dataset is then computed by means of the Bayes formula, i.e., $p(w | \mathcal{D}) \propto p(\mathcal{D} | w) p(w)$.
The cumulative distribution of $\post$ we denote as $P(\cdot)$, so that for $R \subseteq \mathbb{R}^{n_w}$ we have:
\begin{align}\label{eq:cumulative}
\int_R \post dw = P(R).
\end{align}
%The Bayesian posterior, denoted $p(w | \mathcal{D})$, is then computed by integrating over the space of parameters with respect to the prior and a likelihood: $p(w | \mathcal{D}) \propto p(\mathcal{D} | w) p(w)$. 
%\LL{Define it} 
The posterior $p(w | \mathcal{D})$ is in turn  used to calculate the output of a BNN on an unseen point, $x^*$. The distribution over outputs is called the posterior predictive distribution and is defined as: %predict the output for unseen inputs through the posterior predictive distribution:
\begin{align}\label{eq:pred_distr}
    p(y^* | x^*, \mathcal{D}) = \int p(y^{*} | x^{*}, w)p(w | \mathcal{D}) dw.
\end{align}
Equation \eqref{eq:pred_distr} defines a distribution over the BNN output.
When employing a Bayesian model, the overall final prediction %\LL{Not clear what the ultimate prediction is? I guess here you are talking about decision?} 
is taken to be a single value, $\hat{y}$, that minimizes the Bayesian risk of an incorrect prediction according to the posterior predictive distribution and a loss function $\mathcal{L}$. Formally, the final decision of a BNN is computed as 
\begin{align*}
\hat{y} = \argmin_{y} \int_{\mathbb{R}^{n}} \mathcal{L}(y, y^*) p(y^* | x^*, \mathcal{D}) dy^*.
\end{align*}
This minimization is the subject of Bayesian decision theory \cite{berger2013statistical}, and the final form of $\hat{y}$ clearly depends on the specific loss function $\mathcal{L}$ employed in practice. In this paper, we focus on two standard loss functions widely employed for classification and regression problems.\footnote{In Appendix  \ref{appendix:decisionrules} we discuss how our method can be generalised to other losses  commonly employed in practice.} 
\paragraph{Classification Decisions} The 0-1 loss, $\ell_{0-1}$, assigns a penalty of 0 to the correct prediction, and 1 otherwise.
It can be shown that the optimal decision in this case is given by the class for which the predictive distribution obtains its maximum, i.e.:
\begin{align*}
    \hat{y} = \argmax_{i=1,\ldots,n} p_i(y^* | x^*, \mathcal{D}) = \argmax_{i=1,\ldots,n} \mathbb{E}_{w \sim \post}\left[ \sigma_i(f^w(x)) \right],
\end{align*}
where $\sigma_i$ represents the
$i$th output component of the softmax function.
\paragraph{Regression Decisions} The $\ell_2$ loss assigns a penalty to a prediction according to its $\ell_2$ distance from the ground truth. It can be shown that the optimal decision in this case is given by the expected value of the BNN output over the posterior distribution, i.e.:
\begin{align*}
    \hat{y} = \mathbb{E}_{w \sim \post}\left[f^{w}(x)\right].
\end{align*}
%\ap{As discussed, this is where we need a bit of introduction on decision theory for regression and classification so to properly introduce decision robustness and its relation to predictive distributions.}
%\LL{You are already citing \cite{berger2013statistical}, so no need to say that further details are there. Also, we cannot have an appendix (paper must be self contained within the page limit)}Details on Bayesian decision theory are given in  \cite{berger2013statistical} and a further discussion of its relevance to our problem setting is given in the Appendix. 
%\ap{Since we make a distinction between decision robustness and probabilistic robustness, we should here talk about model decision, and set clear notation for decision of the model vs. predictive posterior.}
%\ap{In doing this we can also make clear that $f$ does not include the final activation.}

Unfortunately, because of the non-linearity of neural
network architectures, the computation of the posterior distribution over weights, $p(w | \mathcal{D})$, is generally intractable \cite{neal2012bayesian}. Hence, various approximation methods have been studied to perform inference with BNNs in practice. Among these methods, we will consider   Hamiltonian Monte Carlo (HMC) \cite{neal2012bayesian} and Variational Inference (VI)  \cite{blundell2015weight}, which we now briefly describe. %in the following two paragraphs.

\subsubsection{Hamiltonian Monte Carlo (HMC)} HMC proceeds by defining
a Markov chain whose invariant distribution is $p_{\mathbf{w} } (w \vert \dataset),$ and
relies on Hamiltonian dynamics to speed up the exploration
of the space. Differently from VI discussed below, HMC does not make parametric assumptions on
the form of the posterior distribution, %\ap{I am no expert of this but it might make some assumption (e.g., finite variance or something similar), worth checking this with Matthew, maybe we can simply write "does not assume a particular analytical form for the posterior"?}
and is asymptotically
correct \cite{neal2012bayesian}. The result of HMC is a set of samples that approximates  $p_{\mathbf{w} } (w \vert \dataset)$.

\subsubsection{Variational Inference (VI)} VI proceeds by finding a Gaussian approximating distribution over the weight space $q(w)\sim p_{\mathbf{w} } (w \vert \dataset)$ 
in a trade-off between approximation accuracy and scalability. The core idea is that $q(w)$ depends on some hyperparameters that are then iteratively optimized by minimizing
a divergence measure between $q(w)$ and $p_{\mathbf{w} } (w \vert \dataset)$. Samples
can then be efficiently extracted from $q(w)$.

For simplicity of notation, in the rest of the paper we will indicate with $\post$ the posterior distribution estimated by either of the two methods, and clarify the methodological differences when they arise.

\begin{figure*}
    \centering
    \subfigure{\includegraphics[width=0.82\textwidth]{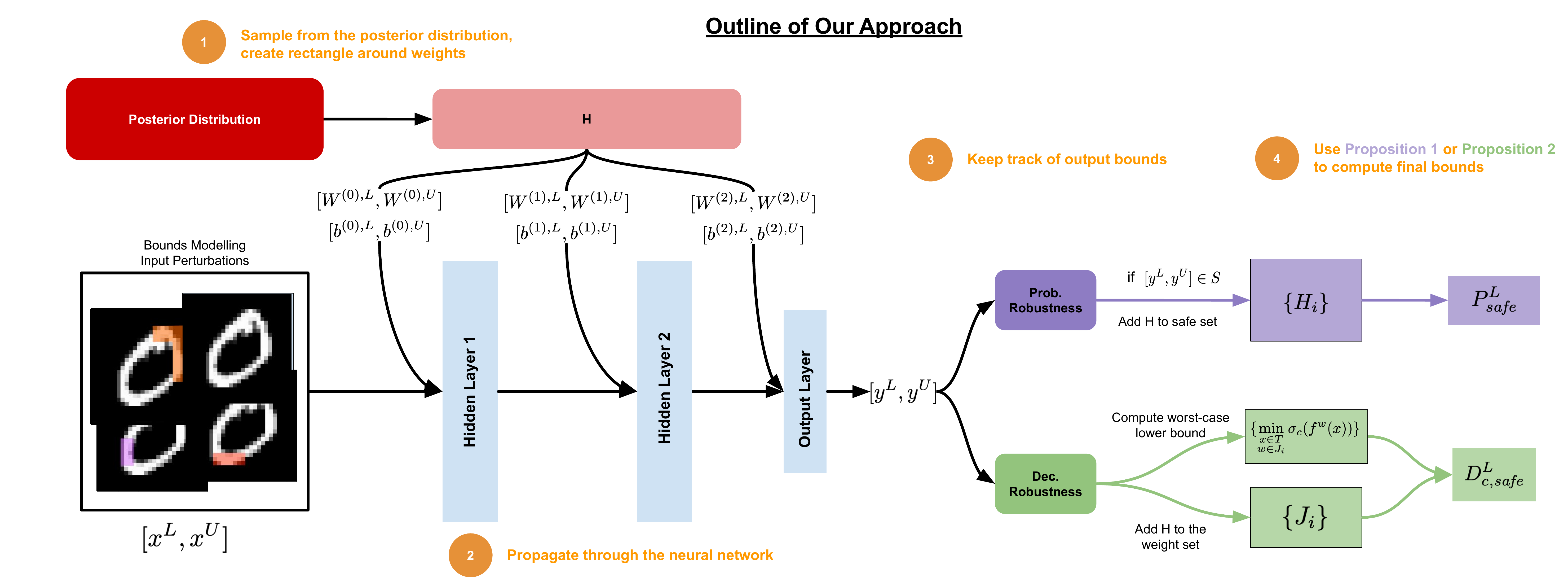}}
    \caption{
    %\MK{Improve by making capitalisation consistent and with the rest of the paper, define concepts and make sure the definitions are consistent (D notation is not)} 
    A diagram illustrating a single iteration of the computational flow for the certification process of a BNN w.r.t.\ decision robustness (green) and probabilistic robustness (purple). This process is summarised %further described by 
    in Algorithm~\ref{alg:psafealgorithm}.  %The first step is to sample a weight from the posterior and build a rectangle around this weight. The second step is to perform bound propagation with respect to the input specification and sampled weight rectangle. The third step is to keep track of safe weights in the case of probabilistic robustness and of all weights sets along with their corresponding worst-case output in the case of decision robustness. The final step is to combine these tracked values with probability computations to arrive at bounds on the specification of interest. This process is further described by Algorithm 1.
    }
    \label{fig:DiagramComputations}
\end{figure*}

\section{Problem Statements}\label{sec:problemstatement}

We focus on local specifications defined over an input compact set $T \subseteq \mathbb{R}^m$ %\LL{Note that in some the experiments on the VCAS-XU (at least in the conference version) we had a non-compact input set}\ap{I think they were also compact. Maybe not connected but definitely compact, as in closed and bounded}, 
and output set $S\subseteq \mathbb{R}^{n}$ in the form of a convex polytope:
%In order to leverage bound propagation techniques for certification we assume that $S$ is a convex polytope:
%
%We consider two notions to analyzing the sensitivity of BNN outputs. The first, called \textit{probabilistic robustness}, considers the probability that a NN sampled from the BNN posterior distribution is susceptible to adversarial attacks. The second, called \textit{decision robustness}, ensures that the final decision of the BNN, typically given by the mean of the predictive distribution, is not vulnerable to adversarial attacks.\footnote{Generalization to other decision theoretic quantities is given in the Appendix.}\LL{you already describe this in the previous section. also there is not this discussion mentioned in the footnote in appendix}  %\LL{This is vague becasue we consider the expectation of the predictive posterior, so I would try to be more clear.} \MRW{I have tried to clear this up, let me know if there is anything specifically still unclear here :)}
%
%Throughout these definitions we will always take local robustness properties to be flexibly modelled as input-output constraints.
%\LL{not clear what you mean} \MRW{Which part is unclear?}
%These constraints come in the form of compact input sets  $T \subseteq \mathbb{R}^m$ in the input space, and a set $S\subseteq \mathbb{R}^{n}$. In order to leverage bound propagation techniques for certification we restrict the form of $S$:
%\begin{assumption}\label{assumption}
%We assume that ${S}$ is described by $n_S$ linear constraints on the values of the final layer of the BNN, that is, 
\begin{align}\label{def:S}
S=\{y\in \mathbb{R}^{n} \, |  \, C_S y + d_S \geq 0 \} ,
\end{align} 
where $C_S \in \mathbb{R}^{n_S\times n}$ and $d_S\in \mathbb{R}^{n_S}$ are the matrix and vector encoding the polytope constraints, and $n_S$ is the number of output constraints considered.
For simplicity of presentation, we assume that $T$ is defined as a box (axis-aligned linear constraints).\footnote{Note that, where a specification is not in this form already, one can first compute a bounding box $R = [x^L,x^U]$ (or a finite sequence of them) such that $T \subseteq R$, and then proving that the output specification holds for $R$ also proves that it holds for $T$. In the case that we do not prove that an output specification holds, then we cannot guarantee it is violated by nature of our method being sound but not complete. %\MK{But what if it does not? Need to discuss, also state minimality}
} However, we stress that all the methods in this paper can be extended to the more general case where $T$ is a convex polytope.
%\end{assumption}
Our formulation of input-output specifications can be used to capture important properties such as classifier monotonicity \cite{stanforth2018dual} and individual fairness \cite{benussi2022individual}, but in this work we focus exclusively on adversarial robustness.  Targeted adversarial robustness, where one aims to force the neural network into a particular wrong classification, is captured in this framework by setting $T$ to be an over-approximation of an $\ell_p$ ball around a given test input, and setting $C_S$ to an $n_S \times n$ matrix of all zeros with a $-1$ entry in the diagonal entry corresponding to the true class and a $1$ on the diagonal entry corresponding to the target class or classes. For regression, one uses $C_S$ to encode the absolute deviation from the target value and $d_S$ to encode the maximum tolerable deviation.
%\LL{For the case above, in the paper, do we handle it specifically also for regression? }\MRW{Yes. Added note}
Throughout the paper we will refer to an input-output set pair, $T$ and $S$, as defined above as a \textit{robustness specification}. %A BNN is \textit{safe} if it satisfies the definitions of robustness presented below.  %\MRW{Maybe this makes things more clear in terms of the terminology I intended? Let me know :)}
\subsection{Probabilistic Robustness}
Probabilistic robustness accounts for the probabilistic behaviour of a BNN in adversarial settings. 
\begin{definition}[Probabilistic robustness]\label{Def:ProbSafety}
Given a Bayesian neural network $f^{\mathbf{w}}$, an input set $T \subseteq \mathbb{R}^{m}$ and a set $S\subseteq \mathbb{R}^{n}$ of safe outputs, then probabilistic robustness is defined as 
\begin{align}
\label{Eq:ProbSafety}
    P_{\text{safe}}(T,S):=\postprob(\forall x \in T, f^{{w}}(x) \in S).
\end{align}
Given $\eta \in [0,1]$, we then say that $f^{\mathbf{w}}$ is probabilistically robust, or safe, for specifications $T$ and $S$, with probability at least $\eta$ iff:
\begin{align*}
    P_{\text{safe}}(T,S) \geq \eta
\end{align*}
\end{definition}
Probabilistic robustness considers the adversarial behaviour of the model while accounting for the uncertainty arising from the posterior distribution. 
$P_{\text{safe}}(T,S)$ represents the (weighted) proportion of neural networks sampled from $f^{\mathbf{w}}$ that satisfy a given input-output specification (captured by $T$ and $S$) and can be used directly as a measure of compliance for Bayesian neural networks. 
As such, probabilistic robustness is %a measure 
particularly suited to quantification of the robustness of a BNN to adversarial perturbations \cite{berrada2021make, michelmore2019uncertainty, de2021adversarial}.
%a particular instance of notions of probabilistic safety generally employed for the formal analysis of stochastic processes \cite{abate2008probabilistic,9003411}. 
%\LL{This discussion is not super strong, as for regression and the setting described above, you would want to compute bounds on the predictive posterior and not on (1).}\ap{Agree – to do is a discussion that makes more sense, maybe talking about uncertainty and stuff. It is important that we stress that we are not making stuff up, but are using notions used in the literature – otherwise it becomes a book not a paper.}
Exact computation of $P_{\text{safe}}(T,S)$ is hindered by both the size and non-linearity of neural networks. As $P_{\text{safe}}(T,S)$ cannot be computed exactly for general BNNs, in this work we tackle the problem of computing provable bounds on probabilistic robustness. 
\begin{problem}[Bounding probabilistic robustness]\label{problem:probabilistic_safety}
Given a Bayesian neural network $f^{\mathbf{w}}$, an input set $T \subseteq \mathbb{R}^{m}$ and a set $S\subseteq \mathbb{R}^{n}$ of safe outputs, compute (non-trivial) $P_{\text{safe}}^L$ and $P_{\text{safe}}^U$ such that
\begin{align}
    P_{\text{safe}}^L \leq P_{\text{safe}}(T,S) \leq P_{\text{safe}}^U.
\end{align}
\end{problem}
%\LL{Following discussion seems more for related work}
We highlight %stress 
the difference between this problem definition and those discussed in prior works \cite{cardelli2019statistical, michelmore2019uncertainty}. In particular, prior works compute upper and lower bounds that hold with probability $1-\delta$ for some pre-specified $\delta$. 
While such statistical bounds can provide an estimation for $P_{\text{safe}}(T,S)$, these may not be sufficient in safety-critical contexts where strong, worst-case guarantees over the full behaviour of the BNN are necessary. The problem statement above holds with probability $1$ and represents sound guarantees on $P_{\text{safe}}(T,S)$.

\subsection{Decision Robustness}
While $P_{\text{safe}}$ attempts to measure the compliance of all functions in the support of a BNN posterior, we are often interested in evaluating robustness w.r.t.\ a specific decision. In order to do so, we consider \emph{decision robustness}, which is computed over the final decision of the BNN. %\ap{Not directly though for the way it's defined. Fix.} \MW{This comment is unclear. I tried to address it. Let me know if it is not addressed}
%The decision from the BNN is computed using by minimizing a loss function over the posterior predictive distribution.
In particular, given a loss function and a decision $\hat{y}$ we have the following.
%\LL{Note that here you say for a general predictive posterior, but then in the definition you just consider classification}\MRW{This isnt true, $\sigma$ is the output not just for classification. I see how this was unclear.}
\begin{definition}[Decision robustness]\label{Def:DecSafety}
Consider a Bayesian neural network $f^{\mathbf{w}}$, an input set $T \subseteq \mathbb{R}^{m}$ and a set $S\subseteq \mathbb{R}^{n}$ of safe outputs. Assume that the decision for a loss $\mathcal{L}$ for $x \in \mathbb{R}^m$ is given by $\hat{y}(x)$. %\LL{Note that this set $S$ is semantically different than the one above as in here we have a softmax that maps each components to $[0,1]$, so I would at least specify $S\subseteq [0,1]^{n}$} \MRW{Not sure where you got this from. This isnt the case as $\sigma$ isnt defined this way. Made this more clear at the beginning to avoid confusion with the prior version of the text.}
Then, the Bayesian decision is considered to be robust if:
\begin{align}
    \forall x \in T \quad \hat{y}(x) \in S.
\end{align}
\end{definition}
%Notice that, decision robustness is a direct analogue to the adversarial robustness studied for deterministic neural networks \cite{}.
%
We notice that, since the specific form of the decision depends on the loss employed in practice, the definition of decision robustness takes different form depending on whether the BNN is used for classification or for regression. 
In particular, we instantiate the definition for the two cases of standard loss discussed in Section~\ref{sec:background}. % \MK{proper reference} \MRW{added}

In the regression case, using the mean square loss we have that $\hat{y}(x) = \mathbb{E}[f^{w}(x)]$, so that if we find upper and lower bounds on $\mathbb{E}[f^{w}(x)]$ for all $x \in T$, i.e., for $i=1,...,m$:
\begin{align*}
       D_{\text{safe},i}^L &\leq \min_{x \in T} \mathbb{E}\left[ f^{w}_i(x) \right] , \; 
    D_{\text{safe},i}^U \geq \max_{x \in T} \mathbb{E}\left[ f^{w}_i(x) \right],  
\end{align*}
we can then simply check whether these are within $S$. 

For the classification case, where the decision is given by the $\argmax$ of the predictive posterior, note that, in order to check the condition in Definition \ref{Def:DecSafety}, it suffices to compute lower and upper bounds on the posterior predictive in $T$, i.e.:
\begin{align*}
    D_{\text{safe},i}^L &\leq \min_{x \in T} \mathbb{E}\left[ \sigma_i(f^{w}(x)) \right] , \; 
    D_{\text{safe},i}^U \geq \max_{x \in T} \mathbb{E}\left[ \sigma_i(f^{w}(x)) \right], 
\end{align*}
for $i = 1,\ldots,m$.
It is easy to see that the knowledge of $D_{\text{safe},i}^L$ and $D_{\text{safe},i}^U$ for all $i = 1,\ldots,m$ can be used to provide guarantees of the model decision belonging to $S$, as defined in Equation \eqref{def:S}, by simply propagating these bounds through the equations. 
Therefore, for both classification and regression we have to bound an expectation of the BNN output over the posterior distribution, with the additional softmax computations for classification. 
We thus arrive at the following problem for bounding decision robustness.
\begin{problem}[Bounding decision robustness]\label{prob:adv_pred}
Let $f^{\mathbf{w}}$ be a BNN with posterior distribution $\post$. Consider an input-output specification ($T$, $S$) and assume $\mathcal{L} = \ell_{0-1}$ for classification or $\mathcal{L} = \ell_2$ for regression. We aim at computing (non-trivial) lower and upper bounds $D_{\text{safe}}^L$ and $D_{\text{safe}}^U$ such that:
\begin{align*}
    D_{\text{safe}}^L \leq \mathbb{E}[s(f^w(x)) ] \leq D_{\text{safe}}^U \quad \forall x\in T,
\end{align*}
where $s = \sigma$ for classification and $s = \mathbb{I}$ for regression. 
\end{problem}

Note that, while for regression we bound the decision directly, for classification we compute the bounds on the predictive posterior and use these to compute bounds on the final decision.
This is similar to what is done for deterministic neural networks, where in the case of classification the bounds are often computed over the logit, and then used to provide guarantees for the final decision \cite{gowal2018effectiveness}. As with probabilistic robustness, our bounds on decision robustness are sound guarantees and do not have a probability of error as with statistical bounds.

\subsection{Outline of our Approach:} 

We design an algorithmic framework for worst-case and best-case bounds on local robustness properties in Bayesian neural networks, taking account of both the posterior distribution ($P_{\text{safe}}^L$ and $P_{\text{safe}}^U$) and the overall model decision ($D_{\text{safe}}^L$ and $D_{\text{safe}}^U$). %for probabilistic and for decision robustness, that is for given $S$, $T$ and BNN $f^{\mathbf{w}}$ we show how to compute two real values $P_{\text{safe}}^L$ and $D_{\text{safe}}^L$ such that:
%\begin{align*}
%   P_{\text{safe}}(T,S) \geq  P_{\text{safe}}^L, \quad D_{\text{safe}}(T,S) \geq  D_{\text{safe}}^L.
%\end{align*}
First, we show how the two robustness properties of Definitions \ref{Def:ProbSafety} and \ref{Def:DecSafety}  can be reformulated  in terms of computation over weight intervals. 
This allows us to derive a unified approach to
%It will turn out that, though defined over different quantities, 
the bounding of  the robustness of the BNN posterior (i.e., probabilistic robustness) and of the robustness of the overall model decision (i.e., decision robustness) %boils down to the same computational requirements, that is,
that is based on \textit{bound propagation} and \textit{posterior integral} computation over  hyper-rectangles. A visual outline for our framework is presented in Figure~\ref{fig:DiagramComputations}.
We organise the presentation of our framework by first introducing a general theoretical framework for bounding the robustness quantities of interest (Section \ref{sec:theorethical_bounds}). 
We will then show how the required integral computations can be achieved for Bayesian posterior inference techniques commonly used in practice (Section \ref{sec:weight_intervals}).
Hence, we will show how to extend bound propagation techniques to deal with both input variable intervals and intervals over the weight space, and will instantiate approaches based on Interval and Linear Bound Propagation techniques (Section \ref{subsec:convexrelax}). Finally (Section \ref{sec:algoirthm}), we will present an overall algorithm that produces the desired bounds.

\section{Formulating BNN Adversarial Robustness via Weight Sets}\label{sec:theorethical_bounds}

%In this section, we derive a theoretical framework for computing formal bounds on $P_{\text{safe}}(T,S)$ and $D_{\text{safe}}(T,S)$.
%\MK{prune this text, repetition of what said just before}
In this section, we show how a single computational framework can be leveraged to compute bounds on both definitions of BNN robustness. We start by converting the computation of robustness into weight space and then define a family of weight intervals that we leverage to bound the integrations required by both definitions.
Interestingly, we find that the resulting theoretical bounds in both cases depend on the same quantities. Proofs for the main results in this section are presented in Appendix~\ref{appendix:proofs}.

\subsection{Bounding Probabilistic Robustness}
%\ap{Sometimes we use $P_{\text{safe}}(T,S)$ sometimes $P_{\text{safe}}(T,S, \mathbf{w})$, make it consistent.}
We first show that the computation of $P_{\text{safe}}(T,S)$ is %\MK{I don't think you show equivalence}\MRW{This is by definition, nothing needs to be shown here I think, but let me know if it is unclear!} 
equivalent to computing a maximal set of safe weights $H$ such that each network associated to weights in $H$ is safe w.r.t.\ the robustness specification at hand. 
%Equivalently, one could compute the set of unsafe weights $K$. 
%
\begin{definition}[Maximal safe and unsafe sets]\label{Def:SafetyWeight}
    We say that ${H}\subseteq \mathbb{R}^{n_{w}}$ is the maximal safe set of weights %\MK{no restrictions on input/output sets?} \MRW{Nothing new, they are restricted above, do we need to restate that T is compact and S is a polytope/polyhedra?}
    from $T$ to $S$, or simply the maximal safe set of weights,  iff $H=\{ w \in  \mathbb{R}^{n_{w}}\, |\,\forall x \in T, f^{{w}}(x)\in S   \}.$ %\LL{ the iff in this context looks weird. I would say instead We say that $H=\{ w \in  \mathbb{R}^{n_{w}}\, |\,\forall \*x \in T, f^{{w}}(\*x)\in S   \}$ is the maximal set of...} \MRW{Tbh I think it is fine this way. If you feel strongly about it let me know. Also others, feel free to weigh in here.}
    Similarly, we say that ${K}\subseteq \mathbb{R}^{n_{w}}$ is the maximal unsafe set of weights from $T$ to $S$, or simply the maximal unsafe set of weights,  iff $K=\{ w \in  \mathbb{R}^{n_{w}}\, |\,\exists \*x \in T, f^{{w}}(\*x)\not\in S \}.$
\end{definition}
Intuitively, $H$ and $K$ simply encode the input-output specifications $S$ and $T$ in the BNN weight space.
%If $\hat{H}$ is a safe set of weights, then Definition \ref{Def:SafetyWeight} implies that for any  $w\in \hat{H}$  the corresponding neural network is safe, i.e., $\forall x \in T, f^w(x)\in S.$ Then, a trivial consequence of the definition of probabilistic safety is the following proposition.
The following lemma, which trivially follows from Equation~\ref{Eq:ProbSafety}, allows us to directly relate the maximal sets of weights to
%\MRW{Updated and I will add a clarifying statement} \MK{safety or robustness?}
the probability of robustness. %\LL{Either you say where it comes from or you add the proof}

\begin{lemma}
\label{Prop:SafetyComputation}
Let $H$ and $K$ be the maximal safe and unsafe sets of weights from $T$ to $S$. Assume that $w \sim \post$. Then, it holds that
\begin{align}
\label{Eq:IntegralSafety}
    P(H) = \int_{H} \post dw = P_{\text{safe}}(T,S)= \\ 1- \int_{K} \post dw = 1 - P(K) \nonumber . 
\end{align} 
\end{lemma}
%
%\textit{Proof:} $\blacksquare$. 
%
Lemma \ref{Prop:SafetyComputation} simply translates the robustness specification from being concerned with the input-output behaviour of the BNN to an integration on the weight space. %By computing the cumulative density of the safe and unsafe weights, Proposition \ref{Prop:SafetyComputation} tells us that we will arrive exactly at the quantity we desire.
%Proposition~\ref{Prop:SafetyComputation} tells us that computing probabilistic robustness is equivalent to computing the cumulative density of the safe and unsafe weights.

%\LL{Why this sentence is here in the middle of two paragraph? It comes out of the blue}\ap{Not sure where else to put it xD it is needed for the next Corollary.}\MRW{I think it flows nicer above this so I have moved it.}

%\ap{Is not exactly clear what function space and weight space are to the reader at this point. This could be introduced in the background.}
%\ap{It is not clear why you got from maximal sets to normal sets. You should put a sentence explaining that properly.}\MW{We say that computing them is infeasible, do you want a detailed reason as to why this is? If so let me know by reinstating this comment and then we can have a discussion about it.}
An exact computation of sets $H$ and $K$ is infeasible in general. However, we can easily compute subsets of $H$ and $K$. Such subsets can then be used to compute upper and lower bounds on the value of $\psafe$ by considering subsets of the maximal safe and unsafe weights. 
\begin{definition}[Safe and unsafe sets]\label{Def:SafetyWeight2}
 Given a maximal safe set ${H}$ or a maximal unsafe set $K$ of weights, we say that $\hat{H}$ and $\hat{K}$ are a safe and unsafe set of weights from $T$ to $S$ iff $\hat{H} \subseteq H $ and $\hat{K} \subseteq K$, respectively.
\end{definition}
$\hat{H}$ and $\hat{K}$ can include \textit{any} safe and unsafe weights, respectively, without requiring they are \textit{maximal}. Without maximality, we no longer have strict equality in Lemma~\ref{Prop:SafetyComputation}, but instead we arrive at bounds on the value of probabilistic robustness. %Clearly, the closer $\hat{H}$ and $\hat{K}$ resemble $H$ and $K$ (especially around the area of high probability mass for $p(w|\mathcal{D})$),  the tighter the obtained bounds will be. 

We proceed by defining $\hat{H}$ and $\hat{K}$ as the union of a family of disjoint weight intervals, as these can provide flexible approximations of $H$ and $K$. %\MRW{I am a bit confused at the distinction between $\mathcal{H}$ and $\hat{H}$, do we need both?}
That is, we consider $\mathcal{H} = \{H_i\}_{i=1}^{n_H}$, with $H_i= [w^{L,H}_i,w^{U,H}_i]$  and  $\mathcal{K} = \{K_i\}_{i=1}^{n_K}$, with $K_i= [w^{L,K}_i,w^{U,K}_i]$ such that $H_i \subset H$ and $K_i \subset K$, $\hat{H} = \bigcup_{i=1}^{n_H} H_i$, $\hat{K} = \bigcup_{i=1}^{n_K} K_i$, and $H_i \cap H_j = \emptyset $ and $K_i \cap K_j = \emptyset $, for any $i \neq j$.
Hence, as a consequence of Lemma \ref{Prop:SafetyComputation}, and by the fact that $\hat{H} = \bigcup_{i=1}^{n_H} H_i \subset H$ and $\hat{K} = \bigcup_{i=1}^{n_K} K_i \subset K$, we obtain the following.
\begin{proposition}[Bounds on probabilistic robustness]
\label{Prop:SafetyComputation2}
Let $H$ and $K$ be the maximal safe and unsafe sets of weights from $T$ to $S$. Consider two families of pairwise disjoint weight intervals $\mathcal{H} = \{H_i\}_{i=1}^{n_H}$, $\mathcal{K} = \{K_i\}_{i=1}^{n_K}$ , where for all $i$: 
\begin{align}\label{eq:condition_on_sets}
    H_i \subseteq H, \quad K_i \subseteq K.
\end{align}
Let $\hat{H} \subseteq H $ and $\hat{K} \subseteq K$ be non-maximal safe and unsafe sets of weights, with $\hat{H} = \bigcup_{i=1}^{n_H} H_i$ and $\hat{K} = \bigcup_{i=1}^{n_K} K_i$. Assume that $w \sim \post $. Then, it holds that
\begin{align}
\label{Eq:IntegralSafetyIneq}
%   P_{\text{safe}}^L := \int_{\hat{H}} \post dw \leq P_{\text{safe}}(T,S, \mathbf{w}) \\
%   \leq 1-\int_{\hat{K}} \post dw =: P_{\text{safe}}^U,  \nonumber
   P_{\text{safe}}^L :=  \sum_{i=1}^{n_H} P(H_i) \leq P_{\text{safe}}(T,S, \mathbf{w}) 
   \leq 1- \sum_{i=1}^{n_K} P(K_i)=: P_{\text{safe}}^U,  
\end{align} 
that is, $P_{\text{safe}}^L$ and $P_{\text{safe}}^U$ are, respectively, lower and upper bounds on probabilistic robustness. 
\end{proposition}
Through the use of Proposition \ref{Prop:SafetyComputation2} we can thus bound probabilistic robustness by performing computation over sets of safe and unsafe intervals.
Note that the bounds are given in the case where $\mathcal{H}$ and $\mathcal{K}$ are families of pairwise disjoint weight sets. The general case can be tackled by using the Bonferroni bound, which is discussed in Appendix \ref{appendix:bonferroni} for hyper-rectangular weight sets. 

Before explaining in detail how such bounds can be explicitly computed, first, in the next section, we show how a similar derivation leads us to analogous bounds and computations for decision robustness.

\subsection{Bounding Decision Robustness}
%\ap{This section is (1): only specific to classification in its current form; (2) it is weirdly defined only for two-class classification. Some clarifications on indexes and stuff need to be done for multi-class. We need to fix these two things once we fix the definition of decision robustness.}

%We now show how the same techniques used for bounding probabilistic robustness can be employed to obtain a bound on decision robustness, whose explicit computation will rely on the same conditions highlighted in the previous section for probabilistic robustness.

The key difference between our formulation of probabilistic robustness and that of decision robustness is that, for the former, we are only interested in the behaviour of neural networks extracted from the BNN posterior that satisfy the robustness requirements (hence the distinction between $H$- and $K$-weight intervals), whereas for the computation of bounds on decision robustness we need to take into account the overall worst-case behaviour of an expected value computed for the BNN predictive distribution in order to compute sound bounds. As such, rather than computing safe and unsafe sets, we only need a family of weight sets, $\mathcal{J} = \{J_i\}_{i=1}^{n_J}$, and rely on that for bounding $\dsafe$.
We explicitly show how this can be done for classification with likelihood $\sigma$. The bound for regression follows similarly by using the identity function as $\sigma$.

\begin{proposition}[Bounding decision robustness]
\label{prop:pred_posterior_bound}
Let $\mathcal{J} = \{J_i\}_{i=1}^{n_J}$, with  $J_i \subset \mathbb{R}^{n_w}$ be a family of disjoint weight intervals. Let $\sigma^L$ and $\sigma^U$ be vectors that lower and upper bound the co-domain of the final activation function, and $c \in \{1,\ldots,m \}$ an index spanning the BNN output dimension. Define:
\begin{align}
     \hspace{-.5em} D_{\text{safe},c}^L \hspace{-.2em}:= \sum_{i=1}^{n_J}  P(J_i) \min_{\substack{x \in T\\ w \in J_i }} & \sigma_c(f^w(x))  + \label{eq:dsafel} \sigma^L \hspace{-.2em} \left( 1 - \sum_{i=1}^{n_J}  P(J_i)  \right)\\
     \hspace{-.5em} D_{\text{safe},c}^U  \hspace{-.2em} := \sum_{i=1}^{n_J}  P(J_i)\max_{\substack{x \in T\\ w \in J_i }} & \sigma_c(f^w(x))  +\label{eq:dsafeu}  \sigma^U \hspace{-.2em} \left( 1 - \sum_{i=1}^{n_J}  P(J_i) \right). 
\end{align}
Consider the vectors $D_{\text{safe}}^L = [D_{\text{safe},1}^L,\ldots,D_{\text{safe},m}^L]$ and $D_{\text{safe}}^U = [D_{\text{safe},1}^U,\ldots,D_{\text{safe},m}^U]$, then it holds that:
\begin{align*}
    D_{\text{safe}}^L \leq \mathbb{E}_{\post} [\sigma(f^w(x))] \leq D_{\text{safe}}^U \quad \forall x \in T,
\end{align*}
that is, $D_{\text{safe}}^L$  and $D_{\text{safe}}^U$ are lower and upper bounds on the predictive distribution in $T$.
\end{proposition}

Intuitively, the first terms in the bounds of Equations \eqref{eq:dsafel} and \eqref{eq:dsafeu} consider the worst-case output for the input set $T$ and each interval $J_i$, while the second term accounts for the worst-case value of the posterior mass not captured by the family of intervals $\mathcal{J}$ by taking a coarse, overall bound on that region.  The provided bound is valid for any family of intervals $\mathcal{J}$. Ideally, however, the partition should be finer around regions of high probability mass of the posterior distribution, as these make up the dominant term in the computation of the posterior predictive. 
We will discuss in Section \ref{sec:methodology} how we select these intervals in practice so as to empirically obtain non-vacuous bounds. 
%\MK{Discuss bound tighteness?}
%Notice that, as was the case for probabilistic robustness, while the bounds reported in the Proposition are valid, in order to explicitly compute them we need to computed $P(J_i) = \int_{J_i}\post dw$ (which is equivalent to Condition (1) in section \ref{sec:todo}) and the minimum and maximum  $\sigma(f^{w}(x))$  for $x \in T$ and  $w \in J_i$ (which is equivalent to Condition (2) highlighted above).
%The computation of these two quantities will be the subjects of the next two sections, where we go into the technical details of how to compute them for BNNs. An overall overview of our full methodology is given in Section \ref{todo}, where we given algorithms to wrap up the key results of the methodological sections.

\subsection{Computation of the Bounds}\label{sec:computation_sketch}
%Through the use of Proposition \ref{Prop:SafetyComputation2} and Proposition \ref{prop:pred_posterior_bound}, we can thus bound probabilistic and decision robustness for a BNN operating in adversarial settings.
We now propose a unified approach to computing these lower and upper bounds.  %Notice, however, that, though valid, the bounds reported still include quantities that require explicit computation. It turns out that the explicit computations needed to instantiate the bounds reduce to the same fundamental procedures in the two cases, which allows us to develop a general framework for the computation of robustness in BNNs.
%
%To see that, 
We first observe that the bounds in Equations \eqref{Eq:IntegralSafetyIneq}, \eqref{eq:dsafel} and \eqref{eq:dsafeu} require the integration of the posterior distribution over weight intervals, i.e., $P(H_i)$, $P(K_i)$ and $P(J_i)$. While this is in general intractable, we have built the bound so that $H_i$, $K_i$ and $J_i$ are axis-aligned hyper-rectangles, and so the computation can be done exactly for approximate Bayesian inference methods used in practice. This will be the topic of Section \ref{sec:weight_intervals}, where, given a rectangle in weight space of the form $R = [w^L, w^U]$, we will show how to compute $P(R) = \int_R \post dw$. %This assumption on the form of $H_i$, $K_i$ and $J_i$ may lead to approximation in the bounds, however, interval bounding techniques have been empirically shown to introduce reasonable levels of approximation \cite{gowal2018effectiveness}. 
%\ap{I remove the paragraph here about approx error, if you want to have something please re-write. Do not re-use what was written previously because it is wrong. Notice that the error of interval bounding techniques has nothing to do with the approximation introduced by $H_i$, $K_i$ and $J_i$. They are only rectangles used to approximate a non-rectangular set. IBP, NN, Gowal et al., etc, have nothing to do with it.}

For the explicit computation of decision robustness,  the only missing ingredient is then the computation of the minimum and maximum  $\sigma(f^{w}(x))$  for $x \in T$ and  $w \in J_i$. 
We do this by bounding the BNN output for any given rectangle in the  weight space $R$.
That is, we will compute upper and lower bounds $y^{L}$ and $y^{U}$ such that:
\begin{align}\label{eq:bnn_bounds}
    y^{L} \leq \min_{ \substack{ x \in T \\ w \in R }   } f^w(x) \quad y^{U} \geq \max_{ \substack{ x \in T \\ w \in R }   } f^w(x),
\end{align}
which can then be used to bound $\sigma(f^{w}(x))$ by simple propagation over the softmax (if needed).
The derivation of such bounds will be the subject of Section \ref{subsec:convexrelax}. 

Finally, observe that, whereas for decision robustness we can simply select any weight interval $J_i$, for probabilistic robustness one needs to make a distinction between safe sets ($H_i$) and unsafe sets ($K_i$). It turns out that this can be done by bounding the output of the BNN in each of these intervals.
For example, in the case of the safe sets,
%\MK{no need to repeat the def, just cite} 
by definition we have that $\forall w \in H_i, \forall x' \in T$ it follows that $f^{w}(x') \in S$. By
defining $y^{L}$ and $y^{U}$ as in Equation \eqref{eq:bnn_bounds}, we can see that it suffices to check whether $[y^{L}, y^{U}] \subseteq S$. Hence, the computation of probabilistic robustness also depends on the computation of such bounds (again, discussed in Section \ref{subsec:convexrelax}).

Therefore, once we have  shown how the computation of $P(R)$ for any weight interval and $y^{L}$ and $y^{U}$ can be done, the bounds in  Proposition \ref{Prop:SafetyComputation2} and Proposition \ref{prop:pred_posterior_bound} can be computed explicitly, and we can thus bound probabilistic and decision robustness. 
Section \ref{sec:algoirthm} will assemble these results together into an overall computational flow of our methodology.

\section{Explicit Bound Computation}\label{sec:methodology}

%\LL{This part is not very clear because you introduce the subroutines without describing the main algorithm. So, I would change the presentation. In particular, I would follow the approach we have in the conference paper where first we discuss how to check if a weight set is safe and then we introduce algorithms for both cases.} 
%\MRW{I agree that the motivation of the subroutines wasn't enough, so I tried to addressed that. The only thing that precedes the checking of weight sets here is how we sample. We could move that behind the propagation techniques, but I prefer it before. Let me know what you think of how I have updated things.}

In this section, we provide details on the specific computations needed to calculate the theoretical bound presented in Section \ref{sec:theorethical_bounds} for probabilistic and decision robustness.
We start by discussing how a weight intervals family can be generated in practice,  and how to integrate over them in Section \ref{sec:weight_intervals}.
 In Section \ref{subsec:convexrelax}, we then derive a scheme based on convex-relaxation techniques for bounding the output of BNNs.

\subsection{Integral Computation over Weight Intervals}\label{sec:weight_intervals}
Key to the computation of the bounds derived in Section \ref{sec:theorethical_bounds} is the ability to compute the integral of the posterior distribution over a combined set of weight intervals. Crucially, the shape of the weight sets $\mathcal{H} = \{H_i\}_{i=1}^{n_H}$, $\mathcal{K} = \{K_i\}_{i=1}^{n_K}$ and $\mathcal{J} = \{J_i\}_{i=1}^{n_J}$ is a parameter of the method, so that it can be chosen to simplify the integral computation depending on the particular form of the approximate posterior distribution used.
We build each weight interval as an axis-aligned hyper-rectangle of the form $R = [w^L,w^U]$ for $w^L$ and $ w^U \in \mathbb{R}^{n_w}$. 

\paragraph{Weight Intervals for Decision Robustness}
In the case of decision robustness it suffices to sample any weight interval $J_i$ to compute the bounds we derived in Proposition \ref{prop:pred_posterior_bound}. Clearly, the bound is tighter if the $\mathcal{J}$ family  is finer around the area of high probability mass for $\post$.
In order to obtain such a family we proceed as follows. 
First, we define a \textit{weight margin} $\gamma > 0$ that has the role of parameterising the radius of the weight intervals we define. We then iteratively sample weight vectors $w_i$ from $\post$, for $i=1,\ldots,n_J$, and finally define $J_i = [w^L_i,w^U_i] = [w_i -\gamma ,w_i + \gamma]$.
As such, thus defined weight intervals naturally hover around the area of greater density for $\post$, while asymptotically covering the whole support of the distribution.

\paragraph{Weight Intervals for Probabilistic Robustness}
On the other hand, for the computation of probabilistic robustness one has to make a distinction between safe weight intervals $H_i$ and unsafe ones $K_i$. As explained in Section \ref{sec:computation_sketch}, this can be done by bounding the output of the BNN in each of these intervals.
For example, in the case of the safe sets, by definition, $H_i$ is safe if and only if $\forall w \in H_i, \forall x' \in T$ we have that $f^{w}(x') \in S$. 
Thus, in order to build a family of safe (respectively unsafe) weight intervals $H_i$ ($K_i$), we proceed as follows. As for decision robustness, we iteratively sample weights $w_i$ from the posterior used to build hyper-rectangles of the form $R_i = [w_i - \gamma , w_i + \gamma]$. We then propagate the BNN through $R$ and check whether the output of the BNN in $R$ is (is not) a subset of $S$. 
The derivation of such bounds on propagation will be the subject of Section \ref{subsec:convexrelax}.

Once the family of weights is computed, there remains the computation of the cumulative distribution over such sets. The explicit computations depend on the particular form of Bayesian approximate inference that is employed.
We discuss explicitly the case of Gaussian variational approaches, and of sample-based posterior approximation (e.g., HMC), which entails the majority of the approximation methods used in practice \cite{nalisnick2018priors}.
%In the following paragraph, we discuss explicitly the case of Variational approximation. A treatment of sample-based posterior (e.g., MCMC) is given in Appendix \ref{sec:todo}.
\paragraph{Variational Inference}
For variational approximations, $\post$ takes the form of a multi-variate Gaussian distribution over the weight space. 
%So to define each weight interval $R$, we proceed by sampling $w^*$ from the posterior distribution and selecting a positive $\epsilon$ so that $w^L < w^U$ and the integral of the posterior over $R$ is non-null.
The resulting computations reduce to the integral of a multi-variate Gaussian distribution over a finite-sized axis-aligned rectangle, which can be computed using standard methods from statistics \cite{chang2011chernoff}.
In particular, under the common assumption of variational inference with a Gaussian distribution with diagonal covariance matrix \cite{khan2018fast}, i.e., $\post = \mathcal{N}(\mu,\Sigma)$, with $\Sigma = \textrm{diag}(\Sigma_1,\ldots,\Sigma_{n_w})$, we obtain the following result for the posterior integration:
\begin{align} \label{eq:vi_integral}
   P(R) =  \int_R &\post dw = \\ &\prod_{j=1}^{n_w} \frac{1}{2} \left( \text{erf}\left(   \frac{ \mu_j - w^L_i  }{ \sqrt{2 \Sigma_j } }  \right)-\text{erf}\left(  \frac{ \mu_j - w^u_i  }{ \sqrt{2 \Sigma_j } }\right) \right). \nonumber
\end{align}
By plugging this into the bounds of Equation \eqref{Eq:IntegralSafetyIneq} for $P(H_i)$ and $P(K_i)$ for probabilistic robustness and in Equations \eqref{eq:dsafel} and \eqref{eq:dsafeu} for decision robustness, one obtains a closed-form formula for the bounds given weight set interval families $\mathcal{H}$, $\mathcal{K}$ and $\mathcal{J}$.
\paragraph{Sample-based approximations}
In the case of sample-based posterior approximation (e.g., HMC), we have that $\post$ defines a distribution over a finite set of weights. 
In this case we can simplify the computations by selecting the weight margin $\gamma = 0$, so that each sampled interval will be of the form $R = [w_i,w_i]$ and its probability under the discrete posterior will trivially be:
\begin{align} \label{eq:hmc_integral}
P(R_i) = p(w_i | \mathcal{D}).
\end{align}

\subsection{Bounding Bayesian Neural Networks' Output}\label{subsec:convexrelax}
%\ap{$H$ or $\hat{H}$?}\MW{$H$, subsets of hat are by def subsets of H. H hat is an arbitrary subset of H that isnt defined. H is well defined.}

%\ap{If you want, I feel that this whole section could be put in the appendix and replaced by a brief explanation that we have methods to compute bounds. But not necessary.}
Given an input specification, $T$, and a weight interval, $R = [w^{L}, w^{U}]$, the second key step in computing probabilistic and decision robustness is the bounding of the output of the BNN over $R$ given $T$.
That is, we need to derive methods to compute $[y^{L}, y^{U}]$ such that, by construction, $\forall w \in [w^{L}, w^{U}], \forall x' \in T$ it follows that $f^{w}(x') \in [y^{L}, y^{U}]$. % An intuitve visualization of how this can be achieved for a single parameter BNN can be seen in Figure~\ref{fig:Diagram} where we show how an input interval and a weight interval can be propagated to arrive at an output interval.
%Given Assumption \ref{assumption} on the form of the safe set $S$, this is then equivalent to checking whether:
%\begin{equation}\label{eq:det_spec}
%    \min_{w \in \hat{H}, x \in T}  C_S f^{w}(x) + d_S \geq 0.
%\end{equation} 

In this section, we discuss interval bound propagation (IBP) and linear bound propagation (LBP) as methods for computing the desired output set over-approximations. Before discussing IBP and LBP in detail, we first introduce common notation for the rest of the section. 
We consider feed-forward neural networks of the form:
\begin{align}
    z^{(0)} &= x \label{eq:nn1} \\
    \zeta^{(k+1)}_i &=  \sum_{j=1}^{n_k} W^{(k)}_{ij} z^{(k)}_j + b^{(k)}_i  \quad i = 0,\ldots,n_{k+1}    \label{eq:nn2}\\
    z^{(k)}_i &= \sigma(\zeta^{(k)}_i )  \qquad  \qquad \qquad i = 0,\ldots,n_{k} \label{eq:nn3}
\end{align}
for $k=1,\ldots, K$, where $K$ is the number of hidden layers, $\sigma(\cdot)$ is a pointwise activation function, $W^{(k)} \in \mathbb{R}^{n_k \times n_{k-1}} $ and  $b^{(k)} \in \mathbb{R}^{n_k}$ are the matrix of weights and vector of biases that correspond to the $k$th layer of the network and $n_k$ is the number of neurons in the $k$th hidden layer.
Note that, while Equations \eqref{eq:nn1}--\eqref{eq:nn3} are written explicitly for fully-connected layers, convolutional layers can be accounted for by embedding them in fully-connected form \cite{zhang2018efficient}.

We write $W^{(k)}_{i:}$ for the vector comprising the elements from the $i$th row of $W^{(k)}$, and similarly $W^{(k)}_{:j}$ for that comprising the elements from the $j$th column.
 $\zeta^{(K+1)}$ represents the final output of the network (or the logit in the case of classification networks), that is, $\zeta^{(K+1)} = f^{w}(x)$.
We write $W^{(k),L}$ and $W^{(k),U}$ for the lower and upper bound induced by $R$ for $W^{(k)}$ and $b^{(k),L}$ and $b^{(k),U}$ for those of  $b^{(k)}$, for $k=0,\ldots,K$.
Observe that $z^{(0)}$, $\zeta^{(k+1)}_i$ and $z^{(k)}_i$ are all functions of the input point $x$ and of the combined vector of weights $w = [ W^{(0)}, b^{(0)} ,\ldots  , W^{(K)}, b^{(K)}  ]$.
We omit the explicit dependency for simplicity of notation.
Finally, we remark that, as both the weights and the input vary in a given set, % of value, 
Equation \eqref{eq:nn2} defines a quadratic form.

\iffalse
\begin{figure*}
    \centering
    \subfigure{\includegraphics[width=1.0\textwidth]{Figures/diagram.pdf}}
    \caption{A diagram explaining the interval bound propagation process for a single variable BNN. Given an interval in the input space (left image) and an interval in the weight space (center image), we use interval bound propagation to get an over approximation of the output of the BNN (right image). We can then reason about if this output region contains only safe outputs then we can rule out adversarial examples.}
    \label{fig:Diagram}
\end{figure*}
\fi

\paragraph{Interval Bound Propagation (IBP)}\label{sec:IBP}
%\MRW{Below is a more polished version of Andrea's text from the previous write-up.}
IBP has already been employed for fast certification of deterministic neural networks \cite{gowal2018effectiveness}.
For a deterministic network, the idea is to propagate the input box around $x$, i.e., $T = [x^L,x^U]$, through the first layer, so as to find values $z^{(1),L}$ and $z^{(1),U}$ such that  $z^{(1)} \in [ z^{(1),L}, z^{(1),U}]$, and then iteratively propagate the bound through each consecutive layer for $k=1,\ldots,K$. The final box constraint in the output layer can then be used to check for the specification of interest \cite{gowal2018effectiveness}.
The only adjustment needed in our setting %(that is,  for the bounding of Equation \eqref{eq:det_spec}) 
is that at each layer we also need to propagate the interval of the weight matrix $[W^{(k),L},W^{(k),U}]$ and that of the bias vector $[b^{(k),L},b^{(k),U}]$. 
This can be done by noticing that the minimum and maximum of each term of the bi-linear form of Equation \eqref{eq:nn2}, that is, of each monomial $W^{(k)}_{ij} z^{(k)}_j$, lies in one of the four corners of the interval $[W^{(k),L}_{ij},W^{(k),U}_{ij}] \times [z^{(k),L}_j,z^{(k),U}_j]$, and by adding the minimum and maximum values respectively attained by $b^{(k)}_i$.
As in the deterministic case, interval propagation through the activation function proceeds by observing that generally employed activation functions are monotonic, which permits %simply
the application of Equation \eqref{eq:nn3} to the bounding interval. Where monotonicity does not hold, we can bound any activation function that has finitely many inflection points by splitting the function into piecewise monotonic functions.
This is summarised in the following proposition.
\begin{proposition}
\label{Prop:IBP}
Let $f^{w}(x)$ be the network defined by the set of Equations \eqref{eq:nn1}--\eqref{eq:nn3}, let for $k = 0,\ldots,K$:
\begin{align*}
    t_{ij}^{(k),L} = \min &\{ W_{ij}^{(k),L} z_j^{(k),L} , W_{ij}^{(k),U} z_j^{(k),L} , \\ &W_{ij}^{(k),L} z_j^{(k),U} , W_{ij}^{(k),U} z_j^{(k),U}       \}  \\
    t_{ij}^{(k),U} = \max &\{ W_{ij}^{(k),L} z_j^{(k),L} , W_{ij}^{(k),U} z_j^{(k),L} , \\ &W_{ij}^{(k),L} z_j^{(k),U} , W_{ij}^{(k),U} z_j^{(k),U}       \} 
\end{align*}
where $i=1,\ldots,n_{k+1}$, $j=1,\ldots,n_k$, and $z^{(k),L} = \sigma(\zeta^{(k),L})$, $z^{(k),U} = \sigma(\zeta^{(k),U})$ and:
\begin{align*}
    \zeta^{(k+1),L} &=  \sum_j t_{:j}^{(k),L} + b^{(k),L} \\
    \zeta^{(k+1),U} &= \sum_j t_{:j}^{(k),U} + b^{(k),U}.
\end{align*}
Then we have that $\forall x \in T$ and $\forall w \in R$:
\begin{equation*}
    f^w(x) =  \zeta^{(K + 1)} \in \left[  \zeta^{(K+1),L} , \zeta^{(K+1),U} \right].
\end{equation*}
\end{proposition}
The proposition above, whose proof is  Appendix~\ref{appendix:ibpproof} (Appendix C subsection C), yields a bounding box for the output of the neural network in $T$ and $R$. %\MK{check appendices, this does not exist} \MRW{Not sure if someone else corrected this but I checked and this appears to be correct now.}

%This can be used directly to find a lower bound for Equation \eqref{eq:det_spec}, and hence checking whether $\hat{H}$ is a safe set of weights.
%As noticed in \cite{gowal2018effectiveness} and discussed in the Supplementary Material \ap{Appendix, where}, for BNNs a slightly improved IBP bound can be obtained by eliding the last layer with the linear formula of the safety specification of set $S$.

\paragraph{Linear Bound Propagation (LBP)}\label{sec:LBP}
We now discuss how LBP can be used to lower-bound the BNN output over $T$ and $R$ as an alternative to IBP.
In LBP, instead of propagating bounding boxes, one finds lower and upper Linear Bounding Functions (LBFs) for each layer and then propagates them through the network.
As the bounding function has an extra degree of freedom w.r.t.\ the bounding boxes obtained through IBP, LBP usually yields tighter bounds, though at an increased computational cost.
Since in deterministic networks non-linearity comes only from the activation functions, LBFs in the deterministic case are computed by bounding the activation functions, and propagating the bounds through the affine function that defines each layer. 

Similarly, in our setting, given $T$ in the input space and $R$ for the first layer in the weight space, we start with the observation that LBFs can be obtained and propagated through commonly employed activation functions for Equation \eqref{eq:nn3}, as discussed in \cite{zhang2018efficient}.

%, lower and upper LBFs can be computed by first bounding the activation functions of each neuron and by then applying the McCormick's inequalities \cite{missing} to the bi-linear forms $w_{ij}^{(k)} z_{j}^{(k)}$ to propagate the bound to the pre-activation of the next layer, for each $k=0,\ldots,K$.

%The bound on the activation functions follows analogously as for the deterministic case. 
%In particular for networks that use ReLU, tanh, or sigmoid activation functions, that following Lemma holds.
\begin{lemma}
\label{lemmma:act_bound}
Let $f^{w}(x)$ be defined by  Equations \eqref{eq:nn1}--\eqref{eq:nn3}. For each hidden layer $k=1,\ldots,K$, consider a bounding box in the pre-activation function, i.e.\ such that $\zeta^{(k)}_i \in [ \zeta^{(k),L}_i ,\zeta^{(k),U}_i ]$ for $i=1,\ldots,n_k$. Then there exist coefficients $\alpha^{(k),L}_i$, $\beta^{(k),L}_i$, $\alpha^{(k),U}_i$ and  $\beta^{(k),U}_i$  of lower and upper LBFs on the activation function such that for all $\zeta^{(k)}_i \in  [ \zeta^{(k),L}_i ,\zeta^{(k),U}_i ]$ it holds that:
\begin{equation*}
    \alpha^{(k),L}_i \zeta^{(k)}_i + \beta^{(k),L}_i \leq \sigma(\zeta^{(k)}_i) \leq \alpha^{(k),U}_i \zeta^{(k)}_i + \beta^{(k),U}_i. 
\end{equation*}
\end{lemma}
%\ap{Add brief explanation sketch and refer to the Supplementary.}
The lower and upper LBFs can thus be minimised and maximised to propagate the bounds of $\zeta^{(k)}$ in order to compute a bounding interval $[z^{(k),L},z^{(k),U}]$ for $z^{(k)} =  \sigma(\zeta^{(k)})$.
Then, LBFs for the monomials of the bi-linear form  of Equation \eqref{eq:nn2} can be derived using McCormick's inequalities \cite{mccormick1976computability}: 
\begin{align}
     W_{ij}^{(k)} z_j^{(k)} \geq W_{ij}^{(k),L} z^{(k)}_j  +  W_{ij}^{(k)} z^{(k),L}_j - W_{ij}^{(k),L} z^{(k),L}_j \label{eq:mc1} \\
     W_{ij}^{(k)} z_j^{(k)} \leq W_{ij}^{(k),U} z^{(k)}_j  +  W_{ij}^{(k)} z^{(k),L}_j - W_{ij}^{(k),U} z^{(k),L}_j \label{eq:mc2}
\end{align}
for every $i=1,\ldots,n_k$, $j=1,\ldots,n_{k-1}$ and $k=1,\ldots,K$.
The bounds of Equations \eqref{eq:mc1}--\eqref{eq:mc2} can thus be used in Equation \eqref{eq:nn2} to obtain LBFs on  the pre-activation function of the following layer, i.e.\ $\zeta^{(k+1)}$.
The final linear bound can be obtained by iterating the application of Lemma \ref{lemmma:act_bound} and Equations \eqref{eq:mc1}--\eqref{eq:mc2}  through every layer.
This is summarised in the following proposition, which is proved in Appendix~\ref{appendix:proofs} along with an explicit construction of the LBFs.
\begin{proposition}
\label{proposition:lbp}
Let $f^{w}(x)$ be the network defined by the set of Equations \eqref{eq:nn1}--\eqref{eq:nn3}. Then for every $k=0,\ldots,K$ there exists lower and upper LBFs on the pre-activation function of the form:
\begin{align*}
     &\zeta^{(k+1)}_i \geq \mu_i^{(k+1),L} \cdot x + \sum_{l = 0}^{k-1} \langle \nu_i^{(l,k+1),L} , W^{(l)}  \rangle + \\ &\nu_i^{(k,k+1),L} \cdot W^{(k)}_{i:} + \lambda_i^{(k+1),L} \quad \textrm{for} \; i=1,\ldots,n_{k+1}\\
     &\zeta^{(k+1)}_i \leq \mu_i^{(k+1),U} \cdot x + \sum_{l = 0}^{k-1} \langle \nu_i^{(l,k+1),U} , W^{(l)}  \rangle + \\ &\nu_i^{(k-1,k+1),U} \cdot W^{(k)}_{i:} + \lambda_i^{(k+1),U}  \quad \textrm{for} \; i=1,\ldots,n_{k+1}
\end{align*}
where  $\langle \cdot , \cdot  \rangle $ is the Frobenius product between matrices, $\cdot$ is the dot product between vectors, and the explicit formulas for the LBF coefficients, i.e.,  $\mu_i^{(k+1),L}$, $\nu_i^{(l,k+1),L}$, $\lambda_i^{(k+1),L}$, $\mu_i^{(k+1),U}$, $\nu_i^{(l,k+1),U}$, are given in Appendix~\ref{appendix:lbpproof}.

Now let $\zeta^{(k),L}_i$ and $\zeta^{(k),U}_i$,  respectively, be the minimum and the maximum of the right-hand side of the two equations above; then we have that $\forall x \in T$ and $\forall w \in R$:
\begin{equation*}
    f^w(x) =  \zeta^{(K + 1)} \in \left[  \zeta^{(K+1),L} , \zeta^{(K+1),U} \right].
\end{equation*}
\end{proposition}

\begin{algorithm} \small
\caption{Lower Bounds for BNN Probabilistic Robustness}\label{alg:psafealgorithm}
\textbf{Input:} $T$ -- Input Region, $f^{\mathbf{w}}$ -- Bayesian Neural Network, $p(w | \mathcal{D})$ -- Posterior Distribution with variance $\Sigma$,  $N$ -- Number of Samples, $\gamma$ -- Weight margin.  \\
\textbf{Output:} A sound lower bound on $\psafe$.%\vspace{-1em}
\hrulealg 
\begin{algorithmic}[1]
\STATE \algcompro{\textit\# {$\mathcal{H}$ is a set of known safe weight intervals}}
\STATE $\mathcal{H} \gets \emptyset$   
\STATE \algcom{\textit\# {Elementwise product to obtain width of weight margin}}
\STATE $v \gets \gamma \cdot I \cdot \Sigma$
\FOR {$i\gets0$ to $N$} 
    \STATE $w^{(i)} \sim p(w | \mathcal{D})$
    \STATE \algcom{\textit\# {Assume weight intervals are built to be disjoint}}
    \STATE $[w^{(i), L}, w^{(i),U}] \gets [w_i - v , w_i + v]$
    \STATE \algcom{\textit{\# Interval/Linear Bound Propagation, Section~\ref{subsec:convexrelax}}}
    \STATE $\*y^{L}, \*y^{U} \gets \texttt{Propagate}(f, T, [w^{(i), L}, w^{(i), U}]$) %\ap{Matthew, is it ok if instead of method we say Propagate? Method sounds a bit too general, but maybe there is a reason for it?}\MRW{Agreed. Changed.}
    \IF{$[\*y^{L}, \*y^{U}] \subset S$}
        \STATE $\mathcal{H} \gets \mathcal{H} \bigcup \{[w^{(i), L}, w^{(i), U}]\}$
    \ENDIF
\ENDFOR
\STATE  $P_{\text{safe}}^L \gets 0.0$
\FOR {$[w^{(i), L}, w^{(i),U}] \in \mathcal{H}$} 
    \STATE \algcompro{\textit\# {Compute safe weight probs, Section \ref{sec:weight_intervals} }}
    \STATE $P_{\text{safe}}^L = P_{\text{safe}}^L+ P([w^{(i), L}, w^{(i),U}])$
\ENDFOR
\STATE return $P_{\text{safe}}^L$% \ap{Why introducing new notation for the lower bound? It should be $P_{\text{safe}}^L$ no?}
\end{algorithmic}
\end{algorithm}

\section{Complete Bounding Algorithm}\label{sec:algoirthm}
%\ap{I need to Recheck after algorithm has been updated}
%\ap{check comment on algorithm}
Using the results presented in Section \ref{sec:methodology}, it is possible to explicitly compute the bounds on probabilistic and decision robustness derived in Section \ref{sec:theorethical_bounds}. 
In this section, we bring together all the results discussed so far, and assemble complete algorithms for the computation of bounds on $\psafe$ and $\dsafe$. We discuss the procedure to lower bound $\psafe$ in Algorithm~\ref{alg:psafealgorithm}. We then discuss the details of upper bounds and bounds on $\dsafe$, leaving the algorithms and their description for these bounds to Appendix 
\ref{appendix:algorithms}.% We hence describe how this can be modified in the case of $\dsafe$ and for the upper bounds (explicit pseudocode listings for each case are given in Appendix \ref{appendix:algorithms}).

%Here we discuss the details of Algorithm~\ref{alg:generalalgorithm} where we present a procedure for lower bounding $\psafe$ and $\dsafe$. We conclude the section with a discussion of the upper bounding algorithm, and we state and further discuss the upper bounding algorithm in detail in Appendix \ref{appendix:algorithms} \ap{Missing} \MRW{Fixed}. %Further, discussion of the algorithm in this section assumes that weight intervals are constructed such that they are disjoint. 

\subsection{Lower Bounding Algorithm}

We provide a step-by-step outline for how to compute lower bounds on $\psafe$ in Algorithm \ref{alg:psafealgorithm}.
%In Algorithm~\ref{alg:generalalgorithm} we provide a step-by-step outline for how to compute lower bounds on $\psafe$.
We start on line 2 by initializing the family of safe weight sets $\mathcal{H}$ to be the empty set and by scaling the weight margin with the posterior weight scale (line 4).
We then iteratively (line 5) proceed by sampling weights from the posterior distribution (line 6), building candidate weight boxes (line 8), and propagate the input and weight box through the BNN (line 10).
We next check whether the propagated output set is inside the safe output region $S$, and if so update the family of weights $\mathcal{H}$ to include the weight box currently under consideration (lines 11 and 12).
Finally, we rely on the results in Section \ref{sec:weight_intervals} to compute the overall probabilities over all the weight sets in $\mathcal{H}$, yielding a valid lower bound for $\psafe$.
For clarity of presentation, we assume that all the weight boxes that we sample in lines 6--8 are pairwise disjoint, as this simplifies the probability computation. 
The general case with overlapping weight boxes relies on the Bonferroni bound and is given in Appendix \ref{appendix:bonferroni}.
The algorithm for the computation of the lower bound on $\dsafe$ (listed in the Appendix~\ref{appendix:algorithms} as Algorithm \ref{alg:dsafelower}) proceeds in an analogous way, but without the need to perform the check in line 11, and by adjusting line 18 to the formula from Proposition \ref{prop:pred_posterior_bound}.

\subsection{Upper Bounding Algorithm}

Upper bounding $\psafe$ and $\dsafe$ follows the same computational flow as Algorithm~\ref{alg:psafealgorithm}. The pseudocode outlines computation of probabilistic and decision robustness are listed respectively in Algorithm \ref{alg:upperbound} and \ref{alg:upperdecisionsafety} in Appendix~\ref{appendix:algorithms} subsection A and Appendix~\ref{appendix:algorithms} subsection B.
%\MK{which} \MRW{Added.}
We again proceed by sampling a rectangle around weights, propagate bounds through the NN, and compute the probabilities of weight intervals. The key change to the algorithm to allow upper bound computation involves computing the best case, rather than the worst case, for $y$ in for decision robustness (line 12 in Algorithm~\ref{alg:upperbound}) and ensuring that the entire interval $[y^{L}, y^{U}] \notin S$ (line 18) for probabilistic robustness. In Appendix \ref{appendix:algorithms} subsection B we also discuss how adversarial attacks can be leveraged to improve the upper bounds.

%\subsection{Sources of Approximation in Computations}
%\MRW{A new section systematically detailing the sources of approximation. }

\subsection{Computational Complexity}\label{sec:time_complexity}

%\MK{Revise to abbreviate the discussion of simple aspects}
Calculations for probabilistic robustness and decision robustness follow the same computational flow and include: bounding  of the neural network output, sampling from the posterior distribution, and computation of integrals over boxes on the input and weight space.

Regarding Algorithm \ref{alg:psafealgorithm} (or equivalently  Algorithm \ref{alg:dsafelower} for decision robustness), it is clear that the computational complexity scales linearly with the number of samples, $N$, taken from the posterior distribution.
Observe that, in order to obtain a tight bound on the integral computation, $N$ needs to be large enough such that $N$ samples of the posterior with width $\gamma$ span an area of high probability mass for $\post$. 
%\MK{This is very informal, need to make more rigorous}
Unfortunately, this means that, for a given approximation error magnitude, $N$ needs to scale quadratically on the number of hidden neurons. %on consecutive layers,  whose dimension is generally much larger than both the input and output dimensions and the number of neurons. 
%In fact, for a given number of neurons, we would expect the bound to perform better on wide rather than deep architectures (as confirmed by the experimental results discussed in Section \ref{sec:mnist_results}). 
%Note also that the tightness of the bound is related to the variance of the posterior, as the same number of weight intervals, $N$, can provide a better covering for $\post$ in case of smaller variances, and hence a smaller effective support.
%
Given the sampling of the hyper-rectangles, computation of the integral over the weight boxes is done through Equations \eqref{eq:vi_integral} and \eqref{eq:hmc_integral}. The integration over the weight boxes is done in constant time for HMC (though a good quality HMC posterior approximation scales with the number of parameters) and $\mathcal{O}(n_w)$ for VI.
The final step needed for the methodology is that of bound propagation, which clearly differs when using IBP or LBP. 
In particular, the cost of performing IBP is $\mathcal{O}(K\hat{n}\hat{m})$, where $K$ is the number of hidden layers and $\hat{n} \times \hat{m}$ is the size of the largest weight matrix $W^{(k)}$, for $k=0,\ldots,K$. 
LBP is, on the other hand,  $\mathcal{O}(K^2\hat{n}\hat{m})$.
Overall, the time complexity for certifying a VI BNN with IBP is therefore $\mathcal{O}(Nn_wK\hat{n}\hat{m})$, and similar formulas can be obtained for alternative combinations of  inference and propagation techniques that are employed.
We remark that, while sound, the bounds we compute are not guaranteed to converge to the true values of $\psafe$ and $\dsafe$ in the limit of the number of sample $N$ because of the introduction of over-approximation errors due to bound propagation. 

\section{Experiments}\label{sec:experiments}

\begin{figure}
    \centering
    \includegraphics[width=0.5\textwidth]{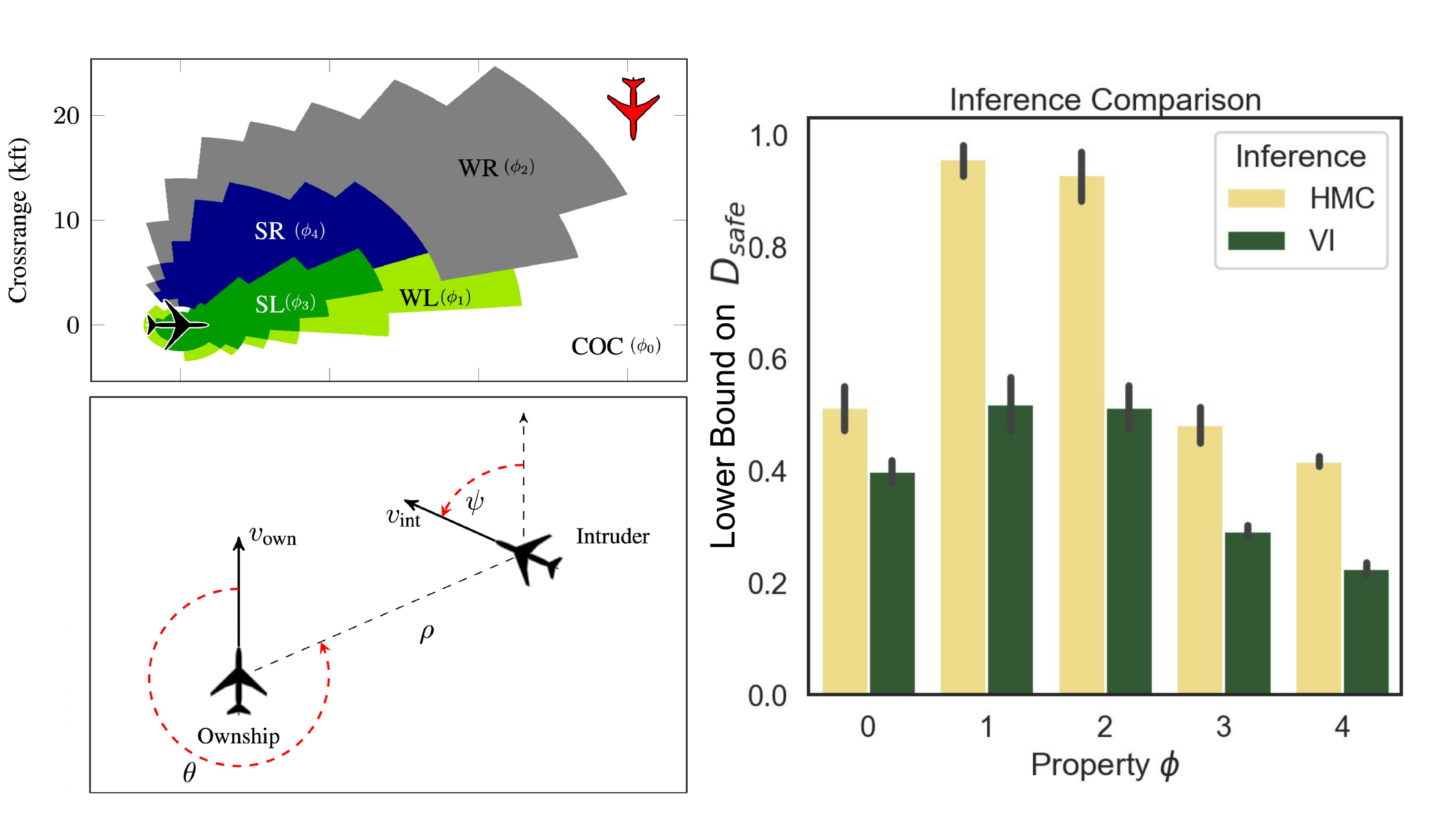}
    \caption{\textbf{Top Left:} Encounter geometry, ground truth and property labels: Clear of Conflict (COC), Strong Left/Right (SL/R), Weak Left/Right (WL/R), for a collision scenario. Diagrams modified from \cite{julian2019guaranteeing}. \textbf{Bottom Left:} Encounter geometry labelled with features used for collision avoidance prediction. \textbf{Right:} Bounds on decision robustness obtained for HMC and VI trained BNNs for each property.}
    \label{fig:HCAS_A}
\end{figure}

%\MK{The experimental section needs to be substantially pruned and reduced in size - at times too much space is devoted to simple aspects (such as detailed description of linear search), but the more difficult ones are not tackled/explained (such as bound tightness, effect of over-approximation, or coverage of the density function} 
%\MK{Concrete values for parameters are given, which implies that this was ran once and not for a range of values, making trends difficult to evaluate (they could be just random), and broader conclusions difficult to justify}
%\MK{Concepts such as 'Verified Uncertain' belong to the methodology and to derive such as answers needs to be explained in the abstract and applied}
%\MK{Also various typos}

In this section, we empirically investigate the suitability of our certification framework for the analysis of probabilistic and decision robustness in BNNs.
We focus our analysis on four different case studies.
First, we provide a comprehensive evaluation of an airborne collision avoidance system \cite{julian2019guaranteeing}. To do so, we partition the entire input space into 1.8 million different specifications and bound $\psafe$ and $\dsafe$ by computing the bounds for each specification.
%R \LL{This is not really true, because we compute bounds for each partition, but then we do not use union bounds or similar. I would tone it down a bit}\MRW{What I mention here is a fact: we partition the entire input domain. You are talking about a different kind of analysis (a global worst-case) that I don't mention anywhere. Also this is the opening of the experimental section so it is inappropriate to go into details here. If you think it is still confusing where we give the details please point that out.}
%, illustrating how our methods can be employed for systematic analysis of safety-critical applications, and further comparing the robustness of HMC- and VI-trained models. %\ap{You mean like that you have partitioned the input space in several $T$s and verified all of them? At the moment is not very clear what you mean} in order to comprehensively understand the safety of the learned avoidance system. We also use this case study to observe the tightness of our bounds as well as to demonstrate that we can verify both HMC and VI posteriors. 
%
We then turn our attention to an industrial regression benchmark \cite{ucidata} and demonstrate how our analysis can provide tight characterization of the worst-case error of predictions in adversarial settings in relation to the magnitude of the maximum attack allowed. %In particular, we study the threshold of adversarial noise that is tolerable in these scenarios which can be used by practitioners to ensure compliance with industrial standards.
Next, we analyse the scalability of our method in the well-known MNIST dataset for handwritten digits recognition \cite{lecun1998mnist}, along with its behaviour on out-of-distribution input samples.
Finally, we  we study a two-class subset of the German Traffic Sign Recognition Benchmark (GTSRB) dataset \cite{stallkamp2012gtsrb}, whose input space is 1500 dimensions larger 
%\ap{Gotta put a reference here} \MRW{Done.} 
than what has previously been studied for BNN certification against adversarial examples, showcasing that we are still able to compute non-trivial guarantees in this setting. For each dataset, we first describe the problem setting and BNN used to solve it, along with its hyperparameters. We then discuss the properties of interest for each dataset. Finally, we provide discussion and illustration of our bounds performance.
All the experiments have been run on 4 NVIDIA 2080Ti GPUs in conjunction with 4 24-core Intel Core Xeon 6230.

\subsection{Airborne Collision Avoidance}

Our first case study is the Horizontal airborne Collision Avoidance System (HCAS) \cite{julian2019guaranteeing}, a dataset composed of millions of labelled examples of intruder scenarios. %\MK{There is a lot of detail, which should probably be gathered in a table (of advisories) to make sure coverage is comprehensive. Same for parameters and their selection (justify why), and also thresholds}

%\begin{itemize}
    %\item On what was fig 4: \ap{We have a notation for the upper and lower bounds, please use that one rather than english words. It is confusing otherwise, especially because we have two different properties we are analysing in this paper. In fact in each Figure in the experiments you write Upper and Lower but each time with a very different meaning, it is a bit sloppy like this.}
    %\item \LL{See comment in the previous subsection, let's be sure we do not over-claim}
    %\item \ap{what is the width of each region?}
    %\item  \ap{A couple of sentences earlier you say 35. Which one is correct?}
%    \item  \ap{PUT NUMBER OF TOTAL BOUNDING RUNS NEEDED HERE.}
    %\item \ap{Put average computational time: this seems a massive experiments, the fact we can do it, plays in our advantage, definately worth writing some computational details down!}
%\end{itemize}

\subsubsection{Problem Setting}

The task of the BNN is to predict a turn advisory for an aircraft given another oncoming aircraft, including clear of conflict (COC), weak left (WL), weak right (WR), strong left (SL), and strong right (SR). These are depicted in the top left of Figure~\ref{fig:HCAS_A}. % clear of conflict (COC) advising no action, weak left (WL) advising a 2\textdegree/s left turn, weak right (WR) advising a 2\textdegree/s right turn, strong left (SL) advising a 4\textdegree/s left turn, and strong right (SR) advising a 4\textdegree/s right turn (depicted in the top left of Figure~\ref{fig:HCAS_A}) – given the geometric layout of the ownship and of the intruder (summarised in the bottom left of Figure~\ref{fig:HCAS_A}).
We follow the learning procedure described in %the original 
\cite{julian2019guaranteeing}, where encounter scenarios are partitioned into 40 distinct datasets. %where each dataset has a fixed time-to-collision and previous advisory. 
We then  learn a BNN to predict the correct advisories for each dataset, resulting in 40 different BNNs which need to be analysed. 

To analyze the system of 40 BNNs, we first discretize the entire state-space into 1.8 million mutually exclusive input specifications. The input specifications are sized according to the spacing of the ground truth labels supplied by \cite{julian2019guaranteeing}. Namely, we consider an $\ell_\infty$ norm ball with different widths for each input dimension. Those widths are $[0.016, 0.025, 0.025, 0.05]$. 
%We perform a complete certification
%for the HCAS system by first partitioning the entire input space in 1.8 mutually exclusive input regions %R(\ap{what is the width of each region?}),
%and bounding the behaviour of each one of the 40 BNNs %R \ap{A couple of sentences earlier you say 35. Which one is correct?}\MRW{40. As in the table. This is corrected.}on each one of the resulting regions.
%
The output specification is taken to be the set of all softmax vectors such that the argmax of the softmax corresponds to the true label. 
We separate these output specifications into 5 different properties, which we termed $\phi_j$ for $j=0,\ldots,4$ corresponding to each of the possible advisories. For all properties in this section we use LBP with 5 samples with a weight margin of 2.5 standard deviations. %R \ap{PUT NUMBER OF TOTAL BOUNDING RUNS NEEDED HERE.}\MRW{"Bounding runs" What does that mean? Forward passes?}\ap{Sorry, meant like how many verifications did you run. IT seems like a massive number, could be good to write it down explicitly}, \ap{Put average computational time: this seems a massive experiments, the fact we can do it, plays in our advantage, definately worth writing some computational details down!}\MRW{In progress. In general, though this kind of analysis is pointless. We have 8 GPUs and 96 CPU cores and that is the only thing that determines the computational time for this experiment. If you try to run it on your phone it will take years, laptop will take months, our server will take days, DeepMind servers will take hours. In the absence of standardization, computational time only measures how deep your pockets are. I have added it, but in general, as a reviewer, I always skip computational times.}

We train BNNs with Variational Online Gauss Newton (VOGN), where the posterior approximation is a diagonal covariance Gaussian, and wih Hamiltonian Monte Carlo (HMC). The BNN architecture has a single hidden layer with 125 hidden units, the same size as the original system proposed in \cite{julian2019guaranteeing}. %The output is a softmax layer with 5 dimensions each corresponding to one of the output advisories for the HCAS system.
We use a diagonal covariance Gaussian prior with variance 0.5 for VOGN and a prior variance of 2.5 for HMC.% We also use the sparse categorical crossentropy likelihood. 

\begin{table*}[h]\caption{\label{tab:HCASTable} Certification of airborne collision avoidance over a complete partition of the input space. Each state is either certified safe, unsafe, or not certifiable with the chosen thresholds. Number of states and proportions are reported along with the number of BNNs involved in the system for each property. %\MK{State what columns mean}.
}
\centering
\small
\begin{tabular}{|l|c|c|c|c|c|}
\hline
                             & \textbf{Total Inputs} & \textbf{\# Certified Safe ($\psafe > 0.98$)} & \textbf{\# Certified Unsafe ($\psafe < 0.05$)} & \textbf{\# Uncertifiable} & \textbf{\# BNNs} \\ \hline \hline
\textbf{$\phi_0$}            & 795,853               & 620,158 (77.9\%)                             & 168,431 (21.1\%)                               & 7,313 (0.9\%)          & 35                           \\ \hline
\textbf{$\phi_1$}            & 324,443               & 257,453 (79.3\%)                             & 34,379 (10.5\%)                                & 32,639 (10.0\%)         & 21                           \\ \hline
\textbf{$\phi_2$}            & 323,175               & 257,724 (79.7\%)                             & 36,839 (11.3\%)                                & 28,612 (8.8\%)         & 21                           \\ \hline
\textbf{$\phi_3$}            & 178,853               & 101,346 (53.4\%)                             & 64,618 (34.0\%)                                & 23,799 (12.5\%)         & 31                           \\ \hline
\textbf{$\phi_4$}            & 189,991               & 104,546 (55.0\%)                             & 70,310 (37.0\%)                                & 15,135 (7.9\%)         & 31                           \\ \hline 
\textbf{Total:} & 1,812,315             & 1,341,227 (74.0\%)                           & 374,577 (20.6\%)                               & 107,498 (5.9\%)        & 40                           \\ \hline
\end{tabular}
\end{table*}

\begin{figure}
    \centering
    \includegraphics[width=0.4\textwidth]{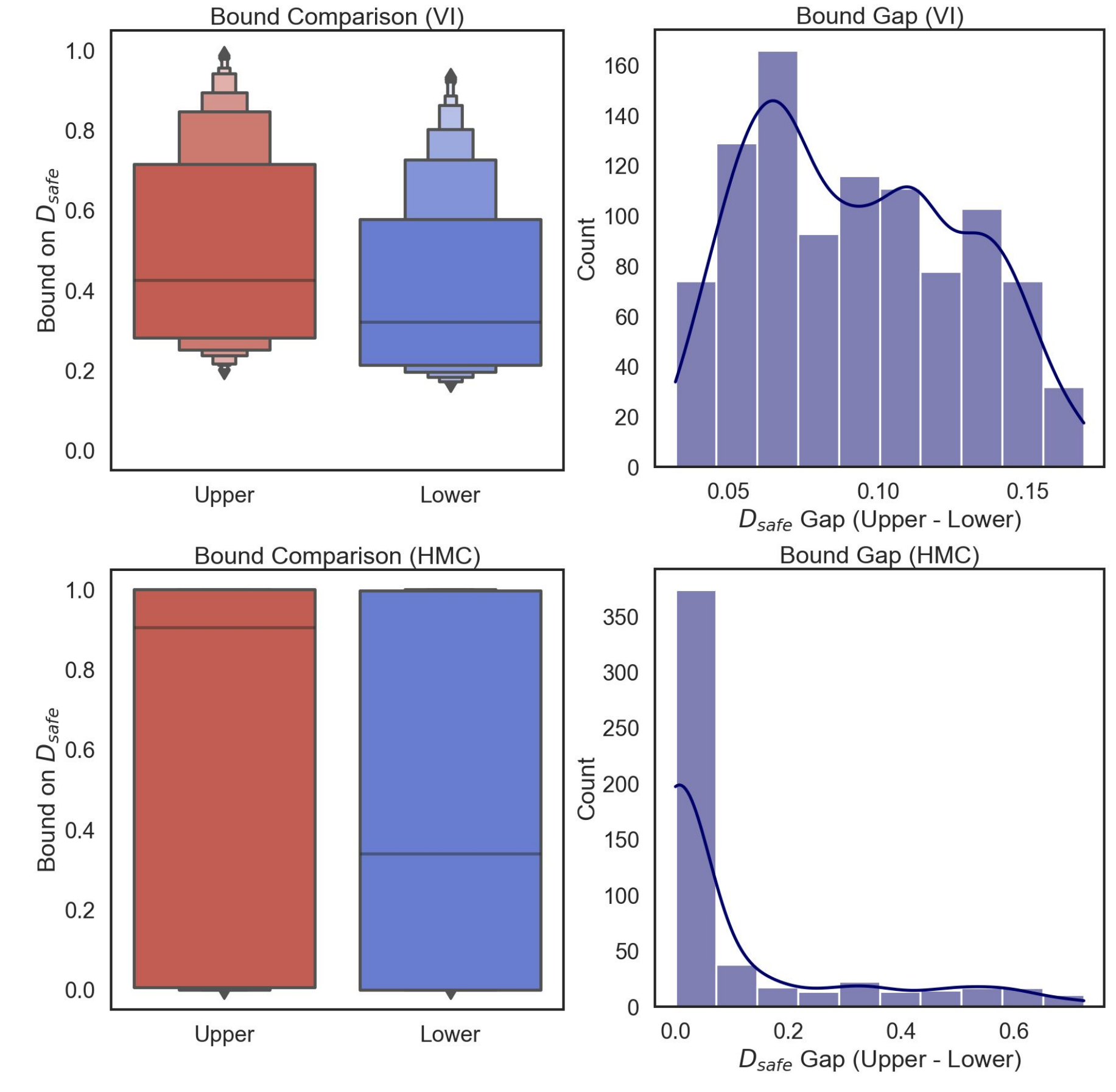}
    \caption{\textbf{Left:} Box plots showing the distribution of upper and lower bounds for VI (top) and HMC (bottom). \textbf{Right:} Histograms showing gap between upper and lower bounds for VI (top) and HMC (bottom). %We highlight that HMC has bounds on $\dsafe$ that are more evenly distributed and that (VI) has more concentrated bounds leading to a tighter bound on average, mean gap of 0.08 for VI mean gap of 0.11 for HMC. 
    }
    \label{fig:HCAS_B}
\end{figure}

\subsubsection{Analysis with $\psafe$ Certification}\label{subsec:prob_hcas} 

%In the case of HCAS, we use $\tau_{\text{safe}} = 0.98$ to determine if an input-output specification is safe. Thus, we declare a specification to be safe if and only if our lower bound guarantees that $\psafe > 0.98$. In practice, this translates to saying that we only consider an input specification to be safe if more than 98\% of the BNN posterior is adversarially robust and correct on a given input. In a less safety-critical scenario, however, one could select $\tau_{\text{safe}} = 0.8$ and be reasonably sure of safe behavior. %In addition, practitioners must set $\tau_{\text{unsafe}}$ which sets the probability threshold after which we determine the BNN is unsafe. For example, setting $\tau_{\text{unsafe}} = 0.05$ means that if less than 5\% of the posterior is safe, then we certify that the BNN is unsafe. This can be computed with our upper bound. A special case is when $\tau_{\text{safe}} = \tau_{\text{unsafe}}$. In this setting, every input is either safe or unsafe and this can be certified with our bounds. \ap{I am not sure I get the relevance of the above discussion. It could mostly be cut out, imho.} %We provide further discussion for when $\tau_{\text{safe}} \neq \tau_{\text{unsafe}}$ in the context of our results below.
%

For each of the 1.8 million disjoint input specification, we compute both upper and lower bounds on $\psafe$. Given that probabilistic robustness is a real-valued probability and not a binary predicate, practitioners must select thresholds that reflect a strong belief that a value is safe or unsafe. We call these thresholds $\tau_{\text{safe}}$ and $\tau_{\text{unsafe}}$. Once one has computed bounds on $\psafe$, the proportions of safe and unsafe states (as reported in Table~\ref{tab:HCASTable}) can be computed by checking thresholds. We check our bounds against strict safety and unsafety thresholds $\tau_{\text{safe}} = 0.98, \tau_{\text{unsafe}} = 0.05$.% and provide more discussion on this in the Appendix ???. \ap{There are some questions mark here}

For the selected threshold values, %we present
Table~\ref{tab:HCASTable} %which represents a 
reports the certified performance of the BNN system. Such a report can be used by regulators and practitioners to determine the if the system is safe for deployment. In this case, we find that across all properties 74\% of the states are certified to be safe while 20\% are certified to be unsafe. The remaining 6\% fall somewhere in between the two safety thresholds. These statistics indicate that roughly 18\% of the decisions issued by the system were correct but not robust, thus the systems accuracy of 92\% does not paint the complete story of its performance. Moreover, we break down each of the properties of the system, represented by each row of Table~\ref{tab:HCASTable}, to understand where the most common failure modes occur. We find that the most unsafe indicators are the strong left, $\phi_3$, and for strong right, $\phi_4$, the system has the lowest certified safety at 53.4\% and 55.0\% respectively. They also have the highest certified unsafety at 34.0\% and 37.0\%, respectively. We conjecture that these these specifications are less safe due to the fact that their is less labeled data representing them in the dataset. Less data has been shown to be correlated with less robustness for BNNs \cite{bekasov2018bayesian}. If the results in Table~\ref{tab:HCASTable} are deemed to be insufficient for deployment, then practitioners can collect more data for unsafe properties e.g.,  $\phi_3$ and  $\phi_4$, or could resort to certified safe training for BNNs as suggested in \cite{wicker2021bayesian}.

\subsubsection{Analysis with $\dsafe$ Certification}

In order to analyze the decision robustness of the BNNs, we again discretize the input space. For these results, we use a coarser discretization, with the input specification being an $\ell_\infty$ ball radius of $0.125$ over each input dimension, and for the sake of computational efficiency %our previous analysis 
we allow some gaps between the input specifications. %\ap{What? Why? I don't understand again, sorry.}
For each specification, we compute upper and lower bounds on $\dsafe$. We plot the result of our bounds on decision robustness in Figures~\ref{fig:HCAS_A} and \ref{fig:HCAS_B}. In the right hand portion of Figure~\ref{fig:HCAS_A} we visualize the average lower bound on decision robustness for two BNNs, one trained with HMC (yellow) and the other trained with VI (green). We find that we are able to certify a higher lower bound, indicating heightened robustness, for the HMC-trained BNN. This corroborates previous robustness studies that highlight that HMC is more adversarially robust \cite{bekasov2018bayesian, carbone2020robustness}. In Figure~\ref{fig:HCAS_B} we analyze the tightness of our bounds in this scenario by comparing the lower and upper bounds provided by our method. For VI, the gap, plotted in purple in the upper right, is tightly centered around a mean of $0.08$, with maximum gap observed in these experiments being approximately $0.16$ and a minimum $0.035$. For HMC, on the other hand, the mean gap is $0.11$, which is higher than VI, but this mean is affected by a small proportion of inputs that have a very high gap between upper and lower bounds, with the highest gap being $0.72$. We further highlight the higher variance bound distribution for the HMC-trained BNN (plotted as blue and red box plots). We hypothesize that this arises due to the higher uncertainty predictive of the HMC in areas of little data \cite{neal2012bayesian}

\subsubsection{Computational Requirements} %\ap{I would first give the scale of the computation requireds (1.8 millions verification is impressive) and then say how much it took. I don't think it is necessary to give the serial, since you have specified the CPU you are using, that is, imo, enough to give the parallel time. I don't fully get the times given. You say you have 4 CPUs of 24 cores each, but the results make it sounds like you have 4 CPUs of 1 core each.}  
For VI certification, we can compute upper and lower bounds in an average of 0.544 seconds. 
Thus, when run in serial mode, the 3.6 million probabilistic bound computations needed for Table~\ref{tab:HCASTable} takes an estimated 11.347 computational days. However, our parallelized certification %\MK{has it been parallelised?}\MRW{Yes.} 
procedure produces Table~\ref{tab:HCASTable} in under 3 days of computational time (61 hours). These computations include the 1.8 million lower bound runs and 1.8 million upper bound runs. For HMC, on the other hand, certification can be done in a fraction of the time, with bounds being computed in 0.073 seconds. This is due to the fact that weight intervals for HMC necessarily satisfy the pairwise disjoint precondition of Proposition~\ref{Prop:SafetyComputation2}, % \ap{What are you referring to in here is not clear! You mean proposition? Refer properly through the command, not with the number}
thus no Bonferroni correction is needed. 
%We highlight a large discrepency in the certification times for VI and HMC posteriors due to the computational cost of Bonferroni bounds needed in the case of VI posteriors, see Appendix~\ref{appendix:bonferroni}.

%\subsubsection{Limitations}

%In the lower row Figure~\ref{fig:HCASAnalysis} we analyze the decision robustness of the HCAS system. In order to do so, we take a coarser descritization of the input space resulting in thousands of input regions with width $\epsilon = 0.125$. In the bottom left, we compare our bounds on a BNN inferred with HMC to bounds on a BNN inferred with VI. We find that the HMC BNNs not only perform better but have more robust behavior. This is inline with the other empirical observations on the robustness of BNNs \cite{bekasov2018bayesian, carbone2020robustness}. Further research is needed to understand if HMC inferred BNNs are inherently more certifiable than VI inferred BNNs. 
%In the center and right plots of Figure~\ref{fig:HCASAnalysis} (bottom row), we plot the tightness of our bounds on the VI inferred BNNs. We find that the gap between upper and lower bounds on $\mathbb{E}_{\post} \sigma(f^{w}(x))$ is usually less than 0.1, indicating that our bounds on the expectation are tight. \ap{For Matthew: Let's do this in a systematic way, we need also the bound gap on HMC, otherwise it looks sloppy.}

\begin{figure*}[h]
    \centering
    \includegraphics[width=0.725\textwidth]{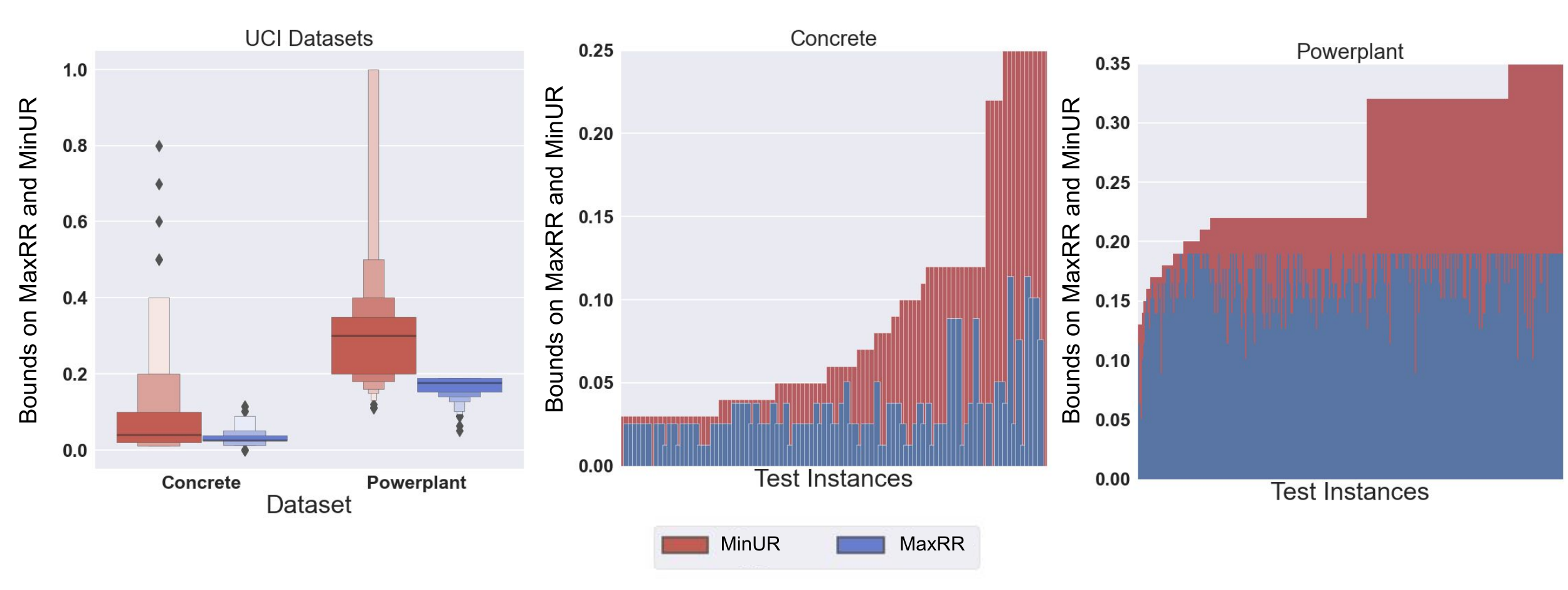}
    \caption{ 
    Computation of the minimum (MinUR) and maximum (MaxRR) safe radius for \textit{Concrete} and \textit{Powerplant} datasets. \textbf{Left:} Boxplots for the empirical distribution of MinUR and MaxRR computed over all test inputs. \textbf{Centre:} Per-test-instance certified radii for the \textit{Concrete} dataset. \textbf{Right:} Per-test-instance certified radii for the \textit{Powerplant} dataset.}
    \label{fig:UCIAnalysis}
\end{figure*}

\subsection{Industrial Benchmarks}\label{subsec:UCIBench}
We now focus our analyses on two safety-critical industrial regression problems taken from the UCI database \cite{ucidata}, and widely employed for benchmarking of Bayesian inference methods \cite{hernandez2015probabilistic, gal2016dropout, khan2018fast}. 

%\begin{itemize}
    %\item \LL{Complete definition.} for max and min RR
    %\item \ap{NUMBER OF VERIFICATIONS PERFORMED}
    %\item \ap{AVERAGE VERIFICATION TIME}
    %\item \ap{Do we have an idea why that is the case? Is powerplant more difficult / larger? Do we obtain similar accuracies in both case or not? Maybe there is some more sensitive features in the second case?}
    %\item On the figure for this section: \ap{Upper and Lower in the legend should be MaxRR and MinRR} \ap{linfinity norm in the axis label looks quite ugly, I'd remove that - and just keep the explanation in the main text. You could just use $\epsilon$ for the y-axis abel}
%\end{itemize}

\subsubsection{Problem Setting}

The \textit{Concrete} dataset involves predicting the compressive strength of concrete, based on 8 key factors including its ingredients and age. 
The \textit{Powerplant} dataset uses six years worth of observations from combined cycle power plants and poses the problem of predicting energy output from a plant given a range of environmental and plant parameters. For each dataset we learn a BNN by using the architecture (i.e., a single hidden layer with 100 hidden units) and inference settings proposed in \cite{khan2018fast}. The BNNs are inferred using VOGN with a diagnonal covariance prior over the weights with variance 0.5 for the \textit{Concrete} dataset and 0.25 for the \textit{Powerplant} dataset. We use a Gaussian likelihood corresponding to a mean squarred error loss function. In this setting we use IBP with 10 samples and a weight margin of 2 standard deviations.

\subsubsection{Analysis with $\psafe$ Certification}

%\ap{I think, rather than stating that this can be "very useful", you should try to pull a citation out of somewhere that states something on this line. Otherwise it sounds like a bachelor thesis :P .}
In industrial applications it is useful to understand the maximum amount of adversarial noise that a learned system can tolerate, as failures can be costly and unsafe \cite{shakiba2022robustness}. To this end, we introduce the maximum and minimum robust radius. Given a threshold $\tau_{\text{safe}}$ (as before) the maximum robust radius (MaxRR) is the largest $\ell_\infty$ radius for which we can certify the BNN satisfies $\psafe > \tau_{\text{safe}}$. Similarly, the minimum unrobust radius (MinUR) is the smallest radius such that we can certify  $\psafe < \tau_{\text{unsafe}}$. The MaxRR gives us a safe lower bound on the amount of adverarial noise a BNN is robust against, whereas the MinUR gives us a corresponding upper bound.

%We aim at computing the maximum (MaxRR) and minimum (MinRR) safe radius for $\psafe$ and on every test instance included in the two datasets, defined respectively as the minimum and minimum radius of $T$ for which  $\psafe$ \LL{Complete definition.}
%As explained in Section \ref{par:metrics} this involves the bounding of $\psafe$ at different values $\epsilon >0$ of the width of the input specification $T$.
%Notice that this is an iteration over each individual checks as those we have performed in Section \ref{subsec:prob_hcas} for the airborne collision dataset.
In our experiments on these datasets we considered $\tau_{\text{safe}} = \tau_{\text{unsafe}} = 0.7$, meaning that we request that over $70\%$ of the BNN probability mass is certifiably safe; however, we stress that similar results can be obtained for different values of $\tau_{\text{safe}}$  similarly to what is discussed in our previous analysis of the HCAS dataset. In order to compute the MaxRR we start with $\epsilon = 0$, check that $\psafe > \tau_{\text{safe}}$ using our lower bound, and if the inequality is satisfied we increase epsilon by 0.01 and  continue this process until the inequality no longer holds. Similarly for the MinUR, we start with $\epsilon = 0.5$ %\ap{From the figure it seems you have $MinUR > 0.5$. How is that possible?} 
and iteratively decrease the value of $\epsilon$ until the upper bound no longer certifies that $\psafe < \tau_{\text{unsafe}}$; if the property does not hold at $0.5$ one can increase the value of $\epsilon$ until the bound holds.

The result of computing the MaxRR and MinUR over the test datasets for the \textit{Concrete} and \textit{Powerplant} datasets are plotted in Figure~\ref{fig:UCIAnalysis}.
We highlight that in the overwhelming majority of the cases our methods is able to return non-vacuous bounds on MinUR (i.e., strictly less than $1$) and on MaxRR (i.e., strictly greater than $0$).
As expected we observe the MaxRR is strictly smaller than MinUR. Encouragingly, as MinUR grows, MaxRR tends to increase indicating that our bounds track the true value of $\psafe$.
We see that the \textit{Concrete} dataset is typically guaranteed to be robust for radius $\epsilon \approx 0.03$ and is typically guaranteed to be unsafe for $\epsilon \approx 0.06$. For the \textit{Powerplant} posterior we compute a MaxRR of roughly $0.18$ for most inputs and a MinUR lower than $0.32$. 
Notice how the results for the \textit{Concrete} datasets systematically display more robustness than those for \textit{Powerplant} and the gap between MaxRR and MinUR is significantly smaller in the former datasets than in the latter. 

\subsubsection{Computational Requirements}
On average, it takes 1.484 seconds to compute a certified upper or lower bound on $\psafe$ for the \textit{Powerplant} dataset and 1.718 seconds for the \textit{Concrete} dataset. We use a linear search in order to compute the MaxRR and MinUR which require, on average, 5 certifications for both MaxRR and MinUR computations. We compute these values over the entire test datasets for both \textit{Powerplant} and \textit{Concrete}, which requires tens of thousands of certifications and each input can be done in parallel.

\begin{figure*}
    \centering
    \subfigure{\includegraphics[width=0.36\textwidth]{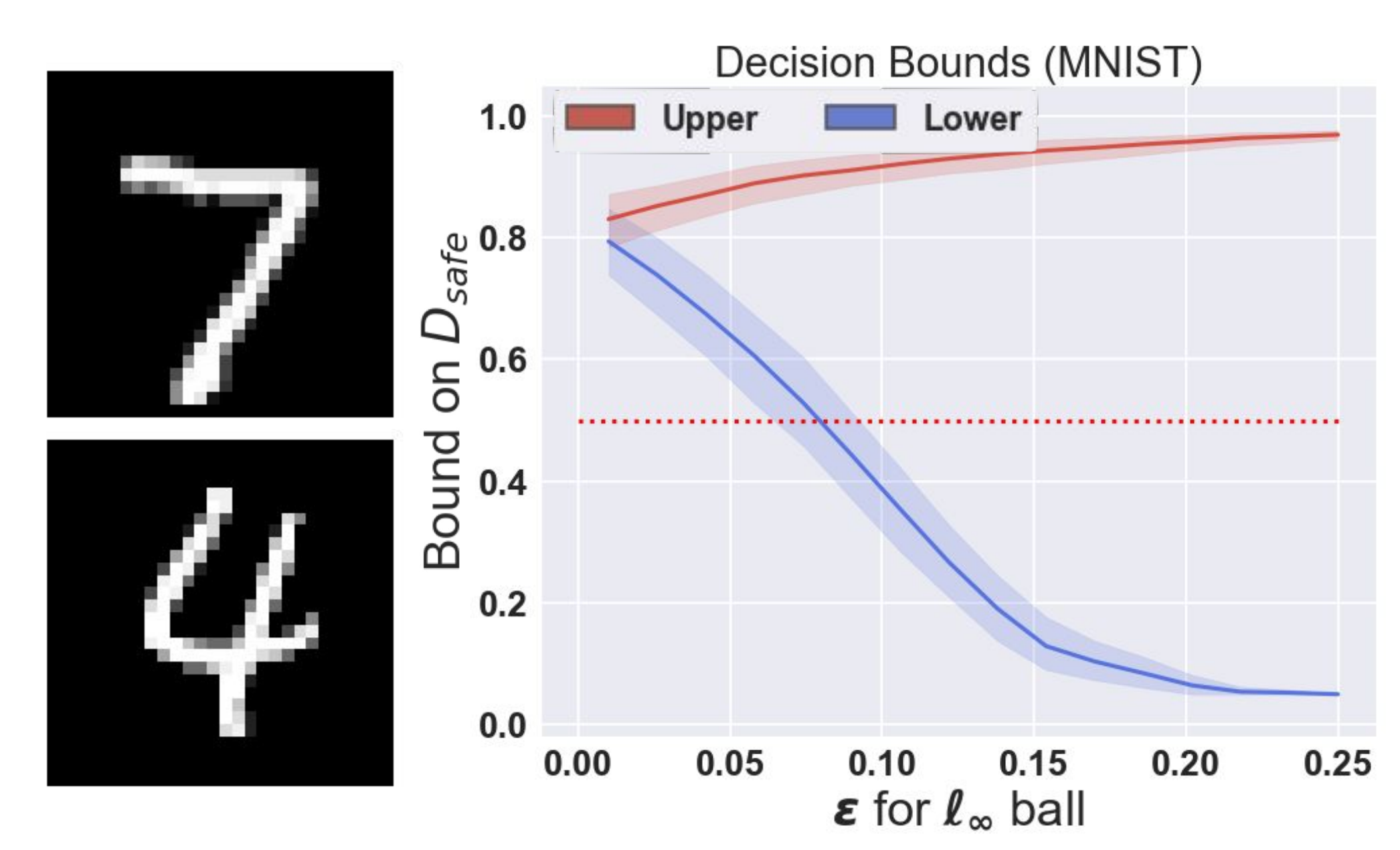}} 
    \subfigure{\includegraphics[width=0.36\textwidth]{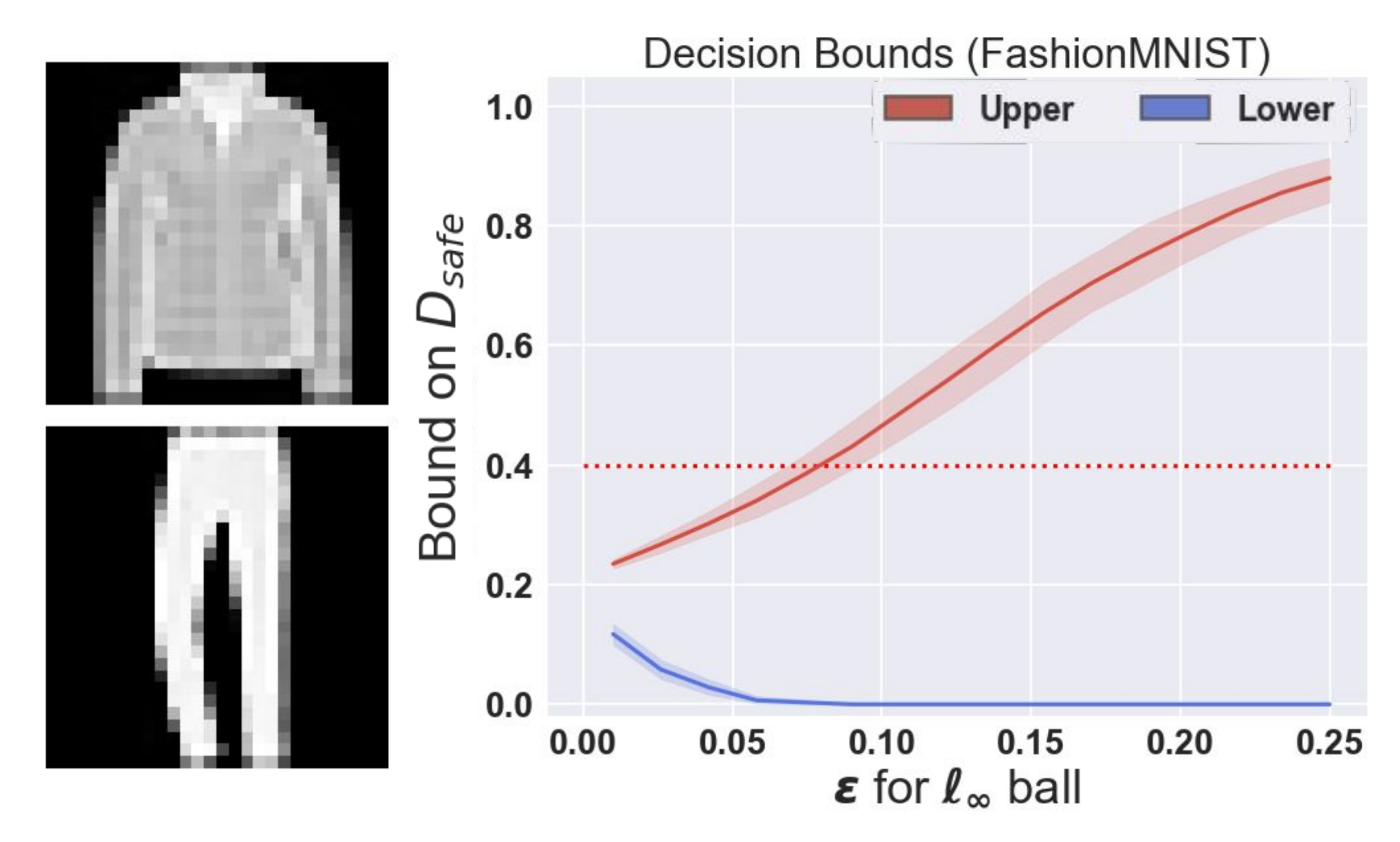}}
    \caption{\textbf{Left:} Mean and standard deviation of upper and lower bounds obtained on $D_{\text{safe}}$ on 1000 images taken from the MNIST test set. \textbf{Right:} Mean and standard deviation  on upper and lower bound on decision robustness on out-of-distribution samples taken from the FashionMNIST dataset.}
    \label{fig:MNISTAnalysis}
\end{figure*}

\begin{figure}
    \centering
    \subfigure{\includegraphics[width=0.425\textwidth]{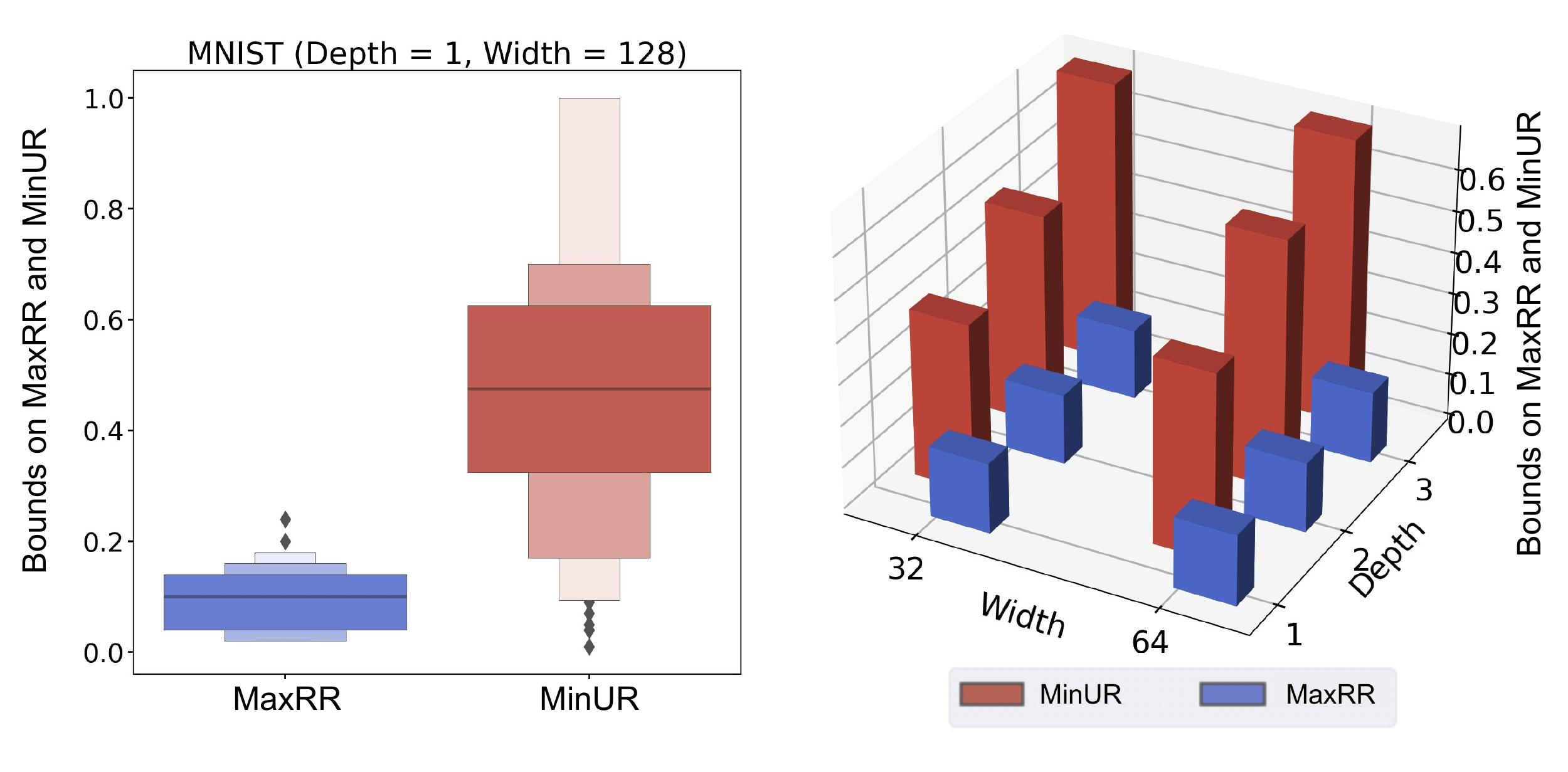}}
    \caption{\textbf{Left:} Boxplots for the empirical distribution of the maximum safe radius and minimum unsafe radius for a BNN with 128 hidden units and a single hidden layer. \textbf{Right:} For a range of architectures, we plot the mean certified safe and unsafe radius. }
    \label{fig:MNISTArch}
\end{figure}

%\begin{itemize}
    %\item \ap{Details of the architecture given somewhere? Perhaps in the Appendix?}
    %\item \LL{Not clear what overall means}
    %\item \LL{Need to define how you check if a BNN is uncertain or not. In your case if. I got it correctly you just check at the expectation of the softmax of the true class right?}
    %\item About figure 7: \ap{Upper and lower should respectively be $D_{\text{safe}}^U$ and $D_{\text{safe}}^L$, we have notation introduced for these quantities, let's use it to be more clear on what is being plotted. Putting the certification threshold a 0.5 is a very conservative overestimation given that the dataset is made of 10 classes, I would keep the line becuse it helps understanding what is going on, but I'd remove the legend about it. The y-axis, Decision Probability, is a term never introduced before, and that doesn't mean much. Either introduce a notation for the predictive probability – or simply put $D_{\text{safe}}$. $\epsilon$ for linfinite ball is a sloppy label, just put $\epsilon$, and we will make it clear from the text.}
    %\item About figure 7: \ap{Not sure why the plots legends is so small}
    %\item About figure 8 \ap{Upper and Lower need to be replaced by what they actually are. ie MinRR and MaxRR.}
%\end{itemize}
\subsection{MNIST}\label{sec:mnist_results}
We investigate the suitability of our methods in providing certifications for BNNs on larger input domains, specifically BNNs learnt for MNIST, a standard benchmark for the verification of deterministic neural networks whose inputs are 784-dimensional. In this setting we use IBP with 5 weight samples with a weight margin of 2.5 standard deviations.

\subsubsection{Problem Setting}

MNIST poses the problem of handwritten digit recognition. Given handwritten digits encoded as a 28 by 28 %(784 dimensional) 
black and white image, the task is to predict which digit – 0 through 9 – is depicted in the image (two images randomly sampled from the dataset are reproduced in the far left of Figure~\ref{fig:MNISTAnalysis}).
We learn BNNs using the standard 50,000/10,000 train/test split that is provided in the original work \cite{lecun1998mnist}. For our experimental analysis, we use one-layer neural network with 128 hidden neurons, each of which uses rectified linear unit activation functions. The BNN has 10 output neurons that use a softmax activation function. We train the network using VOGN with a diagonal covariance Gaussian prior that has variance 2.0. We use a sparse categorical cross-entropy loss modified with the method presented in \cite{wicker2021bayesian} to promote robustness in the BNN posteriors. %In our analysis we vary the architectural parameters, but keep all of the remaining training settings constant. %and investigate a number of different architectures using the variational inference algorithm for posterior estimation. We perform several analyses pertaining the behaviour of our certification methods w.r.t.\ the size of the input specification $T$, the number of layers and hidden units in the BNN architecture, and the behaviour on out-of-distribution samples.

%A few examples are given on the far left of Figure~\ref{fig:MNISTAnalysis}.  In order to learn on MNIST we use the standard 50,000/10,000 test/train split that is provided \cite{lecun1998mnist}. For each tested posterior, we again train using the VOGN algorithm from \cite{khan2018fast}.  In this section, we will study several different architecture configurations; we tune the learning rate parameters for each architecture individually. Each architecture uses rectified linear units as activation functions and a softmax as its final layer. % We take the likelihood of the model to be the sparse categorical cross-entropy. 

\subsubsection{Analysis using $\dsafe$ Certification}

We analyze the trained BNN using decision robustness on 1000 %one thousand 
images taken from the MNIST test dataset. We compute bounds on $\dsafe$ for increasing widths of an $\ell_\infty$ input region $\epsilon$. %\ap{Do these extra ldots have a meaning or is it a typo?}
We plot the mean and standard deviation obtained for the upper ($D_{\text{safe}}^U$, in red) and lower bound ($D_{\text{safe}}^L$, in blue) on decision robustness for the ground truth label of each image in the left hand portion of Figure \ref{fig:MNISTAnalysis}. 
As greater $\epsilon$ implies a larger input specification $T$, increasing values of $\epsilon$ leads to a widening of the gap between the lower and upper bounds, and hence an increased vulnerability of the network. 
Notice that even for $\epsilon = 0.25$, i.e., half of the whole input space, our method still obtains on average non-vacuous bounds (i.e., strictly within $(0,1)$).
In order to get a rough estimation of the adversarial robustness of the network, we observe that, for lower bound values above $0.5$, the BNN is guaranteed to correctly classify all the inputs in the region $T$ (however, as MNIST has 10 classes, even values of the lower bound lower than $0.5$ could still result in correct classification).
Using the $0.5$ threshold, we notice that our method guarantees that the BNN is still robust on average for $\epsilon = 0.075$. Notice that this is on par with results obtained for verification of deterministic neural networks, where $\epsilon = 0.05$ leads to adversarial attack robustness of around $70\%$ \cite{madry2017pgd}.

%In Figure~\ref{fig:MNISTAnalysis} we show how our bounds are affected by the radius of the $\ell_\infty$ ball around an input image. We see that for small radii ($< 0.06$) we are able to certify the correct classification w.r.t. the posterior predictive distribution. A certified robust radius of $0.06$ would be considered robust by deterministic neural network standards as adversarial attacks are readily found for DNNs with $\epsilon = 0.05$ the accuracy has typically decreased below 70\% \cite{madry2017pgd}. 

\subsubsection{Certification of Uncertainty Behaviour}

In this section we study how to certify the uncertainty behavior of a BNN in the presence of adversarial noise. We assume we have an out-of-distribution input, i.e., an input whose ground-truth does not belong to any of the classes in the range of the learned model. As with previous specifications, we build the set $T$ around such an input with an $\ell_\infty$ ball of radius $\epsilon$. Unlike for the previous specifications, we build $S$ as the set of all softmax vectors such that no entry in the vector is larger than a specified value $\tau_{\text{uncertain}}$. The function of $\tau_{\text{uncertain}}$ is to determine the confidence at which a classification is ruled to be uncertain. For example, in Figure~\ref{fig:MNISTAnalysis} we have set $\tau_{\text{uncertain}} = 0.4$, thus any classification that is made with confidence $<0.4$ will be ruled uncertain. By certifying that all values of $T$ are mapped into $S$, we guarantee that the BNN is uncertain on all points around the out-of-distribution input. 

In the right half of Figure~\ref{fig:MNISTAnalysis}, we plot two example images from the FashionMNIST dataset,  which are considered out-of-distribution for the BNN trained on MNIST. In our experiments we use 1000 test set images from the FashionMNIST dataset. On the right of the out-of-distribution samples in Figure~\ref{fig:MNISTAnalysis}, we plot the bounds on decision robustness with various values of $\epsilon$ for the $\ell_\infty$ ball. We start by noticing that the BNN never outputs a confidence of more than $\sim 0.25$ on the clean Fashion-MNIST dataset, which indicates that the network has good calibrated uncertainty on these samples. We notice that up to $\epsilon = 0.06$ we certify that no adversary can perturb the image to force a confident classification; however, at $\epsilon = 0.10$ no guarantees can be made.

\subsubsection{Architecture Width and Depth}
%\ap{Some reviewer will definately complain about the max width being only 64. I would suggest to include a small statement on the line of what kind of networks have been verified in the literature so far for MNIST, *if* they are in any way comparable.}
We now analyse the behaviour of our method when computing bounds on the certified radius on MNIST while varying the width and depth of the BNN architecture. % We select the output specification $S$ to represent the set such that, for each input point, the true class prediction of the BNN is the maximal one (i.e., we ensure that the classification is correct and robust).
The results of this analysis are given in Figure \ref{fig:MNISTArch}.
Notice that we are able to obtain non-vacuous bounds in all the cases analysed. However, as could be expected, we see that the gap between MinUR and MaxRR widens as we increase the depth and/or the width of the neural network.
This inevitably arises from the fact that the tightness of bound propagation techniques decreases as we need to perform more boundings and/or propagations, and because increasing the number of weights in the network renders the bounds obtained by Proposition \ref{prop:pred_posterior_bound} more coarse, particularly as we increase the number of layers of the BNN, as explained in Section \ref{sec:time_complexity}.
In particular, we observe that MinRR increases drastically as we increase the number of layers in the BNN architecture, while, empirically, the bounding for MaxRR is more stable w.r.t.\ the architecture parameters. % We highlight that these insights have been noted \MK{but this is premilinary version?} by previous works on BNNs \cite{wicker2020probabilistic} with architectures of similar size, where they do not analyze both upper and lower bounds. 

\subsubsection{Computational Requirements} On average, it takes 24.765 seconds to verify an MNIST image on a single CPU core. Each of the images in our experiments is run in parallel across 96 cores which allows us to compute all of the results for Figure~\ref{fig:MNISTAnalysis} in less than an hour. 

%\subsubsection{Limitations}

%In order to study how the choice of architecture affects the provable robustness of BNNs we vary both the width and depth of a fully connected BNN. We take the safe outputs, $S$, to be the set of softmax vectors such that the true class is maximal (i.e., we ensure that the classification is correct and robust). As we increase the width of the neural network (keeping all else equal, including the parameters of the algorithm) we see a slight decrease in MaxRR. However, as we increase depth we observe a noticeable increase in the computed MinRR. This indicates that increasing the depth of a BNN increases the computational burden of certification more than increasing the width. We hypothesize that this comes from the approximation introduced by convex relaxation techniques.

\begin{figure}
    \centering
    \subfigure{\includegraphics[width=0.425\textwidth]{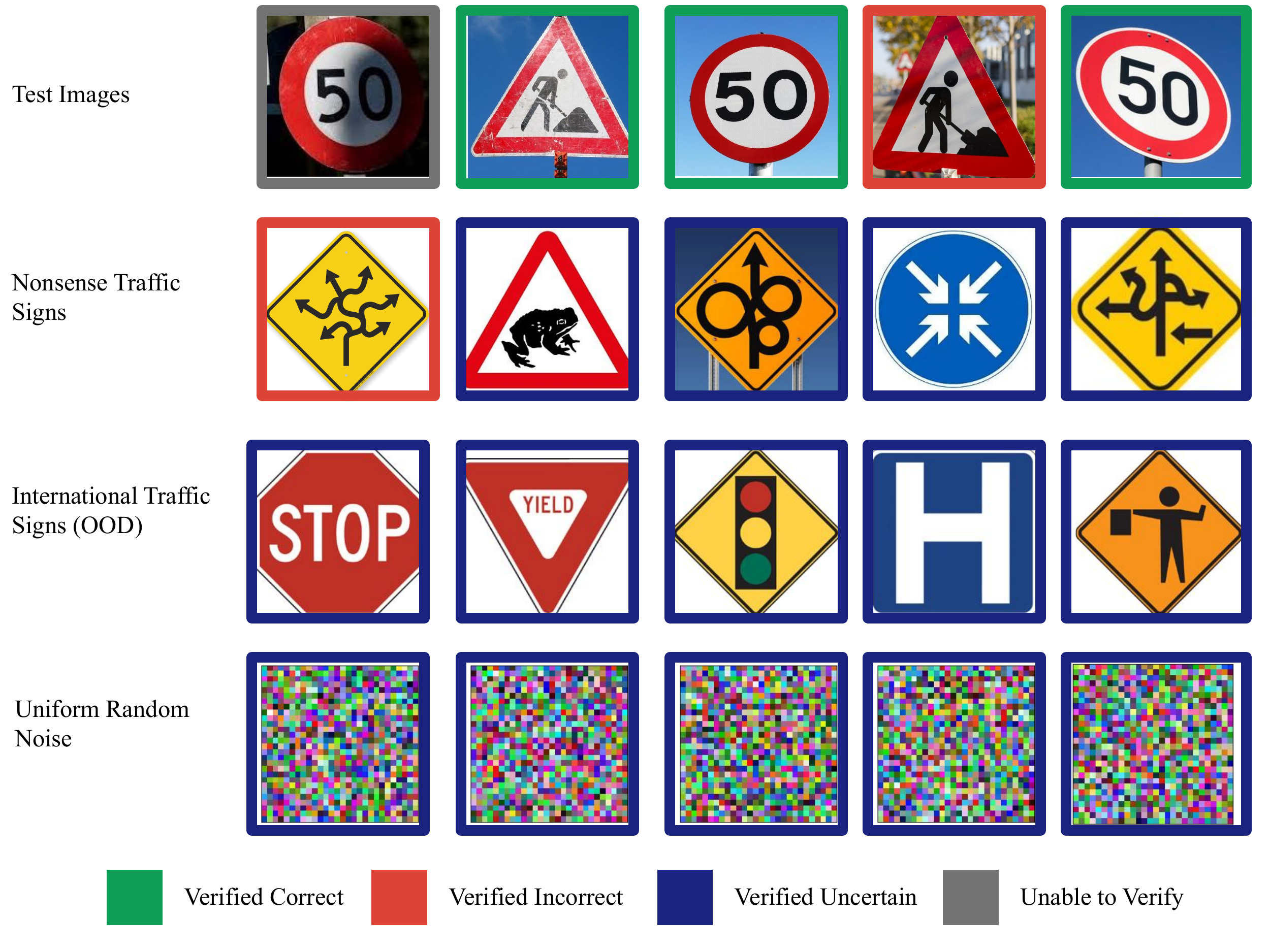}} 
    \caption{Certification of a Bayesian CNN on a two-class subset of the German Traffic Recognition (GTSRB) dataset. In the top row, we plot illustrative examples showing that we can verify correctness of test inputs. In the bottom three rows we visualize the the uncertainty guarantees on various out-of-distribution inputs including nonsense traffic signs (second row), international traffic signs (third row), and random noise (bottom row).
    }
    \label{fig:GTSRBAnalysis}
\end{figure}

%\begin{itemize}
    %\item \LL{Why the GTSR dataset is a paragraph of the MNIST subsection? SHould we call this all subsection Image Recognition instead of MNIST?} 
    %\item \ap{???} (Epsilon unstated)
    %\item \ap{Report average computational time}
    %\item  Overall... \ap{this analysis is a bit meh, It's just a few images with some frames around - underwhelming. Are these all the images you have verified? Can we run the experiments say on at least 50 of them, and we report the various percentages. don't need to make any new plot, but just report some number in the main text itself. }
    %\item \ap{You also need to put some discussion on why you think certain points (for example the top left one) are unverifiable etc. Otherwise is not an analysis, only a list of random results.}
    %\item \ap{If this is done anywhere else in the NN literature, it would be good to cite it}
    %\item \ap{Also this experiment, it would be much better if instead of 15 images, there is something like 100, so that we could report some numbers, other than just some anecdotal evidence like this.}
%\end{itemize}

\subsection{German Traffic Sign Recognition} 

In this section, we investigate the ability of our method to scale to a full-color image dataset, which represent safety-critical tasks with high-dimensional inputs ($2,352$ dimensions). 

\subsubsection{Problem Setting}

We study BNNs on a two-class subset of the German Traffic Sign Recognition Benchmark (GTSRB), consisting of the images that represent the `construction ahead' and `50 Km/H speed limit' \cite{stallkamp2012gtsrb}. Though this dataset is only comprised of two classes, full-colour images stretch the capabilities of BNN training methods, especially robust Bayesian inference. The dataset is comprised of 5000 training images and 1000 test set images. We employ VOGN to train a Bayesian convolutional architecture, with 2 convolutional layers and one fully-connected layer first proposed in \cite{gowal2018effectiveness}. We employ the method of \cite{wicker2021bayesian} in order to encourage robustness in the posterior. We find that this dataset poses a challenge to robust inference methods, with the BNN achieving 72\% accuracy over the test set after 200 epochs. We found that, without robust training, we are able to achieve 98\% accuracy over the test set, but were unable to certify robustness or uncertainty for any tested image (see  discussion of limitations below).
%in Section \ref{sec:gtsrb_limitations} below). 
We verify these networks with 3 weight samples with a weight margin of 3.0 standard deviations.

\subsubsection{Analysis with $\dsafe$}

For our analysis for GTSRB, we take $T$ to be a $\ell_\infty$ ball with radius $2/255$. As in our previous analysis, for test set images we take $S$ to be the set of all vectors such that the true class is the argmax. %For this analysis 
We study 250 images and find that 53.8\% of the images are certified to be correct. We plot a visual sample of these images in the top row of Figure~\ref{fig:GTSRBAnalysis}. We also study the out-of-distribution performance of various kinds of images with $\tau_{\text{uncertain}} = 0.55$. Of 400 images of random noise, visualized in the bottom row of Figure~\ref{fig:GTSRBAnalysis}, we certified that the BNN was uncertain on 398 images, indicating that on that set the BNN has correctly calibrated uncertainty as it does not issue confident predictions on random noise. We then turned our attention to two more realistic sets of out-of-distribution images: nonsense traffic signs and international traffic signs. We were limited to a small set of free-use images for these tests but found that for eight out of ten nonsense traffic signs we were able to certify the BNN's uncertainty, and for nine out of ten international traffic signs we were able to certify the uncertainty.  On average these certifications took 34.2 seconds.
%For our initial analysis, we provide certification on the correct behaviour of the BNN for input regions of radius $\epsilon = 2/255$. Illustrative results on this analysis are given in the first row of Figure~\ref{fig:GTSRBAnalysis}, where each image is framed with different color depending on whether it is certifiably robust (green), unrobust (red), or our method cannot provide any guarantees (grey). When run over 250 test set images we find that the BNN is proven to be correct and adversarailly robust on 53.8\% of the images. Each individual certification took an average computational time of approximately 34.271 seconds. In the remaining rows of Figure~\ref{fig:GTSRBAnalysis} we analyse the robustness of the BNN on making uncertain predictions around out-of-distribution samples – a task which is of particular relevance in autonomous driving settings. For this purpose we collect free-use images online representing nonsense traffic signs, traffic signs from the United States, and we generate random noise. As this is a two-class classification problem, for illustrative purposes we set the uncertainty threshold at the $0.55$ prediction level and we furthermore set $\epsilon =  2/255$.  Interestingly, we find that for all but one image in this set we are able to certify that the BNN is correctly uncertain.
%In addition, we tested the uncertainty of the BNN over 400 random additional noise images. Of these 400 we found that 398 were certified to have an upper bound on $\dsafe$ of less than $0.55$.

\subsubsection{Limitations}\label{sec:gtsrb_limitations}
While this analysis represents an encouraging proof of concept for certification of BNNs, we find that datasets whose inputs are of this scale and complexity are not yet fully accessible to robust inference for BNNs, as 74\% test set accuracy is not strong enough performance to warrant deployment. However, with approaches such as \cite{osawa2019practical, izmailov2021bayesian} investigating more powerful methods for scaling Bayesian inference for neural networks, we are optimistic that future works will be able to apply our method to more advanced Bayesian approximate posteriors.

%\vspace{2 cm}

% We find that for this BNN we are able to certify the robust prediction of many test-set inputs which can be seen in Figures~\ref{fig:GTSRBTitleFig} and \ref{fig:GTSRBAnalysis}. Given the highly unpredictable nature of autonomous navigation it is also highly desirable to certify the uncertainty of out-of-distribution points on this task. For this we collect free-use images online representing non-sense traffic signs, traffic signs from the United States, and we generate random noise. These three categories represent images on which our Bayesian CNN ought to be uncertain. Interestingly, we find that for all but one image in this set we are able to certify that the BNN is correctly uncertain with decision, posterior predictive, softmax guaranteed to be no greater than 0.55 for all points in a $\ell_\infty$ ball of radius 2/255.
%\input{Sections/Limitations}
\section{Conclusion}
In this work, we introduced a computational framework for evaluating robustness properties of BNNs operating under adversarial settings. In particular, we have discussed how probabilistic robustness and decision robustness -- both employed in the adversarial robustness literature for Bayesian machine learning \cite{cardelli2018robustness, smith2019adversarial} -- can be upper- and lower-bounded via a combination of posterior sampling, integral computation over boxes and bound propagation techniques. 
We have %provided explicit 
detailed how to compute these properties   %the computation can be made in both the case 
for the case
of HMC and VI posterior approximation, and how to instantiate the bounds for interval and linear propagation techniques, although the framework presented is general and can be adapted to different inference techniques and to most of the verification techniques employed for deterministic neural networks.

In an experimental analysis comprising 5 datasets (airborne collision avoidance, concrete, powerplant, MNIST, and GTSRB), we have showcased the suitability of our approach for computing effective robustness bounds in practice, and for various additional measures that can be computed using our technique including certified robust radius and analysis of uncertainty. 

With verification of deterministic neural networks already being NP-hard, inevitably certification of Bayesian neural networks poses several practical challenges.
The main limitation of the approach presented here arises directly from the Bayesian nature of the model analysed, i.e., the need to bound and partition at the weight space level (which is not needed for deterministic neural networks, with the weight fixed to a specific value). Unfortunately, this means that the computational complexity, and also the tightness of the bounds provided, scale quadratically with the number of neurons across successive layer connections. We have discussed methods for mitigating the resulting gap between the bounds, including adaptive partitioning based on weight variance and implementing a  branch-and-bound refinement approach for the bound, which would, however, results in a sharp increase in computational time.   
Nevertheless, the methods presented here provide the first formal technique for the verification of robustness in Bayesian neural network systematically and across various robustness notions, and as such can 
provide a sound basis for %provide significant 
%practical impact 
for future practical applications in safety-critical scenarios.

%\vspace{2cm}
%In this work we study certification of Bayesian neural networks (BNNs). We present two definitions of robustness for BNNs which can be used by both practitioners and regulators to ensure that a BNNs predictions are safe and well-calibrated. We validate that our method can guarantee BNN performance w.r.t. notions of robustness in the face of adversarial examples as well as guaranteed uncertainty on out-of-distribution points. We present upper and lower bounds on both of these definitions which allow users to formally certify the safety of a BNN prior to deployment. Our empirical investigation provides us with evidence that these algorithms are able strong certification guarantees for BNNs in safety-critical scenarios. This includes comprehensive verification of 1.8 million input specifications on an airborne collision avoidance dataset and safety verification of traffic sign classification task with full color input images.  
\bibliographystyle{IEEEtran}
\bibliography{bib}

\clearpage
\newpage
\appendices

%\ap{Doesn't the appendix need some sort of header? It's weird that it just starts like this}\MRW{This is the style guide way of doing it. }
In this Appendix, we provide further discussion, proofs, and hyper-parameters for reproducing our experiments.  We begin with a discussion on the use cases for the bounds. %, outlined in Table~\ref{tab:propertycomputation}. 
We then move on to discuss other decision criteria that we can certify. Following this, we provide proofs for the theoretical results stated in the main paper, algorithms for lower bounding  $\dsafe$ and upper bounding both $\psafe$ and $\dsafe$, as well as computing the cumulative probability of a hyper-rectangle using Bonferroni bounds. Finally, we summarise the training details, including all hyper-parameters needed to reproduce our results.

\section{Bound Use Case Discussion}

\begin{table}[]\caption{\label{tab:propertycomputation}  Use cases for each bound presented in this paper.}
\begin{tabular}{|l|c|c|}
\hline
\multicolumn{1}{|c|}{Property:} & App:   & Bound for Certification: \\ \hline \hline
Correctness                     & Cls. & Lower \& Upper on $\dsafe$      \\ \hline
Aleatoric Uncert.          & Cls. & Upper on $\dsafe$               \\ \hline
Epistemic Uncert. (OOD)     & Cls. & Lower \& Upper on $\psafe$      \\ \hline
Correctness                     & Reg.     & Lower \& Upper on $\dsafe$      \\ \hline
Aleatoric Uncert.          & Reg.     & Lower \& Upper on $\psafe$      \\ \hline
Epistemic Uncert. (OOD)     & Reg.     & Lower \& Upper on $\psafe$      \\ \hline
\end{tabular}
\end{table}

%\MK{Very informal and loose description, needs to be turned into a methodology }
We use this section to present the use cases for each bound we derive in this paper and highlight their importance. %We decompose our discussion into different properties of interest for BNNs. 
A summary of the use cases we suggest is given in Table~\ref{tab:propertycomputation}. %For determinsitic neural networks the only property that is generally considered is correctness i.e., can the adversary change the prediction of a given input. Discussion of this property and others can be found in \cite{dvijotham2018verification}. In the Bayesian setting, modellers are often concerened with \textit{uncertainty}. Below, 
We discuss various kinds of uncertainty quantities (full definitions and discussion in \cite{gal2016uncertainty}) as well as how one can gain relevant certification on these quantities. %Due to the fundemental difference that classification involves discrete outputs and regression involves continuous output objective their certification is treated differently, we provide further discussion below. 

%\ap{Since this discussion is not supposed to be something super new specific to this paper, I would put here, a list of references that the reader can read. I think it removes from us the responsability of justifying fully the various properties. Something like, 'Discussions on various certification properties needed for machine learning can be found here here here and there'.}
%\ap{Regression and classification have different entries in the table, point that out and explain why with a brief sentence}

\paragraph{Correctness} One of the most widely studied properties in NN robustness is that of ``correctness'' \cite{katz2017reluplex}, which requires prediction of the NN to match the ground truth even in the face of adversarial perturbations. For classification, as discussed in the main text, correctness boils down to checking that, for all adversarial perturbations, the argmax of the softmax output remains the same. For regression, due to the continuous nature of outputs correctness involves establishing a range of outputs that correspond to the tolerable error.
%Here we refer to a correctness property as the property of a NN (in our case a BNN) predicting the correct label (or within tolerable error for regression) for every input in $T$. 
Given that correctness relies on the ultimate decision of the BNN in either the classification or regression, we use upper and lower bounds on the posterior predictive expectation (i.e., $D_{\text{safe}}^L$ and $D_{\text{safe}}^U$). To prove that classification is correct, one must prove that the lower bound of the true class softmax probability is higher than the upper bound of all other classes softmax probability, which implies:
$$ \forall x \in T, \ \ \argmax 
\mathbb{E} \big[ \sigma(f^{w}(x)) \big] = c$$
For regression one must use upper and lower bounds in order to show that output prediction lies within tolerable error. For this, one needs to check that the end points of the decision, $[D^{L}_{\text{safe}}, D^{L}_{\text{safe}}]$, are contained within the tolerable noise.  

\paragraph{Aleatoric Uncertainty} Measures of aleatoric uncertainty are %irreducible \ap{What does irreducible mean?}, 
input-dependent and come from the noise within the data observation process \cite{gal2016uncertainty}. %Aleatoric uncertainty is often called irriducible due to the fact (unlike epsitemic uncertainty) increasing the amount of data you observe does not affect this source of uncertainty. Though difficult to disentangle at inference time, one typically uses the likelihood as a marker of aleatoric uncertainty \cite{bishop1995neural}. For classification one can use the softmax probability, the maximum value of the posterior predictive distribution, is taken as measure of aleatoric uncertainty. %Typically this value (the maximum of the expected softmax) is termed the `confidence.' % and the variance in the softmax probabilities arising from different sampled neural networks as the epistemic uncertainty (discussed below). 
For classification the aleatoric uncertainty is usually measured as $\max_{i \in [n]} \mathbb{E}_{\post} \sigma(f^{w}(x))_{i}$. This is also termed the `confidence.' For regression, one can predict both the mean and variance of a Gaussian likelihood, where the variance represents the aleatoric uncertainty \cite{nix1994estimating}.
%Thus a bound on the aleatoric uncertainty $\tau_{\text{uncertain}}$ is akin to placing a bound on confidence of the BNN. 
Computing bounds on the posterior predictive mean allows us to ensure that a point has sufficiently high or low aleatoric uncertainty. For classification, $D^{L}_{\text{safe}}$ represents a lower bound on $\mathbb{E}_{\post} \sigma(f^{w}(x))_{i}$, thus $D^{L}_{\text{safe}}$  allows for a bound on aleatoric uncertainty. For regression, the same holds save it is only the bound $D^{L}_{\text{safe}}$ in the dimension corresponding to the predicted variance. %For example: $D^{L}_{\text{safe}} > \tau_{\text{uncertain}}$ ensures that no adversary can reduce the aleatoric uncertainty below a threshold. Or for OOD points one can check that $D^{U}_{\text{safe}} < \tau_{\text{uncertain}}$ in order to ensure that no adversary can trick our model into producing a highly confidence prediction on an OOD. % \ap{Formula} 

%For regression, bounds on the mean tell us little about the uncertainty of the prediction. Typically the aleatoric uncertainty of a regression prediction is represented by a Gaussian with variance $\varsigma$. If one  wants to ensure that a BNN has a prediction spread that matches this variance $\varsigma$, then one could check the output spread of the BNN is in line with this assumption using bounds on $\psafe$. However, we note that this is not a property of significant interest, and so we do not discuss it further.  % Instead, we must rely on the spread predictions for a given point. If we would like check that our aleatoric uncertainty at an input $x$ is roughly calibrated according to a Gaussian with variance $\varsigma$, then we can ensure that for the output range $[f^{\mathbf{w}}(x)-\sqrt{\varsigma}, f^{\mathbf{w}}(x)+\sqrt{\varsigma}]$ we have that $\psafe = 0.67$ thus indicating that the BNN output respects the observation noise we expect. \ap{Rather than an example, write down the explicit formula, and then in case make the example.}
%\ap{OK but what information does a model get out of this? I guess something like that the problem is intrinsically noisy around T? Maybe one sentance about it?}

\paragraph{Epistemic Uncertainty} Model or epistemic uncertainty measures the uncertainty from the lack of data at training time. We expect that epistemic uncertainty is high for out-of-distribution samples. % which is indicated by a large variance in model behavior at a point. %Though one cannot disentangle epsitemic from aleatoric in the case of regression, %\ap{Is this a problem of our method, or a general problem?}
Epistemic uncertainty is measured as the spread of prediction from various models under the posterior distribution. To measure this, it is natural to consider the variance of the posterior predictive distribution. % probability from different models in the Bayesian posteriors as an indication of epistemic uncertainty. 
Given an input $x$ that is out-of-distribution one can certify that $P_{\text{safe}}$ is not sufficiently high for any class. This guarantees that there is not one class that the BNN maps all of its predictions into and thus guarantees that the BNN is uncertain. By checking $P_{\text{safe}}$ across in- and out-of-distribution points, modellers can certify that their BNN is well calibrated with respect to epistemic uncertainty. %This allow us to reason that model has diverse predictions at $x$.
%\ap{OK, but what information does the modeller obtain by all of this? One sentence about it would be interesting. I guess something like that T is around an out-of-distribution area?}

\section{Certifying Further Decision Rules}\label{appendix:decisionrules}

As discussed in the main paper, decision robustness is clearly dependent on the function used for Bayesian decisions on top of the learning model.
In the main paper we have given explicit results for the two standard losses, $\ell_{0-1}$ for classification and $\ell_2$ for regression. However, with some minor adjustments, our method can be employed for different losses too. In this section we give the example of the $\ell_1$ loss for regression and the weighted loss for classification.  

\subsection{Bounding Decisions for the $\ell_1$ Loss}
For the $\ell_1$ decision loss, it is known that the \textit{median} of the posterior predictive distribution is the value that minimizes the loss. Thus we must bound the median, defined as usual to be $m(Z) := x \iff \int_{-\infty}^{x} p_{Z}(v)dv = 0.5$. %In order to compute upper and lower bounds on the median using our method, first we consider the definition of a median. The median of a random variable $Z$ which is distributed according to $p_{Z}$ is given as $m(Z) := x \iff \int_{-\infty}^{x} p_{Z}(v)dv = 0.5$. That is, when the cumulative density or mass is 0.5. In order to lower bound this value, assume we have $N$ disjoint weight intervals $\{J_i\}_{i=1}^{N}$ as well as their corresponding lower bound outputs $\{y^{L}_{i}\}_{i=1}^{N}$. Firstly, we sort the two sets such that $y^{L}_{i} \leq y^{L}_{i+1}$. %Next, one can observe the sum $\sum_{i=0}^{M} P(J_i)$. 
Assuming $\sum_{i=1}^{N} P(J_i) = 1.0$, we can arrive at a lower bound by picking $y^{L}_m$ to be our median lower bound such that $\sum_{i=1}^{m} P(J_i) \leq 0.5$ but $\sum_{i=1}^{m+1} P(J_i) \geq 0.5$. One can similarly find an upper-bound via this routine by first computing upper bounds for each weight rectangle and then picking $y^{U}_m$ such that $\sum_{i=1}^{m} P(J_i) \geq 0.5$ but $\sum_{i=0}^{m-1} P(J_i) \leq 0.5$.  When the condition $\sum_{i=1}^{N} P(J_i) = 1.0$ does not hold, we can modify the procedure to get valid bounds on the median. We assume that $\sum_{i=1}^{N} P(J_i) = 1.0 - \eta$ for any $\eta$ such that $0.5 > \eta > 0$. Then we pick the lower bound to the median to be $y^{L}_m$ such that $\eta + \sum_{i=1}^{m} P(J_i) \leq 0.5$ but $\eta + \sum_{i=1}^{m+1} P(J_i) \geq 0.5$. This yields a valid bound on the median.
Similar formulas can be computed for the upper bound, by relying on the laws of complementary probabilities.

\subsection{Bounding Decisions for the $K-0$ Loss}

In some safety-critical decision-making problems, particularly in medical diagnosis, predicting one class comes with more risk (formally, loss) than predicting another. In this case, the 0-1 loss is made more general and is defined as the 0-$K$ loss, which assigns a penalty of 0 to the correct prediction, and $K_{i}$ otherwise, where $i$ indexes the classes. Thus, the posterior expected losses in a binary classification case are $K_0 p(y_{0} | x, \mathcal{D})$ and $K_1 p(y_{1} | x, \mathcal{D})$. In this scenario the decision rule is not to take the argmax as before, but to predict class $i$ if the $p(y_{i} | x, \mathcal{D}) > \dfrac{K_i}{\sum_{i=0}^{n_c} K_i}$. Thankfully, this is straightforward in our framework. To certify this decision rule it is enough to check that $D^{L}_{\text{safe}, i} \geq \dfrac{K_i}{\sum_{i=0}^{n_c} K_i}$. We refer interested readers to Section 4.4.3 of \cite{berger2013statistical} for more in-depth discussion.

%\ap{I would add some trivial sentences on how to deal with the classification when instead of being the 0-1 loss is the weighted one. In this way we give a sense of completeness to this section which is called "Decision Rules", while at the moment it discusses only one.}.

\section{Proofs}\label{appendix:proofs}
In this section of the Appendix, we provide proofs for the main theoretical results stated in the paper.
\subsection{Lemma~\ref{Prop:SafetyComputation}}
\begin{proof}
By the definition of the maximal safe weight set we have $w \in H \iff \forall x \in T, f^{{w}}(x)\in S$. Moreover, we have that the probability of a weight being in such a set is given as $Prob_{w \sim \post}(w \in H) = \int_{H} \post dw $. By making explicit the definition of $H$, together these two give us $Prob_{w \sim \post}(\forall x \in T, f^{{w}}(x)\in S) = \int_{H} \post dw $.
The second equality stated in the lemma formulation follows directly from the latter result and the property of complementary probabilities, with $w \in H$ and $w \in K$ being two complementary events.
\end{proof}

\subsection{Proposition~\ref{prop:pred_posterior_bound}}
\begin{proof}
We prove the results explicitly for the lower bound; the derivation of the upper bound is analogous. 
Consider the minimisation over $T$ of the expected value computed over the posterior distribution of Problem \ref{prob:adv_pred} for output index $c \in \{1,\ldots,m\}$:
\begin{align*}
    \min_{x \in T} \mathbb{E}_{\post} [\sigma_c(f^w(x))] = \min_{x \in T} \int \sigma_c(f^w(x)) p(w | \mathcal{D}) dw.
\end{align*} 
Let $I = \mathbb{R}^{n_w} \setminus \bigcup_{i=1}^{n_J} J_i$  
Since the weight intervals in $\mathcal{J   }$ are disjointed we can rely on the linearity of integrals to obtain:
\begin{align*}
    &\min_{x \in T} \int \sigma_c(f^w(x)) p(w | \mathcal{D}) dw = \min_{x \in T} \Bigg( \sum_{i=1}^{n_J} \int_{J_i} \sigma_c(f^w(x)) p(w | \mathcal{D}) \\
    & + \int_{I} \sigma_c(f^w(x)) p(w | \mathcal{D} ) \Bigg).
\end{align*}
We notice that, for every $x$, $\int_{J_i} \sigma_c(f^w(x)) p(w | \mathcal{D}) \geq \min_{w \in J_i} \sigma_c(f^w(x))\int_{J_i}  p(w | \mathcal{D}) $.
By combining this result with the above chain of equalities, and further relying on the property of minimum, we obtain that:
\begin{align*}
    \min_{x \in T} \Bigg( \sum_{i=1}^{n_J} \int_{J_i} \sigma_c(f^w(x)) p(w | \mathcal{D}) + \int_{I} \sigma_c(f^w(x)) p(w | \mathcal{D} ) \Bigg) \geq \\ \sum_{i=1}^{n_J}  \int_{J_i}\post dw  \min_{\substack{x \in T\\ w \in J_i }} \sigma_c(f^w(x)) + \\ \sigma^L \left( 1 - \sum_{i=1}^{n_J}  \int_{J_i}\post dw  \right) = D_{\text{safe},c}^L, 
\end{align*}
which proves the theorem statement.
\end{proof}

\subsection{Proposition \ref{Prop:IBP}}\label{appendix:ibpproof}
The bounding box can be computed iteratively in the number of hidden layers of the network, $K$.
We show how to compute the lower bound of the bounding box; the computation for the maximum is analogous. 

Consider the $k$-th network layer, for $k=0,\ldots,K$, we want to find for $i=1,\ldots n_{k+1}$:
\begin{equation*}
 \min_{ \substack{ W^{(k)}_{i:} \in [W^{(k),L}_{i:} , W^{(k),U}_{i:}]  \\ z^{(k)} \in [z^{(k),L},z^{(k),U}] \\ b^{(k)}_i \in [b^{(k),L}_i, b^{(k),U}_i ] }}    z^{(k+1)}_i = \sigma \left( \sum_{j=1}^{n_k} W^{(k)}_{ij} z^{(k)}_j + b^{(k)}_i  \right).
\end{equation*}
As the activation function $\sigma$ is monotonic, it suffices to find the minimum of:  $\sum_{j=1}^{n_k} W^{(k)}_{ij} z^{(k)}_j + b^{(k)}_i$.
Since $W^{(k)}_{ij} z^{(k)}_j$ is a bi-linear form defined on an hyper-rectangle, it follows that it obtains its minimum in one of the four corners of the rectangle $[W^{(k),L}_{ij} , W^{(k),U}_{ij}] \times [z^{(k),L}_j,z^{(k),U}_j]$.\\
Let  $t_{ij}^{(k),L} = \min \{ W_{ij}^{(k),L} z_j^{(k),L} , W_{ij}^{(k),U} z_j^{(k),L} , \\ W_{ij}^{(k),L} z_j^{(k),U} , W_{ij}^{(k),U} z_j^{(k),U} \}$
we hence have:
$$
\sum_{j=1}^{n_k} W^{(k)}_{ij} z^{(k)}_j + b^{(k)}_i \geq  \sum_{j=1}^{n_k} t_{ij}^{(k),L} + b^{(k),L}_i =: \zeta^{(k+1),L}_i.
$$
Thus for every $W^{(k)}_{i:} \in [W^{(k),L}_{i:} , W^{(k),U}_{i:}]$,  $z^{(k)} \in [z^{(k),L},z^{(k),U}]$ and $b^{(k)}_i \in [b^{(k),L}_i, b^{(k),U}_i ]$ we have:
$$
    \sigma \left( \sum_{j=1}^{n_k} W^{(k)}_{ij} z^{(k)}_j + b^{(k)}_i  \right) \geq \sigma \left( \zeta^{(k+1),L}_i \right) 
$$
that is $z^{(k+1),L}_i = \sigma \left( \zeta^{(k+1),L}_i \right)  $ is a lower bound to the solution of the minimisation problem posed above.

%Consider $ z^{(0)} = \sigma ( W^{(0)} x + b^{(0)}  )$ with $W^{(0)} \in [w^{(0),L},w^{(0),U} ] \subset \mathbb{R}^{n_{out} \times n_{in}}$  , $b^{(0)} \in [b^{(0),L},b^{(0),U} ] \subset \mathbb{R}^{n_{out}}$ and $x \in [x^L,x^U]  \subset \mathbb{R}^{n_{in}} $.
%We want to find the minimum and maximum values for the entries of $z^{(0)} = [z^{(0)}_1,\ldots,z^{(0)}_{n_{out}}]$.
%Consider the $i$th index of this, as $\sigma$ is monotonic increasing entry-wise it suffice to find minimum and maximum of the linear form  $\zeta^{(0)} = W^{(0)} x + b^{(0)}$, and pass them through the activation function.
%So we just need to compute
%\begin{align*}
%    \min_{x,W^{(0)}_i,b^{(0)}_i } \zeta^{(0)}_i = (W^{(0)}_i x + b^{(0)}_i) \\
%    \max_{x,W^{(0)}_i,b^{(0)}_i } \zeta^{(0)}_i =  (W^{(0)}_i x + b^{(0)}_i) 
%\end{align*}
%As we have $W^{(0)}_i x + b^{(0)}_i = \sum_j  W^{(0)}_{ij} x_j + b^{(0)}_i $ it suffice to propagate the interval through the multiplication and we are done. Defining $\underline{w}_j = \min \{ W^{(0),L}_{ij} \cdot x_j^L, W^{(0),L}_{ij} \cdot x_j^U, W^{(0),U}_{ij} \cdot x_j^L , W^{(0),U}_{ij} \cdot x_j^U \}$ we have that:
%$\min \zeta^{(0)}_i = \sum_j \underline{w}_j + b^{(0)}_i $. Analogous stuff for the maximum. We can propagate this through the activation function, and we find in this way lower and upper bounds on $z^{(0)}$.
%These we can use to iterate the process for the next layer.

%
\subsection{Proposition \ref{proposition:lbp}}\label{appendix:lbpproof}
We first state the following lemma that follows directly from the definition of linear functions:
\begin{lemma}
\label{lemmma:linear_prop}
Let $f^L(t) = \sum_j a_j^L t_j + b^L $ and $f^U(t) = \sum_j a_j^U t_j + b^U$ be lower and upper  LBFs to a function $g(t)$ $\forall t \in \mathcal{T}$, i.e., $ f^L(t) \leq g(t) \leq  f^U(t) $ $\forall t \in \mathcal{T}$. Consider two real coefficients $\alpha \in \mathbb{R}$ and $\beta \in \mathbb{R}$.
Define
\begin{align}
    &\bar{a}_j^L = \begin{cases}  \alpha a_j^L \; \textrm{if} \,   \alpha \geq 0  \\ \alpha a_j^U \; \textrm{if} \,   \alpha < 0  \end{cases}
    \bar{b}^L = \begin{cases}  \alpha b^L + \beta \; \textrm{if} \,   \alpha \geq 0  \\ \alpha b^U  + \beta \; \textrm{if} \,  \alpha < 0  \end{cases} \label{eq:lin_transf1}\\
    &\bar{a}_j^U = \begin{cases}  \alpha a_j^U \; \textrm{if} \,   \alpha \geq 0  \\ \alpha a_j^L \; \textrm{if} \,   \alpha < 0  \end{cases}
    \bar{b}^U = \begin{cases}  \alpha b^U + \beta \; \textrm{if} \,   \alpha \geq 0  \\ \alpha b^L  + \beta \; \textrm{if} \,   \alpha < 0  \end{cases} \label{eq:lin_transf2}
\end{align}
Then:
\begin{align*}
\bar{f}^L(t) := \sum_j \bar{a}_j^L t_j + \bar{b}^L \leq \alpha g(t) + \beta \leq \sum_j \bar{a}_j^U t_j + \bar{b}^U \\ =: \bar{f}^U(t)
\end{align*}
That is, LBFs can be propagated through linear transformation by redefining the coefficients through Equations \eqref{eq:lin_transf1}--\eqref{eq:lin_transf2}.
\end{lemma}

We now proof Proposition  \ref{proposition:lbp} iteratively on $k=1,\ldots,K$ that is that for $i=1,\ldots,n_k$ there exist $f_i^{(k),L}(x,W)$ and $f_i^{(k),U}(x,W)$ lower and upper LBFs such that:
\begin{align}
        &\zeta^{(k)}_i \geq  f_i^{(k),L}(x,W) := \mu_i^{(k),L} \cdot x +  \label{eq:lbp1} \\ & \sum_{l = 0}^{k-2} \langle \nu_i^{(l,k),L} , W^{(l)}  \rangle + \nu_i^{(k-1,k),L} \cdot W^{(k-1)}_{i:} + \lambda_i^{(k),L}  \nonumber\\
     &\zeta^{(k)}_i \leq  f_i^{(k),U}(x,W) := \mu_i^{(k),U} \cdot x + \label{eq:lbp2} \\ & \sum_{l = 0}^{k-2} \langle \nu_i^{(l,k),U} , W^{(l)}  \rangle + \nu_i^{(k-1,k),U} \cdot W^{(k-1)}_{i:} + \lambda_i^{(k),U} \nonumber
\end{align}
and iteratively find valid values for the LBFs coefficients, i.e., $\mu_i^{(k),L}$, $\nu_i^{(l,k),L}$, $\lambda_i^{(k),L}$, $\mu_i^{(k),U}$, $\nu_i^{(l,k),U}$ and $\lambda_i^{(k),U}$.

For the first hidden-layer we have that $\zeta^{(1)}_i = \sum_j W_{ij}^{(0)} x_j + b^{(0)}_i$.
By inequality \eqref{eq:mc1} and using the lower bound for $b^{(0)}_i$ we have:
\begin{align*}
&\zeta^{(1)}_i \geq \sum_j \left( W_{ij}^{(0),L} x_j  +  W_{ij}^{(0)} x^{L}_j - W_{ij}^{(0),L} x^{L}_j  \right) + b^{(0),L}_i \\
&= W_{i:}^{(0),L} \cdot x + W_{i:}^{(0)} \cdot x^{L} - W_{i:}^{(0),L} \cdot x^{L} + b_i^{(0),L}
\end{align*}
which is a lower LBF on $\zeta^{(1)}$.
Similarly, using Equation \eqref{eq:mc2} we obtain:
\begin{align*}
   \zeta^{(1)}_i \leq W_{i:}^{(0),U} \cdot x + W_{i:}^{(0)} \cdot x^{L} - W_{i:}^{(0),U} \cdot x^{L} + b_i^{(0),U}
\end{align*}
which is an upper LBF on $\zeta^{(1)}$. 
By setting: 
\begin{align*}
    \mu_i^{(1),L} &= W_{i:}^{(0),L} \quad , \quad  \mu_i^{(1),U} = W_{i:}^{(0),U} \\
    \nu_i^{(0,1),L} &= z^{(0),L} \quad , \quad  \nu_i^{(0,1),U} = x^{L}  \\
    \lambda_i^{(1),L} &= - W_{i:}^{(0),L} \cdot x^{L} + b_i^{(0),L}  \\  
    \lambda_i^{(1),U} &= - W_{i:}^{(0),U} \cdot x^{L} + b_i^{(0),U}
\end{align*}
we obtain LBFs $f_i^{(1),L}(x,W)$ and $f_i^{(1),U}(x,W)$ of the form \eqref{eq:lbp1}--\eqref{eq:lbp2}.

Given the validity of Equations \eqref{eq:lbp1}--\eqref{eq:lbp2} up to a certain $k$, we now show how to compute the LBF for layer $k+1$, that is, given $f_i^{(k),L}(x,W)$ and $f_i^{(k),U}(x,W)$ we explicitly compute $f_i^{(k+1),L}(x,W)$ and $f_i^{(k+1),U}(x,W)$.
Let $\zeta^{(k),L}_i = \min f_i^{(k),L}(x,W) $ and $\zeta^{(k),U}_i = \max f_i^{(k),U}(x,W) $ be the minimum and maximum of the two LBFs (which can be computed analytically as the functions are linear). 
For Lemma \ref{lemmma:act_bound} there exists a set of coefficients such that $ z^{(k)}_i = \sigma(\zeta^{(k)}_i) \geq \alpha^{(k),L}_i \zeta^{(k)}_i + \beta^{(k),L}_i$.
By Lemma \ref{lemmma:linear_prop} we know that there exists $\bar{f}_i^{(k),L}(x,W)$ with coefficients $\bar{\mu}_i^{(k),L}$, $\bar{\nu}_i^{(l,k),L}$, $\bar{\lambda}_i^{(k),L}$  obtained through Equations \ref{eq:lin_transf1}--\ref{eq:lin_transf2} such that:
\begin{align*}
    z^{(k)}_i  \geq \alpha^{(k),L}_i f_i^{(k),L}(x,W) + \beta^{(k),L}_i \geq \bar{f}_i^{(k),L}(x,W) 
\end{align*}
%
%\begin{align*}
%    &z^{(k)}_i = \sigma(\zeta^{(k)}_i) \geq \alpha^{(k),L}_i \zeta^{(k)}_i + \beta^{(k),L}_i \geq \\
%    & \geq \begin{cases}
%    \alpha^{(k),L}_i f_i^{(k),L}(x,W) +  \beta^{(k),L}_i \quad \textrm{if} \; \alpha^{(k),L}_i \geq 0    \\
%    \alpha^{(k),L}_i f_i^{(k),U}(x,W) +  \beta^{(k),L}_i \quad \textrm{if} \; \alpha^{(k),L}_i < 0.
%    \end{cases}
%\end{align*}
%
that is $\bar{f}_i^{(k),L}(x,W)$ is a lower LBF on $z^{(k)}_i$ with coefficients $\bar{\mu}_i^{(k),L}$, $\bar{\nu}_i^{(l,k),L}$, $\bar{\lambda}_i^{(k),L}$.
Analogously, let $\bar{f}_i^{(k),U}(x,W)$ be the upper LBF on $z^{(k)}_i$ computed in a similar way.

Consider now the bi-linear layer $\zeta_i^{(k+1)} = \sum_{j} W_{ij}^{(k)} z_j^{(k)} + b_i^{(k)}$.
From Equation \eqref{eq:mc1} we know that: $W_{ij}^{(k)} z_j^{(k)} \geq  W_{ij}^{(k),L} z^{(k)}_j  +  W_{ij}^{(k)} z^{(k),L}_j - W_{ij}^{(k),L} z^{(k),L}_j$.
By applying Lemma \ref{lemmma:linear_prop} with $\alpha = W_{ij}^{(k),L}$ and $\beta = 0$ we know that there exists a lower LBF $\hat{f}_{ij}^{(k),L}(x,W)$ with  a set of  coefficients $a_{ij}^{(k),L}$, $b_{ij}^{(l,k),L}$ and  $c_{ij}^{(k),L}$ computed applying Equations \eqref{eq:lin_transf1}--\eqref{eq:lin_transf2} to $\bar{\mu}_i^{(k),L}$, $\bar{\nu}_i^{(l,k),L}$, $\bar{\lambda}_i^{(k),L}$ such that: $W_{ij}^{(k),L} z^{(k)}_j \geq \hat{f}_{ij}^{(k),L}(x,W)$.
Hence we have:
\begin{align*}
    & \zeta_i^{(k+1)} = \sum_j  W_{ij}^{(k)} z_j^{(k)} + b_i^{(k)} \geq \sum_j \big( W_{ij}^{(k),L} z^{(k)}_j  + \\
    & W_{ij}^{(k)} z^{(k),L}_j  - W_{ij}^{(k),L}  z^{(k),L}_j \big) + b_i^{(k),L} \geq  \\
    & \sum_j \hat{f}_{ij}^{(k),L}(x,W) +  \sum_j W_{ij}^{(k)} z^{(k),L}_j   - \\ & \sum_j W_{ij}^{(k),L}  z^{(k),L}_j + b_i^{(k),L} = \\
    & \sum_j  \big( a_{ij}^{(k),L} \cdot x +   \sum_{l = 0}^{k-2} \langle b_{ij}^{(l,k),L} , W^{(l)} \rangle \\ &  + b_{ij}^{kl-1,k),L} \cdot W^{(k-1)}_{j:} + c_{ij}^{(k),L} \big) + \\ & W_{i:}^{(k)} \cdot z^{(k),L} - W_{i:}^{(k),L}  z^{(k),L}.
\end{align*}
By setting 
\begin{align*}
&  \mu_i^{(k+1),L} = \sum_j a_{ij}^{(k),L}\\
&  \nu_i^{(l,k+1),L} = \sum_j b_{ij}^{(l,k),L}  \quad k = 0,\ldots,l-2 \\
&  \nu_i^{(k-1,k+1),L} = b_{i}^{(k-1,k),L} \\
&  \nu_i^{(k,k+1),L} =  z^{(k),L} \\
&  \lambda_i^{(k+1),L} = \sum_j c_{ij}^{(k),L} - W_{i:}^{(k),L} \cdot z^{(k),L} +  b_i^{(k),L} 
\end{align*}
and re-arranging the elements in the above inequality, we finally obtain:
\begin{align*}
    &\zeta_i^{(k+1)} \geq \mu_i^{(k+1),L} \cdot x + \sum_{l = 0}^{k-1} \langle \nu_i^{(l,k+1),L} , W^{(l)}  \rangle + \\
    &\nu_i^{(k,k+1),L} \cdot W^{(k)}_{i:} + \lambda_i^{k+1),L}  =: f_i^{(k + 1),L}(x,W)
\end{align*}
which is of the form of Equation \eqref{eq:lbp1} for the lower LBF for the $k+1$-th layer.
Similarly, an upper LBF of the form of Equation \eqref{eq:lbp2} can be obtained by using Equation \eqref{eq:mc2} in the chain of inequalities above.

\section{Algorithms and Discussion}\label{appendix:algorithms}
\iffalse
\subsection{Computational Complexity}\label{sec:complexity}

Each or our algorithms is built in two main stages: first is output over-approximation of a set of $N$ NN weight samples, and second is the probability computation by Bonferroni bounds. 

The first stage has a complexity which is linear in the number of samples, $N$, taken from the posterior distribution of $w$. The computational complexity of this stage method is then determined by the computational complexity of the method used to propagate a given interval $\hat{H}$ (that is, line 5 in Algorithm 1). The cost of performing IBP is $\mathcal{O}(Knm)$ where $K$ is the number of hidden layers and $n \times m$ is the size of the largest weight matrix $W^{(k)}$, for $k=0,\ldots,K$. LBP is instead  $\mathcal{O}(K^2nm)$.

The second stage, the probability computation is $\mathcal{O}(N)$ in the case that all weight intervals are disjoint where each weight interval requires $\mathcal{O}(2Knm)$ time to compute its posterior probability. In the event that Bonferroni bounds are used, they dominate the complexity, requiring $\mathcal{O}(N^{d})$ where $d$ is the `depth' of the Bonferroni computation, $v$ or $u$ in Corollary~\eqref{Corollary:Bonferroni}. The exponential complexity of computing such bounds hints and the incredible speed up achievable through smart sampling schemes. 
\fi

\begin{algorithm} 
\caption{Lower Bounds for $\dsafe$}\label{alg:dsafelower}
\textbf{Input:} $T$ -- Compact Input Region, $f^{\mathbf{w}}$ -- Bayesian Neural Network, $p(w | \mathcal{D})$ -- Posterior Distribution,  $N$ -- Number of Samples, $\gamma$ -- Weight margin.  \\
\textbf{Output:} A sound lower bound on $\dsafe$.%\vspace{-1em}
\hrulealg 
\begin{algorithmic}[1]
\STATE \algcomdec{\textit\# {$\mathcal{J}$ is an arbitrary set of weight intervals}}
\STATE $\mathcal{J} \gets \emptyset$ 
\STATE \algcomdec{\textit\# {$\hat{\Psi}$ is a set of worst-case predicted outputs}}
\STATE $\hat{\Psi} \gets \emptyset$  
\STATE \algcom{\textit\# {Element-wise products to get width of weight margin.}}
\STATE $v \gets \gamma \cdot I \cdot \Sigma$
\FOR {$i\gets0$ to $N$} 
    \STATE $w^{(i)} \sim p(w | \mathcal{D})$
    \STATE \algcom{\textit\# {Assume weight intervals are built to be disjoint}}
    \STATE $[w^{(i), L}, w^{(i),U}] \gets [w_i - v , w_i + v]$
    \STATE \algcom{\textit{\# Interval/Linear Bound Propagation, Section~\ref{subsec:convexrelax}}}
    \STATE $\*y^{L}, \*y^{U} \gets \texttt{Propagate}(f, T, [w^{(i), L}, w^{(i), U}]$) 
    \STATE \algcom{\textit{\# Output worst-case see Section VI-F}}
    \STATE $y^{\text{worst}} \gets \texttt{Output-Worst}([\*y^{L}, \*y^{U}])$
    \STATE $ \mathcal{J} \gets \mathcal{J} \bigcup \{[w^{(i), L}, w^{(i), U}]\}, \quad \hat{\Psi} \gets \hat{\Psi} \bigcup \{y^{\text{worst}}\}$ %\ap{Is this supposed to be worst instead of lower?}
\ENDFOR
\STATE $y_{\text{mean}} \gets 0.0$;  $p_{\text{total}} \gets 0.0$
%\STATE $p \gets 0.0$ 
\FOR {$i\gets0$ to $|\mathcal{J}_i|$} 
    \STATE \algcomdec{\textit\# {Mult. weight probs and output bounds.}}
    \STATE $y_{\text{mean}} = y_{\text{mean}} + \hat{\Psi}_i P(\mathcal{J}_i)$ %\ap{$\Psi$ seems to be defined as a set, so I am not sure what $\Psi_i$ means? Do you want it to be a vector?}\MRW{$\Psi$ is a set of vectors. The $i$th value is the $i$th vector. I think its fine as it is.}
    \STATE $p_{\text{total}} = p_{\text{total}} + P(\mathcal{J}_i)$
\ENDFOR
\STATE \algcomdec{\textit\# {Complete the bound according to Proposition 2.}}
\STATE $D_{\text{safe}}^{L} = y_{\text{mean}} + (1-p_{\text{total}}) \sigma^{L}$
%\STATE \algcomdec{\textit{We now know that  $y_{\text{lower}} \leq \dsafe$.}} \ap{Why $y_{lower}$??? It should be $D_{\text{safe}}^L$ no?} %\ap{Carefull $y_{\text{lower}}$ appears to be overloaded}\MRW{It was just a typo above :)} 
\STATE return $D_{\text{safe}}^{L}$
\end{algorithmic}
\end{algorithm}

\subsection{Lower Bound on $\dsafe$}
%\ap{Check comment on pseudocode!}
In Algorithm~\ref{alg:dsafelower}, we provide step-by-step pseudocode for lower bounding $\dsafe$. One can notice that the algorithm follows a similar computational flow to Algorithm~\ref{alg:psafealgorithm} in the main text. Namely, on lines 2 and 4 we establish the sets that we will keep track of (weight intervals and worst-case outputs, respectively). Then in lines 7--16 we iteratively sample pairwise disjoint weight intervals and compute their worst case outputs. On line 14, a key modification is added compared to the lower bound on $\psafe$, which is the computation of the worst-case output.  In the case of softmax classification we have that $\texttt{Output-Worst}$ takes the form: %\ap{See comment on $O$ on the pseudocode. In general I don't like to name something big $O$ it might be confusing to some for the computational time stuff.}\MRW{I have made the change :) }
\begin{equation}\label{eq:worst-softmax}
    \texttt{Output-Worst}([y^{L}, y^{U}]) = \dfrac{\exp(y^{L}_c)}{\exp(y^{L}_c) + \sum_{l\neq c}^{n_c}\exp(y^{U}_l)}
\end{equation}
That is, the lower bound for the true class and the upper bound for all other classes. For regression $\texttt{Output-Worst} = y^{L}$.  %\ap{Oh I guess $O$ is either Eq (23) or this $y^l$ thing. Maybe you can call it something else? But define it properly, cause at the moment is not clear.}.
Both of these represent the worst-case output and satisfy the conditions needed for Proposition 2 in the main text. Finally, in lines 17-22 we compute the necessary components for our bound in Proposition 2 and complete the bound on line 24. Overall, the computational complexity of this algorithm is exactly the same as the lower bound on probabilistic safety and in practice the computational times are only fractionally different.

\subsection{Upper Bound on $\psafe$}

%\MK{Likewise, too informal and needs more mathematical rigour}

We provide a pseudocode for the computation of the upper bound on $\psafe$ in Algorithm \ref{alg:upperbound}.
%\ap{Maybe begin by addind an initial sentence that refers to the algorithm, and the continue like you are doing to explain the differences?}
To do this we compute unsafe weight sets. We wish to determine that a weight interval is unsafe i.e., the logical inverse of our safety property: $\lnot (f^{w}(x) \in S \  \forall x \in T) = (\exists x \ s.t. \  f^{w}(x) \notin S)$. %\ap{What is happening with that big cross?} 
Notice that, unlike the procedure for computing safety, here we do not need to jointly propagate a weight-space interval together with the full input specification $T$ %\ap{together with the input space region $T$?}\MRW{Yes, corrected. But maybe this is still confusing. Let me know!}
as we only need to find %\MK{how to guarantee this can be found? what if not?} 
a single $x$ which causes the entire weight interval to be mapped outside of $S$, and note that every $x \in T$ returns a valid bound.  Finding an $x$ that violates the property is identical to the formulation for adversarial examples. Thus, in order to test if there exists a single input that causes the weight interval to be unsafe, we leverage the developments in adversarial attacks in order to attack each sampled weight $w_i$ (done on line 4 of Algorithm~\ref{alg:upperbound}). %We highlight that our algorithm is transparent to the methods of computing adversarial examples as well as the method for performing bound propagation. As such, users can select an attack or propagation method corresponding to the desired computational complexity. Moreover, using fast methods such as FGSM and IBP requires the computation equivalent of only four forward passes through the neural network architecture. % which can be done efficiently with modern deep learning software.

\begin{algorithm}
\caption{Upper Bounding $\psafe$}\label{alg:upperbound}
\textbf{Input:} $T$ -- Input Set, $S$ -- Safe Set, $f^{\mathbf{w}}$ -- Bayesian Neural Network, $\mathbf{w}$ -- Posterior Distribution,  $N$ -- Number of Samples, $\gamma$ -- Weight Margin.  \\
%\MK{$S$ not explicit but used below, also fix line break}
\textbf{Output:} Safe upper bound on $\psafe$.
\hrulealg
\begin{algorithmic}[1]
\STATE \algcom{\textit\# {$\mathcal{K}$ is a set of known unsafe weight intervals}}
\STATE $\mathcal{K} \gets \emptyset$  
\STATE \algcom{\textit\# {Element-wise products to get width of weight margin.}}
\STATE $v \gets \gamma \cdot I \cdot \Sigma$ \hfill
\FOR {$i\gets0$ to $N$} 
    \STATE $w^{(i)} \sim p(w | \mathcal{D})$
    \STATE \algcom{\textit\# {Assume weight intervals are built to be disjoint}}
    \STATE $[w^{(i), L}, w^{(i),U}] \gets [w_i - v , w_i + v]$
    \STATE \algcom{\textit{\# FGSM/PGD}}
    \STATE $\*x_{\text{adv}} \gets \texttt{Attack}(f, w_i, T)$ 
    \STATE \algcom{\textit{\# Interval/Linear Bound Propagation}}
    \STATE $\*y^{L}, \*y^{U} \gets \texttt{Propagate}(f, \*x_{\text{adv}} , [w^{(i), L}, w^{(i), U}]$) 
    \IF{$\forall \*y \in [y^L, y^U] \*y \notin S$)}
        \STATE $ \mathcal{K} \gets \mathcal{K} \bigcup \{[w^{(i), L}, w^{(i), U}]\}$
    \ENDIF
\ENDFOR
\STATE $P_{\text{unsafe}} \gets 0.0 \quad$  %\ap{Can you make the following bit more similar to Algorithm 1? Or is there a reason why you have written it differently?} \MRW{Yes! Done.} \ap{Is actually still different! You should use that interval belonging to the family notation. Also use the proper notation for p, the one for the upper bound introduced in the text, not this random p thing}
\FOR{$i=0..|\mathcal{K}|$} 
    \STATE $P_{\text{unsafe}} = P_{\text{unsafe}} + P(\mathcal{K}_{i})$
\ENDFOR
\STATE $P_{\text{safe}}^{U} = 1 - P_{\text{unsafe}}$
\STATE return $P_{\text{safe}}^{U}$
\end{algorithmic}
\end{algorithm}

\subsection{Upper Bound on $\dsafe$}
%\ap{check comments on algorithm.}
We provide a pseudocode for the computation of the upper bound on $\dsafe$ in Algorithm \ref{alg:upperdecisionsafety}.
The main change to this algorithm is a change from computing the worst-case output to computing the best-case output. This is done with the \texttt{Output-Best} function. %\ap{$O$ function again is kind of undefined} \MRW{Fixed!} 
In the case of softmax classification $\texttt{Output-Best}$ takes the form: 
\begin{equation}\label{eq:upperboundsoftmax}
    \texttt{Output-Best}(y^{L}, y^{U}) = \dfrac{exp(y^{U}_c)}{exp(y^{U}_c) + \sum_{l\neq c}^{n_c}exp(y^{L}_l)}
\end{equation}
and for regression, $\texttt{Output-Best} = y^{U}$. 

\begin{algorithm} [h]
\caption{Upper Bounding $\dsafe$}\label{alg:upperdecisionsafety}
\textbf{Input:} $T$ -- Input Set, $f^{\mathbf{w}}$ -- Bayesian Neural Network, $p(w | \mathcal{D})$ -- Posterior Distribution,  $N$ -- Number of Samples, $\gamma$ -- Weight margin. \\
\textbf{Output:} A sound lower bound on $\dsafe$.
%\ap{I feel this intervals should be $\mathcal{J}$ and they need not be safe.}
\hrulealg
\begin{algorithmic}[1]
\STATE \algcom{\textit\# {$\mathcal{H}$ is a set of known safe weight intervals}}
\STATE $\mathcal{J} \gets \emptyset$  
\STATE \algcom{\textit\# {$\hat{\Psi}$ is a set of best-case predicted outputs}}
\STATE $\hat{\Psi} \gets \emptyset$  
\STATE \algcom{\textit\# {Element-wise products to get width of weight margin.}}
\STATE $v \gets \gamma \cdot I \cdot \Sigma$
\FOR {$i\gets0$ to $N$} 
    \STATE $w^{(i)} \sim p(w | \mathcal{D})$
    \STATE $[w^{(i), L}, w^{(i),U}] \gets [w_i - v , w_i + v]$
    \STATE \algcom{\textit{\# Interval/Linear Bound Propagation, Section~\ref{subsec:convexrelax}}}
    \STATE $\*y^{L}, \*y^{U} \gets \texttt{Propagate}(f, T, [w^{(i), L}, w^{(i), U}]$) 
    \STATE \algcom{\textit{\# Output upperbound see Eq~\eqref{eq:upperboundsoftmax}}}
    \STATE $y^{\text{upper}} \gets \texttt{Output-Best}([\*y^{L}, \*y^{U}])$ 
    \STATE $ \mathcal{J} \gets \mathcal{J} \bigcup \{[w^{(i), L}, w^{(i), U}]\}, \quad \hat{\Psi} \gets \hat{\Psi} \bigcup \{y^{\text{upper}}\}$
    %\STATE $ \hat{\Psi} \gets \hat{\Psi} \bigcup \{y^{\text{worst}}\}$
\ENDFOR
\STATE $y_{\text{mean}} \gets 0.0$;  $p_{\text{total}} \gets 0.0$
%\STATE $p \gets 0.0$ 
\FOR {$i\gets0$ to $N$} 
    \STATE \algcomdec{\textit\# {Mult. weight probs and output bounds}}
    \STATE $y_{\text{mean}} = y_{\text{mean}} + \hat{\Psi}_i P(\mathcal{H}_i)$ 
    \STATE $p_{\text{total}} = p_{\text{total}} + P(\mathcal{H}_i)$
\ENDFOR
\STATE \algcomdec{\textit\# {Complete the bound according to Proposition 2.}}
\STATE $D_{\text{safe}}^{U} = y_{\text{mean}} + (1-p_{\text{total}}) \sigma^{U}$
%\STATE \algcomdec{\textit{\# $y_{\text{upper}} \geq \dsafe$.}} \ap{Careful $y_{\text{upper}}$ appears to be overloaded.} \ap{!!!!} \ap{USe the aactual notation for the upper bound not this $y_{upper}$ thing you ahve made up here on the spot.}
\STATE return $D_{\text{safe}}^{U}$
\end{algorithmic}
\end{algorithm}

\subsection{Bonferroni Bounds for Overlapping Weight Intervals}\label{appendix:bonferroni}
\iffalse
\ap{I think the paragraph below is legacy and not needed anymore}
For sample-based posteriors where our hyper-rectangles have zero-width we can simplify this computation to the following: 
\begin{corollary}
\label{Corollary:SafetyComp}
Assume that that we have a sample-based posterior approximation, thus that we have a set of $c$ non-unique weight-samples $\{w^{(i)}\}_{i=0}^{c}$.  Further, let $H = \{\hat{H}_1,...,\hat{H}_m\}$ be $m$ safe sets of weights such that, for $i\in \{1,..,m \}$, $\hat{H}_i=w^{(i)}$. Then, it holds that
\begin{align}
  P_{\text{safe}}(T,S)\geq \dfrac{1}{c}\sum_{i=0}^{c} \mathbb{I}(w^{(i)} \in H)
\end{align}
\end{corollary}
where $\mathbb{I}(w_i \in H)$ returns 1 if the sample is inside of the safe interval(s) and 0 otherwise. 
%
\ap{End legacy here?}
\fi

%\MK{Not sure I follow - citation for Bonferoni bounds?}
A key challenge of Proposition~\ref{Prop:SafetyComputation2} in the variational inference case  is ensuring that the hyper-rectangles are pairwise disjoint (i.e., $\hat{H}_i \cap \hat{H}_j=\emptyset$). If this is not the case, then enforcing independence can be computationally tricky, as the relative complement of two or more hyper-rectangles is not necessarily a hyper-rectangle. While one could modify the sampling procedure so to reject overlapping intervals, or could devise a scheme for sampling pairwise disjoint hyper-rectangles, for a high values of $N$ and for a large number of parameters this becomes computationally intensive. %\ap{While one could modify the sampling procedure so to reject overlapping intervals, for a high values of $N$ this becomes computationally intensive.} \MRW{Revised and added.}
To solve this, we highlight that the disjoint union of two or more hyper-rectangles is necessarily a hyper-rectangle. %\ap{Use $\sqcup$ where/if necessary} \MRW{Done :) }
Therefore, we can employ Bonferroni inequalities \cite{bonferroni1936teoria} to get upper and lower bound on the posterior probability of non-disjoint hyper-rectangles: %we first consider that a single hyper rectangle, $\hat{H}_i$ satisfies the conditions of Corrollary 2 and therefore we can compute $Prob(\hat{H}_i)$. For multiple dependent hyperrectangles, we note that unlike the relative compliment, the union of any $n$ hyperrectangles is necessarily a hyperrectangle. Thus, one can bound the cumulative probability of dependent hyperrectangles from above and below through the use of Bonferroni inequalities as follows: 
\begin{corollary}
\label{Corollary:Bonferroni}
Assume that $\Sigma$, the covariance matrix of the posterior distribution of the weights, is diagonal with diagonal elements $\Sigma_1,...,\Sigma_{n_w}$. Let $\hat{H}_1,...,\hat{H}_M$ be $M$ safe sets of weights not necessary satisfying $\hat{H}_i \cap \hat{H}_j=\emptyset$ and let the probability of any $k$ of these safe sets simultaneously occurring be defined as:
\begin{align*}
  S_{k} := \bigsqcup_{i_{1} < ... < i_{k}} H_{i_{1}} \sqcup ... \sqcup H_{i_{k}}
\end{align*}
We then have that for any even integer $v$ and odd integer $u$ that the probability of the weights under the posterior is bounded: 
\begin{align*}
  \sum_{j=1}^{v} (-1)^{j} Prob(S_{j})
  \leq Prob(\hat{H}_1,...,\hat{H}_M) \leq  \sum_{j=1}^{u} (-1)^{j} Prob(S_{j})
\end{align*}
where $Prob(S_{j})$ is computed according to Corollary 2 as $S_{j}$ is a single hyper-rectangle.
\end{corollary}
Now that we can compute if a weight interval is guaranteed to be safe and we can compute a lower bound to the posterior probability covered by many weight intervals, we can combine these subroutines into algorithms for computing the required probability bounds. For bounds on decision robustness we need to consider the upper or lower bound output in conjunction with this probability. Recall that the upper or lower bound output determined by $\texttt{Output-Worst}$ or $\texttt{Output-Best}$ described in Appendix~\ref{appendix:algorithms} and the upper and lower bounds are stored such that the output bound of $\mathcal{J}_i$ is stored in $\hat{\Psi}_i$. To get a lower bound we modify the above corollary to be: 
\begin{align}
  \sum_{j=1}^{v} (-1)^{j} Prob(S_{j}) \max \{ \hat{\Psi}_i \}_{i=1}^{j}
  \leq \sum_{i=1}^{M} \hat{\Psi}_i Prob(\hat{J}_i).
\end{align}
To get an upper bound we use:
\begin{align*}
  \sum_{i=1}^{M} \hat{\Psi}_i Prob(\hat{J}_i) \leq \sum_{j=1}^{u} (-1)^{j} Prob(S_{j}) \min \{ \hat{\Psi}_i \}_{i=1}^{j}
\end{align*}
Here we can use the $\max$ operator for our lower bound and $\min$ for our upper bound as every value in the set $\{ \hat{\Psi}_i \}_{i=1}^{M}$ is a valid output bound for the disjoint union of hyper-rectangles.
%\ap{Do we do something similar for the decision robustness case as well? In that case I'd just mention that very briefly with one or two sentences}\MRW{You are right this is a pretty big oversight actually. I have added it. Perhaps these bounds need to be stated as a proposition and given a proof? After I have written them down they don't seem as obvious anymore.} \ap{If you are sure that they are correct, no need to put a proof. I would put a sentence on the computational time though.}

\section{Empirical Bound Validation}

\begin{figure}
    \centering
    \includegraphics[width=0.5\textwidth]{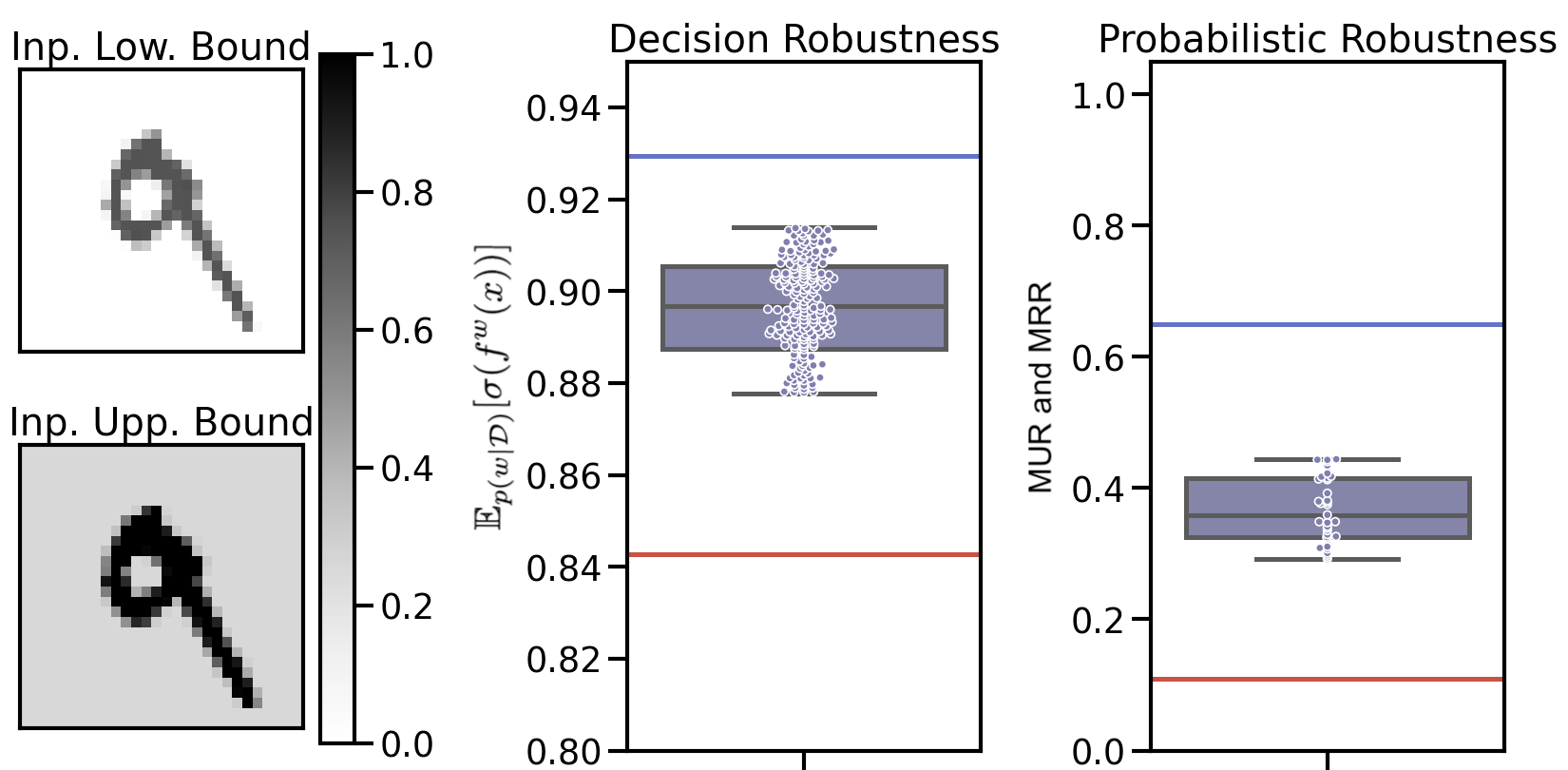}
    \caption{Comparing numerical bounds with adversarial examples to our certifications. \textbf{Left:} Example upper and lower bounds on a MNIST image. \textbf{Center:} Bounds on softmax output of a BNN, thus $\dsafe$. \textbf{Right:} Bounds on $\psafe$ maximum safe radius MinUR in red and MaxRR in blue. %\ap{MUR and MRR are still wrong notation.}
    }
    \label{fig:empanal}
\end{figure}

\subsection{Numerical Result}
%\MK{Unclear what adversarial examples are, this needs to be clarified and improved}
In Figure~\ref{fig:empanal}, we study how our bounds compare against empirical estimates of robustness achieved with adversarial attacks. In the centre panel of Figure~\ref{fig:empanal} ,we plot the upper- and lower-bounds on $\mathbb{E}[\sigma(f^{w}(x))]$ computed according to $\dsafe$. We also use 25 iterations of PGD to attempt to minimize $\mathbb{E}[\sigma(f^{w}(x))]$ for the true class, i.e., an adversarial attack on the BNN. % the confidence of the BNN, $\mathbb{E}_{\post} \sigma(f^{w}(x_{\text{adv}}))$, 
We run the attack 100 different times and plot the distribution of the results as the purple box plot in Figure~\ref{fig:empanal}. We see that our lower bound is strictly less than what any of the adversaries were able to achieve, and the upper bound is strictly greater than any of the points. This is due to the conservative nature of certification compared with attacks. Following the same procedure i.e., first using our bounds and then using PGD attacks for the same optimization, we study the MaxRR and MinUR. In the right-hand plot of Figure~\ref{fig:empanal}, we plot bounds on the MaxRR and MinUR with respect to probabilistic robustness. Similarly, we plot the empirical robust radius, which is the radius at which an adversarial attack was able to reduce a statistical estimate of $\psafe$ below 0.5. As we expect, the adversarial attacks all fall between our upper and lower bounds due to the conservative nature of certification.

\end{document}